%% file: arxiv_uniform.tex
\pdfoutput=1

\documentclass[11pt]{article}

\usepackage[top=1in, bottom=1in, left=1in, right=1in]{geometry}

\usepackage{times} 
\usepackage{bm} 
\usepackage{color} 
\usepackage{tikz}
\usetikzlibrary{positioning}
\usepackage{adjustbox}
\usepackage{mathtools}
\usepackage{thmtools}
\usepackage{thm-restate}
\mathtoolsset{showonlyrefs}
\usepackage[most]{tcolorbox} 
\usepackage[numbers,sort&compress]{natbib}

\input{old_macros}

\input{mf_macros}

\usepackage{joanstyle}
\usetikzlibrary{decorations.pathreplacing}

\newtheorem{assumptionalt}{Assumption}[assumption]
\newenvironment{assumptionp}[1]{
  \renewcommand\theassumptionalt{#1}
  \assumptionalt
}{\endassumptionalt}

\renewcommand{\unif}[1]{}
\renewcommand{\hpt}{H_t}

\usepackage[scr=esstix]{mathalpha}

\hypersetup{
  colorlinks,
  linkcolor={red!50!black},
  citecolor={blue!50!black},
}

\usepackage{graphicx} 
\usepackage{amsmath,amsfonts,amssymb,amsthm,bbm}
\usepackage{xcolor}
\usepackage[shortlabels]{enumitem}

\usepackage{amsmath,amssymb}
\usepackage{tikz}
\usetikzlibrary{arrows.meta,calc,backgrounds}

\definecolor{depInk}{HTML}{22364A}
\definecolor{depBlue}{HTML}{EAF1FA}
\definecolor{depBlueLine}{HTML}{5B7DA8}
\definecolor{depTeal}{HTML}{EAF5F1}
\definecolor{depTealLine}{HTML}{528675}
\definecolor{depGold}{HTML}{FFF4DD}
\definecolor{depGoldLine}{HTML}{AF843C}
\definecolor{depGray}{HTML}{F4F5F7}
\definecolor{depMuted}{HTML}{697684}
\tikzset{
  depbox/.style={draw=depBlueLine,fill=white,rounded corners=2pt,
    line width=.65pt,text=depInk,align=center,text width=3.35cm,
    minimum height=1.10cm,inner sep=5pt,font=\small},
  depassumption/.style={depbox,fill=depBlue},
  deplemma/.style={depbox,fill=white},
  depclaim/.style={depbox,draw=depTealLine,fill=depTeal},
  depresult/.style={depbox,draw=depTealLine,fill=depTeal,line width=1pt},
  depcondition/.style={depbox,draw=depGoldLine,fill=depGold},
  deparrow/.style={-{Stealth[length=2.0mm,width=1.4mm]},draw=depInk,
    line width=.65pt,rounded corners=2pt},
  depaux/.style={deparrow,draw=depMuted,densely dashed,line width=.55pt},
  deplabel/.style={font=\scriptsize,text=depMuted,fill=white,inner sep=2pt,
    align=center},
  deptitle/.style={font=\large\bfseries,text=depInk,align=center},
  depnote/.style={font=\footnotesize,text=depMuted,align=left,text width=17.6cm}
}

\renewcommand{\hpt}{H_t}

\renewcommand{\dpt}{D_t}
\renewcommand{\beps}{\bar{\eps}}
\renewcommand{\xii}{\xi_{\infty}}

\title{Uniform-in-Time Weak Propagation-of-Chaos in Shallow Neural Networks}

\author[1]{Margalit Glasgow}
\author[2]{Joan Bruna}

\affil[1]{Massachusetts Institute of Technology}
\affil[2]{Courant Institute School of Mathematics, Computing and Data Science, New York University}

\begin{document}

\maketitle

\begin{abstract}
 We consider one-hidden layer neural networks trained in the feature-learning regime using gradient descent, and relate the output of the finite-width network $f_{\rtm}$ to its infinite-width counterpart $f_{\rtmf}$, which evolves in the mean-field dynamics. 
 While constant-time horizon bounds for $\|f_{\rtmf} - f_{\rtm}\|$ may be obtained via standard Grönwall estimates, the long-time behavior of the fluctuation is a more delicate matter. Uniform-in-time bounds often rely on (local) strong convexity in the landscape or Logarithmic Sobolev inequalities present in noisy gradient dynamics.

In this work, we establish non-asymptotic weak propagation-of-chaos that holds uniformly in time, obtained by exploiting instead the convergence rate of the mean-field deterministic Wasserstein-gradient-flow dynamics. Specifically, denoting by $\loss_t$ the mean-field excess MSE loss at time $t$ and $m$ the number of neurons, under standard regularity assumptions and the condition $\int_0^\infty \loss_t^{1/2} dt =O(\log d)$, we obtain the uniform in time bound $\|f_{\rtmf}- f_{\rtm} \|^2 \lesssim \text{poly}(d) m^{-\min(1,c/6)}$ whenever $\loss_t \lesssim t^{-c}$. Our result holds in a noiseless setting and does not make any assumptions on the geometry of the landscape near the optimum, and extends seamlessly to other forms of discretization, including finite number of samples and time discretization. 
A key takeaway of our result is that whenever the convergence rate of the mean-field, population-loss dynamics is faster than $t^{-2}$, we can attain a loss of $\eps$ with only $\on{poly}(d/\eps)$ neurons, training samples, and GD steps.
\end{abstract}

\tableofcontents

\input{neurips26/intro}

\input{neurips26/setting}

\input{technical_approach}

\input{uniform_kmd_body}

\input{neurips26/uniform_experiments}
\section{Conclusions and Future Work}
The Mean-field Wasserstein gradient flow dynamics offer one of the clearest analytic windows into the training of overparameterized shallow neural networks, but their practical relevance ultimately depends on how accurately finite-width networks track this continuum limit. In this work we
studied the long-time behavior of fluctuations 
introduced by several forms of discretization of these idealized dynamics, such as finite neurons or data samples. Our main contribution is a novel uniform-in-time PoC bound that, focusing on function error, exploits the rate of convergence of the mean-field dynamics, overcoming the short-time barriers of standard Grönwall estimates. As a takeaway, whenever $\int_0^\infty \sqrt{\mathcal{L}_t} dt = O(\log d)$, one can attain population loss $\epsilon$ using $\text{poly}(d/\eps)$ resources, including neurons, training samples and gradient steps. 

Our result thus reinforces the powerful role of mean-field descriptions in providing novel algorithmic guarantees. 
While  applicable on a variety of idealized learning scenarios, our work leaves several interesting avenues for future research. 
Two concrete questions are to understand (i) whether other functionals, besides the loss functional, enjoy the same uniform-in-time PoC under our same assumptions, and (ii) the necessity of our assumption on the MF convergence rate --- by either weakening it to slower rates $\mathcal{L}_t = O(t^{-c})$, $c<2$, or else by finding a counter-example. 
In that respect, as exhibited by empirical scaling laws \cite{kaplan2020scaling, ben2026learning}, many tasks do not enjoy such fast convergence rates $\mathcal{L}_t = O(t^{-2})$,  and further, many problems exhibit dimension-dependent burn-in times to escape saddles. Both of these cases require novel technical tools to avoid exponential dependencies. 
Thinking bigger, another tantalizing question is whether our tools can be extended to other NN models where mean-field formulations have proven useful, such as Resnets \cite{chizat2025hidden} or Transformers \cite{geshkovski2025mathematical}, or even more broadly to interacting particle systems evolving under deterministic gradient dynamics. 

\paragraph{Acknowledgments:} We thank Andrea Agazzi, Shi Chen, Gerard Ben Arous, Philippe Rigollet, and Neurips reviewers for stimulating discussions and helpful feedback during the completion of this work. JB acknowledges the generous support of Flatiron Institute, which hosted his sabbatical leave. MG’s work is supported by the NSF under award 2402314 and DMS-2031883, the DARPA SIQ program, AFOSR FA9550-25-1-0375, and the NSF TILOS award.

\paragraph{AI disclosure.} ChatGPT 5.5 Pro was instrumental in converting our proof from the language of functions over $\mathbb{R}$ to the language of kernel mean discrepancies and Hilbert spaces. This mainly was as aesthetic choice, and did not affect the main result of the paper or the techniques. In particular, in May 2026, it proposed changes in Lemmas~\ref{lemma:loss_to_D},  \ref{lemma:concentration_new}, \ref{lemma:kmmd_error} and in the proof of Theorem~\ref{prop:uniform}, which we verified and edited. We also used ChatGPT 5.5 Pro to proofread and give us feedback on the manuscript. In October 2026, we used ChatGPT 6 Pro to prove the current versions of Lemmas~\ref{claim:threshold_vc} and \ref{lemma:bv_vc}, which allowed us to relax an assumption on the activation $\sigma$. We verified and edited these proofs. We also used ChatGPT 6 Pro to derive and prove the latest versions of the corollaries of the main Theorem~\ref{prop:uniform}, which is a routine calculation. We used a version of Claude code for simulations.

\bibliographystyle{alpha}
\bibliography{ref}


\appendix

\input{neurips26/apx_conc_dyn_new}
\input{neurips26/experiments_app}
\input{apx_second_layer}
\end{document}

%% file: old_macros.tex
\usepackage{comment,url,graphicx,subcaption,relsize}
\usepackage{amssymb,amsfonts,amsmath,amsthm,amscd,dsfont,mathrsfs,nicefrac}
\usepackage{float,psfrag,epsfig,color,xcolor,url,hyperref}
\usepackage{bbm,epstopdf}
\usepackage[toc,page]{appendix}

\def\ddefloop#1{\ifx\ddefloop#1\else\ddef{#1}\expandafter\ddefloop\fi}

\def\ddef#1{\expandafter\def\csname bb#1\endcsname{\ensuremath{\mathbb{#1}}}}
\ddefloop ABCDEFGHIJKLMNOPQRSTUVWXYZ\ddefloop

\def\ddef#1{\expandafter\def\csname c#1\endcsname{\ensuremath{\mathcal{#1}}}}
\ddefloop ABCDEFGHIJKLMNOPQRSTUVWXYZ\ddefloop

\def\ddef#1{\expandafter\def\csname s#1\endcsname{\ensuremath{\mathscr{#1}}}}
\ddefloop ABCDEFGHIJKLMNOPQRSTUVWXYZ\ddefloop

\def\ddef#1{\expandafter\def\csname v#1\endcsname{\ensuremath{\boldsymbol{#1}}}}
\ddefloop ABCDEFGHIJKLMNOPQRSTUVWXYZabcdefghijklmnopqrstuvwxyz\ddefloop

\def\ddef#1{\expandafter\def\csname v#1\endcsname{\ensuremath{\boldsymbol{\csname #1\endcsname}}}}
\ddefloop {alpha}{beta}{gamma}{delta}{epsilon}{varepsilon}{zeta}{eta}{theta}{vartheta}{iota}{kappa}{lambda}{mu}{nu}{xi}{pi}{varpi}{rho}{varrho}{sigma}{varsigma}{tau}{upsilon}{phi}{varphi}{chi}{psi}{omega}{Gamma}{Delta}{Theta}{Lambda}{Xi}{Pi}{Sigma}{varSigma}{Upsilon}{Phi}{Psi}{Omega}{ell}\ddefloop

\def\balign#1\ealign{\begin{align}#1\end{align}}
\def\baligns#1\ealigns{\begin{align*}#1\end{align*}}
\def\balignat#1\ealign{\begin{alignat}#1\end{alignat}}
\def\balignats#1\ealigns{\begin{alignat*}#1\end{alignat*}}
\def\bitemize#1\eitemize{\begin{itemize}#1\end{itemize}}
\def\benumerate#1\eenumerate{\begin{enumerate}#1\end{enumerate}}

\newenvironment{talign*}
 {\let\displaystyle\textstyle\csname align*\endcsname}
 {\endalign}
\newenvironment{talign}
 {\let\displaystyle\textstyle\csname align\endcsname}
 {\endalign}

\def\balignst#1\ealignst{\begin{talign*}#1\end{talign*}}
\def\balignt#1\ealignt{\begin{talign}#1\end{talign}}

\let\originalleft\left
\let\originalright\right
\renewcommand{\left}{\mathopen{}\mathclose\bgroup\originalleft}
\renewcommand{\right}{\aftergroup\egroup\originalright}

\def\tinycitep*#1{{\tiny\citep*{#1}}}
\def\tinycitealt*#1{{\tiny\citealt*{#1}}}
\def\tinycite*#1{{\tiny\cite*{#1}}}
\def\smallcitep*#1{{\scriptsize\citep*{#1}}}
\def\smallcitealt*#1{{\scriptsize\citealt*{#1}}}
\def\smallcite*#1{{\scriptsize\cite*{#1}}}

\def\mc#1{\mathcal{#1}}

\newcommand{\op}{\mathrm{op}}

\theoremstyle{plain}  
\newtheorem{remark}{\textbf{Remark}}
\newtheorem*{remark*}{\textbf{Remark}}

\def\R{\mathbb{R}}

\def\<{\left\langle} 
\def\>{\right\rangle}

\DeclareSymbolFont{rsfs}{U}{rsfs}{m}{n}
\DeclareSymbolFontAlphabet{\mathscrsfs}{rsfs}

\providecommand{\arccos}{\mathop\mathrm{arccos}}

\ifdefined\nonewproofenvironments\else
\ifdefined\ispres\else
\newtheorem{theo}{Theorem}

\newtheorem{theorem}{Theorem}
\newtheorem{lemma}[theo]{Lemma}
\newtheorem{corollary}[theo]{Corollary}
\newtheorem{definition}[theo]{Definition}

\newtheorem{example}{Example}

\newtheorem*{theorem*}{\textbf{Theorem}}

\renewenvironment{proof}{\noindent\textbf{Proof.}\hspace*{.3em}}{\qed\\}
\newenvironment{proof-sketch}{\noindent\textbf{Proof Sketch}
  \hspace*{1em}}{\qed\bigskip\\}
\newenvironment{proof-idea}{\noindent\textbf{Proof Idea}
  \hspace*{1em}}{\qed\bigskip\\}
\newenvironment{proof-of-lemma}[1][{}]{\noindent\textbf{Proof of Lemma {#1}}
  \hspace*{1em}}{\qed\\}
\newenvironment{proof-of-theorem}[1][{}]{\noindent\textbf{Proof of Theorem {#1}}
  \hspace*{1em}}{\qed\\}
\newenvironment{proof-attempt}{\noindent\textbf{Proof Attempt}
  \hspace*{1em}}{\qed\bigskip\\}
\fi
\newtheorem{observation}[theo]{Observation}
\newtheorem{claim}[theo]{Claim}
\newtheorem{assumption}{Assumption}

\fi
\makeatletter
\@addtoreset{equation}{section}
\makeatother

\hypersetup{
  colorlinks,
  linkcolor={blue!50!black},
  citecolor={blue!50!black},
  urlcolor={blue!80!black}
}



%% file: mf_macros.tex
\newcommand{\on}{\operatorname}
\newcommand{\md}{\mathcal{D}}

\newcommand{\sd}{\mathbb{S}^{d-1}}

\newcommand{\eps}{\epsilon}
\newcommand{\nspace}{\mathcal{S}}

\global\long\def\unif#1{\textcolor{blue}{\textbf{[Unif:} #1\textbf{]}}}

\newcommand{\rtmf}{\rho_t^{\textsc{MF}}}
\newcommand{\rsmf}{\rho_s^{\textsc{MF}}}
\newcommand{\rimf}{\rho_{\infty}^{\textsc{MF}}}

\newcommand{\brt}{\bar{\rho}^m_t}

\newcommand{\brz}{\bar{\rho}^m_0}

\newcommand{\rtm}{\hat{\rho}^m_t}
\newcommand{\rsm}{\hat{\rho}^m_s}
\newcommand{\rzm}{\hat{\rho}^m_0}

\newcommand{\dit}{\Delta_t(i)}

\newcommand{\djt}{\Delta_t(j)}

\newcommand{\diz}{\Delta_0(i)}

\newcommand{\jmax}{J_{\text{max}}}

\newcommand{\hpt}{H_t^{\perp}}

\newcommand{\dpt}{D_t^{\perp}}

\newcommand{\xii}{\xi^{\infty}}
\newcommand{\hxiti}{\hat{\xi}_t(w_i)}
\newcommand{\bwti}{\xi_t(w_i)}
\newcommand{\bwtj}{\xi_t(w_j)}
\newcommand{\bwzi}{w_i}
\newcommand{\bwii}{\xii(w_i)}

\newcommand{\creg}{C_{\text{reg}}}

\newcommand{\teps}{\tilde{\eps}}
\newcommand{\cfunc}{C_{V}}
\newcommand{\cinf}{C_{\ell_{\infty}}}
\newcommand{\ctwo}{C_{\ell_{2}}}
\newcommand{\beps}{\eps}
\newcommand{\tstar}{(4 \creg \ctwo \cinf^2 I_S\beps_m)^{-1/3}}
\newcommand{\loss}{\mathcal{L}}


%% file: neurips26/intro.tex
\section{Introduction}

\paragraph{Feature Learning in Shallow NNs:} The defining characteristic of neural networks is their ability to automatically learn useful representations out of high-dimensional data, which can then be transferred to downstream tasks. The simplest instance is given by one hidden-layer neural networks, which construct function approximations $f : \mathbb{R}^d \to \mathbb{R}$ of the form 
\begin{equation}
\label{eq:shallow_model}
    f(x) = \frac1m \sum_{i=1}^m \sigma(w_i^{\top}x)~,~w_1, \ldots, w_m \in \nspace,
\end{equation}
where $\sigma$ is a non-linear activation function, and $\nspace \subseteq \R^d$.

While \eqref{eq:shallow_model} is an idealized model, far from the bleeding edge of modern architectures, its non-asymptotic learning guarantees under gradient descent remain largely open. As expected, the difficulty comes from the non-convexity of the loss in this model: 
$\loss(w_1, \ldots, w_m) = \mathbb{E}_{x, y \sim \md}[ (f(x) - y)^2] $. 
Thanks to the permutation symmetry of neurons in \eqref{eq:shallow_model}, the gradient-flow training dynamics admit an Eulerian description in terms of the empirical measure $\hat{\rho}^{m} = \frac1m \sum_{i=1}^m \delta_{w_i} \in \mathcal{P}(\nspace)$. In \cite{chizat2018global, mei2018mean, rotskoff2018neural, sirignano2020mean}, the resulting dynamics are shown to be a Wasserstein Gradient Flow (WGF) for the associated functional 
$\mathcal{L}(\rho) := \mathbb{E}_x[ (f_\rho(x) - y)^2]$, where $f_\rho(x) := \mathbb{E}_{w \sim \rho} [\sigma(w^{\top}x)]$. By leveraging the convexity of this objective in $\rho$, global asymptotic convergence to minimizers was established under appropriate conditions \cite{chizat2018global, petit2026global} in the \emph{mean-field} over-parametrized limit $m \to \infty$. \footnote{We remark that these results are qualitatively different from the overparameterized neural tangent kernel (NTK) limit \cite{jacot2018neural,du2018gradient,allen-zhu2019convergence,zou2020gradient}, where global convergence is also attained, but there is no feature learning.} More recently, quantitative $\on{poly}(1/\eps)$ local convergence rates were attained for mean-field ReLU networks, under assumptions on the smoothness of the measure minimizing the loss~\cite{chizat2026quantitative}. 


The key question is then to understand \textbf{under which conditions can one bring these mean-field convergence guarantees to a finite-width network}, bypassing known negative results (eg \cite{goel2020superpolynomial}). This naturally raises the question of stability of the above WGF dynamics to particle discretization (i.e. finite width), which is the main focus of this work. 


\paragraph{Propagation-of-Chaos, Coupling, and Gronwall's Inequality:} 
Given a probability measure $\rho_0 \in \mathcal{P}(\nspace)$ from which the neurons are initialized, the Monte-Carlo iid discretization $\rzm = \frac1m \sum_{i\leq m} \delta_{w_i}$, with $w_i \sim \rho_0$, satisfies $\mathbb{E}_{\rzm \sim \rho_0^{\otimes m} }\|f_{\rzm} - f_{\rho_0} \|^2 \lesssim 1/m$, where here $\|f\|^2 := \mathbb{E}_{x \sim \md}(f(x)^2)$.
A key question --- known as the Propagation-of-Chaos (PoC) \cite{sznitman1991topics} --- is to understand how the initial `chaos' or particle independence evolves under the WGF dynamics. That is, letting $\rtm$ denote the empirical distribution of the $m$ particles at training time $t$, and $\rtmf$ the mean-field distribution at time $t$, we seek to understand how close the joint law of any $k$-tuple of particles from $\rtm$ is to ${\rtmf}^{\otimes k}$. More relevant to the neural network setting is the notion of \emph{weak} PoC, 
which only focuses on the convergence of certain observables. In our case, we consider evolution of the function error $ \mathcal{E}(\rtmf, \rtm):=\| f_{\rtm} - f_{\rtmf} \|^2$ as a function of $t$ and $m$. 




A standard approach to establishing PoC is to consider an appropriate coupling 
between the mean-field and the empirical evolution \cite{mei2018mean, mei2019mean, de2020quantitative}. 
Specifically, we can decompose the error $\mathcal{E}(\rtmf, \rtm) \leq 
2\mathcal{E}(\rtmf, \brt) + 2\mathcal{E}( \brt, \rtm)$ in terms of an auxiliary model $\brt$ which \emph{first} evolves along the mean-field dynamics (starting from $\rho_0$) and \emph{then} discretizes using an $m$-particle iid empirical measure. As a result, the first term is at the Monte-Carlo scale $O(1/m)$, while the second term captures the commutation error between sampling a measure and evolving it along the WGF dynamics. This commutation  error can then be controlled by coupling the dynamics of $\rtm$ and $\brt$, which we describe in detail in Section~\ref{sec:coupling}.
Grönwall's inequality is then used to establish a bound of the form $\mathcal{E}(\brt, \rtm) \leq \exp(L t) O(\text{poly}(d)/m)$, where $L$ is a uniform Lipschitz bound of the gradient. 

\paragraph{Towards uniform-in-time estimates}
Because of the exponential dependence on $t$, the direct Grönwall estimate is inherently limited to short time-horizons, 
and leaves open the question of obtaining PoC guarantees for longer timescales. Longer timescales are important for several reasons. First, high dimensional problems often exhibit long ``burn-in'' times to escape saddles \cite{benarous2021online, damian2023smoothing, abbe2023sgd, bietti2023learning}. Second, convergence to small thresholds $\eps$ may require time polynomial in $1/\eps$, e.g. in problems with flat landscapes near the optimum \cite{attouch2013convergence} or when the minimizing measure lies on a continuous manifold~\cite{arbel2019maximum, chizat2026quantitative}. The latter question of small-$\eps$ convergence is the main motivation of this work.

A common strategy to attain PoC at longer timescales is to add noise in the dynamics, leading to the so-called Mean-Field Langevin dynamics \cite{hu2019mean,nitanda2022convex,chizat2022mean}. The diffusion term creates a contraction in the Wasserstein metric, 
quantified via a uniform logarithmic Sobolev inequality (LSI), leading to \textit{uniform-in-time propagation of chaos} \cite{chen2022uniform,suzuki2023convergence,kook2024sampling,nitanda2024improved0}. However, the LSI assumption often transfers the exponential dependency to the runtime \cite{suzuki2023feature,wang2024mean,mousavi2024learning,takakura2024mean}.

In this work, we take an alternative route towards obtaining uniform weak PoC guarantees. Our approach builds from \cite{glasgow2025mean}, which developed a dedicated PoC analysis beginning 
with an ODE describing the evolution of the  fluctuation $\Delta_t \in \R^{m \times d}$ (see \eqref{eq:coupling_err}) which tracks the differences between the $m$ coupled neurons in  $\rtm$ and $\brt$:  
\begin{equation}\label{eq:dyn_intro}
    \frac{d}{dt}\Delta_t = D_t \odot \Delta_t - H_t \Delta_t + \varepsilon_t + O(\|\Delta_t\|^2), \tag{$\star$}
\end{equation}
Here $D_t(\cdot)$ is a non-interacting diagonal term that amplifies fluctuations whenever a neuron is visiting non-convex regions of the landscape, and $H_t$ is a PSD interaction kernel that dissipates fluctuations across the neurons (see Defn.~\ref{def:DH}). Finally, the source term $\varepsilon_t$ `pumps' error at the Monte-Carlo scale $O(1/\sqrt{m})$. 
The focus of \cite{glasgow2025mean} was on obtaining non-asymptotic bounds on $\|\Delta_t\|$ for problems with long burn-in times that improved upon the Grönwall estimate, specifically targeting systems where the required timescale to achieve small error is $O(\text{poly}(d))$, e.g. single-index models (SIMs) with large information exponent. 

\paragraph{Our Contributions:}
In this work, under a suitable decay of the loss in the mean field system, we attain PoC guarantees that hold uniformly in time. This allows us to transfer mean-field convergence guarantees at arbitrary convergence thresholds $\eps$ to the finite-width setting. In Lemma~\ref{lemma:errdynamicstight}, we establish a key refinement of the linearization \eqref{eq:dyn_intro} that further decomposes the source term $\varepsilon_t$ into a term that becomes constant near convergence, and a higher-order term. We then show that under standard regularity assumptions, if $S=\int_0^\infty \sqrt{\loss_t}dt < \infty$, the short-time Grönwall bound $\exp( L t) O(\text{poly}(d)/m)$ can be extended to a \emph{uniform-in-$t$} bound. We state an informal version of this main result below.  
\begin{theorem}[Informal version of Theorem.~\ref{prop:uniform}]\label{thm:informal}
Suppose $S := \int_{t = 0}^{\infty} \sqrt{\loss_t} dt < \infty$, $m \geq d$, and Assumption.~\ref{assm:reg} holds for the regularity constant $\creg$. Then with probability $1 - m^{-\Theta(1)}$ over the initial sample of neurons,
\begin{align}
    \|f_{\rtm} - f_{\rtmf}\|^2 \lesssim \begin{cases}\exp(4\creg S)\frac{\on{poly}(S \creg\log(dm))d^{1/3}}{m^{1/3}} & \text{always}~, \\
    \exp(10\creg S)\frac{\on{poly}(\creg (S' + S)\log(dm))d}{m} & S' := \int_{t = 0}^{\infty} t^2\sqrt{\loss_t} dt < \infty
    \end{cases}
\end{align}
\end{theorem}
A key implication of this theorem is that whenever the mean-field dynamics has a convergence rate faster than $1/t^2$ (possibly after a burn-in time of order $O(\log(d))$), we can attain a loss of $\eps$ in a network with only $\text{poly}(d/\eps)$ neurons. 

Our proof of PoC can be viewed as a stability analysis of the mean-field dynamics, and thus our uniform-in-time result can also be extended to other sources of discretization error, beyond just finite neurons. Under the same assumptions, the formal version of this theorem gives a uniform-in-time bound that decays polynomially in the smallest of the width $m$, the number of training samples $n$, and in $\eta^{-1}$, the learning rate. Thus whenever the mean-field convergence rate of gradient flow on the population loss is faster than $1/t^2$, we can attain $\eps$ loss with $\text{poly}(d)/\eps$ neurons, training samples, and gradient descent steps. Finally, we empirically verify our convergence rate assumption on several synthetic examples in Section~\ref{sec:experiments}. We observe that in many settings, if the target measure is smooth enough, the convergence is fast enough to meet our assumption.

\begin{remark}
In the case that the ground truth $f^*(x) := \mathbb{E}_{x, y \sim \md}[y | x]$ is realizable by some distribution $\rho^*$, ie. $f^* = \mathbb{E}_{w \sim \rho^*}\sigma(\langle{w, \cdot\rangle})$, our setting is an instance of particle gradient descent on the kernel mean discrepancy (KMD) between $\rho$ and $\rho^*$ (also called maximum mean discrepancy)~\cite{arbel2019maximum, chizat2026quantitative}, with the kernel $K(w, w') := \mathbb{E}_{x \sim \md}\sigma(w^{\top}x)\sigma(w'^{\top}x)$. Indeed, the KMD loss is
$\loss(\rho) := \|\mathbb{E}_{w \sim \rho}K(w, \cdot) - \mathbb{E}_{w \sim \rho^*}K(w, \cdot)\|^2_{\mc H}$, where $\mc H$ is the RKHS associated to the kernel $K$. All of our results hold for the particle and time discretization of Wasserstein gradient flow for KMD, under appropriate assumptions on the kernel $K$ (see Assumption.~\ref{assm:smoothness}).
\end{remark}

\paragraph{Related Work}

 \noindent \textit{General Uniform-in-Time Propagation-of-Chaos:} 
   There is a rich and developing literature on uniform-in-time propagation of chaos in general interacting particle systems. \cite{delarue2025uniform} establish $O(1/m)$ weak PoC uniformly in time for weakly interacting diffusions under regularity assumptions. \cite{lacker2023sharp} obtain sharp uniform-in-time PoC rates for interacting diffusions with convex potentials and small torus interactions, while \cite{monmarche2025free} derive time-uniform log-Sobolev inequalities as a tool for uniform propagation of chaos, including sharp marginal estimates in smooth cases; see also \cite{delgadino2023phase, guillin2022uniform}. These results all rely on contraction mechanisms, where the noise in the dynamics plays an instrumental role. In contrast, our results exploit a different structural property (the decay of the energy functional), enabling uniform-in-time guarantees in the determinisitc setting. 

\noindent \textit{Mean-Field Langevin dynamics:} \cite{chizat2022mean, hu2019mean, nitanda2022convex} studied the effect of adding a diffusive term in the Wasserstein Gradient Flow dynamics arising from Shallow NNs; in particular, by leveraging the aforementioned log-sobolev contraction tools, \cite{chen2022uniform} prove uniform-in-time propagation of chaos for mean-field Langevin dynamics under functional convexity, with bounds in Wasserstein and relative entropy. \cite{suzuki2022uniform} prove a quantitative weak propagation-of-chaos result for mean-field gradient Langevin dynamics, with $O(1/m)$ finite-particle discretization error uniformly over time, explicitly motivated by infinite-width two-layer neural networks; see also \cite{wang2024mean, mousavi2024learning,takakura2024mean,nitanda2024improved0}.

Most of this work on uniform-in-time PoC exists in settings where there is a unique and stable invariant measure, and the local convergence to this measure is exponentially fast. However, there are several rotatable exceptions: in \cite{delarue2025uniform}, weak PoC is shown for the super-critical Kumamoto model which has continuum of invariant measures,
and \cite{rosenzweig2023global} attain a uniform-in-time PoC guarantee in a singular-interaction setting where the mean-field object converges at an inversely polynomial rate. We emphasize that in the context of shallow NN feature learning, one typically needs to anneal the diffusive dynamics, leading to exponential runtime \cite{takakura2024mean, wang2024mean}. 

\noindent \textit{PoC under deterministic WGF: } Besides the aforementioned \cite{glasgow2025mean}, closest to our results is \cite{chen2020dynamical}, which focuses on deterministic WGF in the asymptotic regime $\lim_{t \to \infty} \lim_{m \to \infty} \allowbreak m \mathcal{E}( \rtmf, \rtm)$, establishing uniform-in-time PoC under a similar assumption on loss decay as ours. We instead look at the more natural reverse order of the limits, which requires handling the high-order error terms in \eqref{eq:dyn_intro}; see Remark~\ref{rem:asymtotic_comp}. Finally, \cite{chizat2021sparse} studied weak uniform-in-time PoC under (deterministic) Wasserstein-Fisher-Rao dynamics for atomic targets, although with exponential dependencies in $d$. For finite times, several works \cite{mei2018mean, mei2019mean, abbe2022merged, mahankali2023beyond} have used a Grönwall-based approach to attain non-asymptotic PoC guarantees with tight dependencies on the dimension, and have applied these to show feature learning in neural networks.

Finally, propagation of chaos has been studied in several other neural network settings, for example ResNets~\cite{chizat2025hidden} and deep transformers~\cite{geshkovski2025mathematical, shi2026}, where in the latter the relevant particle system is the embeddings of the $m$ tokens in the context.

\paragraph{Notation.}  $\mathcal{P}(\Omega)$ denotes the space of probability distributions over $\Omega$. For a vector $w \in \R^d$, we let $\|w\|$ denote its 2-norm, and for a matrix $M \in \R^d \times \R^d$, we let $\|M\|$ denote its operator norm.

We will use lower-case letters ($f, g, h$) to denote functions in $L^2( \R^d, \md)$. We use Greek letters ($\Delta$, $\xi$, etc) to denote vector-valued functions $\nspace \rightarrow \mathbb{R}^d$, and upper-case letters to denote matrix-valued functions $\nspace \rightarrow \mathbb{R}^{d \times d}$ or $\nspace \times \nspace \rightarrow \mathbb{R}^{d \times d}$. 
When $\hat{\mu}$ is an empirical measure of the form $\hat{\mu}=\frac1m \sum_{i} \delta_{w_i}$, we will use the shorthand $\Lambda(i) = \Lambda(w_i)$, and denote $\mathbb{E}_i \Lambda(i) := \frac1m \sum_i \Lambda(w_i)$. 

In general, we will denote dot products and norms without explicitly specifying inner product we are using, since throughout, the Hilbert space in which objects lie should be clear. Eg. 
For $H \in L^2(\nspace \times \nspace, \mu^2, \mathbb{R}^{d \times d})$, $D \in L^2(\nspace, \mu, \mathbb{R}^{d \times d})$ and $\Lambda \in L^2(\nspace, \mu, \mathbb{R}^{d})$, we use $H\Lambda (w) := \mathbb{E}_{w' \sim \mu} H(w, w') \Lambda(w')$, and $(D\Lambda)(w) := D \odot \Lambda (w) = D(w)\Lambda(w)$. Similarly, we let $\|\Lambda\|_p := \left(\mathbb{E}_{w \sim \mu}\|\Lambda(w)\|^p\right)^{1/p}$, with the default that $p=2$ if $p$ is omitted. Further, we let $\|D\|$ and $\|H\|$ denote the operator norms $\|D\| := \sup_{w \in \nspace} \|D(w)\|$, and $\|H\| := \sup_{\|\Lambda\| \leq 1}\|H\Lambda\|$. 
Occasionally, we will write $\|V\|_{p \rightarrow q}$ to denote $\sup_{\|X\|_p \leq 1}\|VX\|_q $. Typically the relevant measure will be the measure from which the network is initialized, $\rho_0$. When one or more objects in the inner product is defined only on $\{w_i\}_{i \in [m]}$ (or equivalently on $[m]$), then the relevant measure will be $\brt$ (or equivalently $\on{unif}([m])$).
We use $f \lesssim g$ or $f = O(g)$ to denote asymptotic dominance up to a constant factor. We use $\tilde{O}$ notation to omit logarithmic factors, and use $O_C(\cdot)$ to denote arbitrary dependence on $C$.

%% file: neurips26/setting.tex
\section{Setting and Preliminaries}\label{sec:setting}

\subsection{Projected Gradient Dynamics on Neural Networks}
Consider a neural network to be parameterized by some distribution $\rho \in \mathcal{P}(\nspace)$, such that 
\begin{align*}
    f_{\rho}(x) := \mathbb{E}_{w \sim \rho} \sigma(w^{\top}x),
\end{align*}
for some activation function $\sigma$. We will be primarily interested in the case where $\nspace = \R^d$ or $\sd$.

A supervised regression problem is parameterized by an initial distribution for the network weights, $\rho_0$, and a distribution $\md$ over datapoints $(x, y) \in \mathbb{R}^{d} \times \mathbb{R}$. Given $(\rho_0, \md)$, we define $f^*(x) = \mathbb{E}_{\md}[y | x]$. 
We will train the neural network to minimize the excess squared loss 
\begin{align}
    \loss_{\md}(\rho) := \mathbb{E}_{(x, y) \sim \md} (f_{\rho}(x) - y)^2 - \mathbb{E}_{(x, y) \sim \md} (f^*(x) - y)^2 ~.
\end{align}

We study the (projected) gradient flow dynamics of $\rho$ induced by moving each particle $w \sim \rho$ in the (negative) direction of the gradient of the loss $\loss_{\md}(\rho)$, and then optionally projecting the particle back to $\nspace$. Let $P_w$ denote the orthogonal projection on the tangent space $T_w \nspace$. Thus when $\nspace = \sd$, $P_w = I - ww^{\top}$; when $\nspace = \R^d$, we have $P_w = I$.

To simplify the constant, we study the gradient flow slowed by a factor of two. Thus in our two settings of interest, we have the dynamics $\frac{d}{dt} w = \nu_{\md}(w, \rho)$, where:
\begin{align}\label{eq:GF_dynamics}
    \nu_{\md}(w, \rho) := P_w\nabla_w F_{\md}(w) - P_w\nabla_w \mathbb{E}_{w' \sim \rho} K_{\md}(w, w'),
\end{align}
and
\begin{align}\label{eq:FKdefbody}
    F_{\md}(w) := \mathbb{E}_{(x, y)\sim \md} y\sigma(w^{\top}x) \qquad \text{and} \qquad K_{\md}(w, w') := \mathbb{E}_{(x, y) \sim \md} \sigma(w^{\top}x)\sigma({w'}^{\top}x). 
\end{align}

When the data distribution $\md$ is clear from context, we will often abbreviate and drop the $\md$ subscript on $\nu_{\md}(w, \rho)$, $K_{\md}$, $F_{\md}$. Further, whenever an expectation over $x$ appears in this paper without explicit distribution, it should be interpreted being drawn from the $x$-marginal of $\md$. 

\subsection{Coupling between Mean Field and Finite-Neuron Dynamics}\label{sec:coupling}
We will study the evolution of two different learning dynamics in this paper.
\paragraph{Infinite-width, infinite-data \em mean-field \em gradient flow dynamics.}
We denote the mean-field distribution at time $t$ by $\rtmf \in \mathcal{P}(\nspace)$, where we initialize $\rho_0^{\textsc{MF}} = \rho_0$. Each particle $w \in \nspace$ in the mean-field dynamics evolves according to the infinite-data velocity $\nu(w, \rtmf) \in T_w \nspace$. $\xi_t(w) \in \nspace$ denotes the \emph{characteristic} of a particle initialized at $w$ and evolved under the mean-field dynamics, equivalently expressed in Eulerian form as a continuity equation:
\begin{align}
\textstyle
    \frac{d}{dt} \xi_t(w) &= \nu(\xi_t(w), \rtmf) ~,\qquad \xi_0(w) = w~, \label{eq:MFdyn}
\end{align}

\paragraph{Finite-width, finite-data GD dynamics.}
Let $\rtm$ denote the empirical measure defined by \em $m$ neurons \em under the (projected) gradient descent induced by the \em empirical loss \em from $n$ training samples. Let $\hat{\md}$ denote the empirical distribution of the $n$ training samples. We initialize $\rzm = \frac{1}{m}\sum_{i = 1}^m \delta_{w_i}$, where $w_i \sim \rho_0$ i.i.d. for each $i \in [m]$. Each particle in the finite dynamics evolves according to the empirical  velocity $\nu_{\hat{\md}}$ evaluated at discrete time steps which are multiples of the step size, $\eta$. 
This defines a delay differential equation in $\nspace^{\otimes m}$, whose characteristics are now denoted by $\hxiti$, and solve  
\begin{align}\label{eq:finite_dynamics}\textstyle
    \frac{d}{dt} \hxiti = P_{\hat{\xi}_t}\nu_{\hat{\md}}(\hat{\xi}_{t_{\eta}}(w_i), \hat{\rho}_{t_{\eta}}^m) \qquad \hat{\xi}_0(w_i) = w_i~,~i\in [m]~, \qquad t_{\eta} := \eta\lfloor{t/\eta}\rfloor.
\end{align}
We will study the setting where the training data are drawn i.i.d.~from a sub-Gaussian distribution with sub-Gaussian label noise (See Assumption~\ref{assm:reg}, \ref{assm:data}). 

\paragraph{Coupling the dynamics.}
Let $\brt$ be the distribution initialized at $\rzm$, but that evolves according to the dynamics $\nu(\cdot, \rtmf)$. That is, $\brt = \frac1m \sum_{i = 1}^m \delta_{\xi_t(w_i)}$. Note that $\brt$ is equivalent in distribution to a random sample of $m$ particles drawn iid from $\rtmf$. 

Now let the coupling error at neuron $w_i$ be
 \begin{align}\label{eq:coupling_err}
    \dit :=  \hxiti - \bwti \in \R^d,~i\in [m]~,
\end{align}
such that $\diz = 0$ for all $i$.



\subsection{Description of the Dynamics of $\Delta_t$}\label{sec:dyn_desc}
As described in \eqref{eq:dyn_intro}, \cite[Lemma 5]{glasgow2025mean} gave a first-order approximation to the ODE describing the dynamics of $\Delta_t$, with a source term on the scale $1/\sqrt{m}$. In Lemma~\ref{lemma:errdynamicstight} below, we obtain a refinement of this result, which improves over \cite[Lemma 5]{glasgow2025mean} in two ways. First, under appropriate regularity assumptions, we extend beyond the spherical setting they studied and to the case where $\nspace = \R^d$. Second, we refine the source term by separating it into two terms: (1) a term on the scale $1/\sqrt{m}$ that becomes constant at convergence, and (2) a smaller source term on the scale $1/m$. This refinement of the source term is key to attaining the $1/m$ rate in our uniform-in-time guarantee.

Before stating Lemma~\ref{lemma:errdynamicstight}, we review from \cite{glasgow2025mean} the two key quantities governing the dynamics of $\Delta_t$: a self-interaction term, and an interaction term. The self-interaction term is described by what we call the \em local Hessian, \em the derivative of a particle's velocity with respect to that particle's position. The part of the dynamics driven by the other $\djt$ is described by what we term the \em interaction Hessian, \em the (rescaled) derivative of a particle's velocity with respect to the other particles' position. 
\begin{definition}[Local and Interaction Hessians; cf. \cite{glasgow2025mean}]\label{def:DH}
We define the \em local Hessian \em $D_t: \nspace \to \{\R^d \rightarrow \R^d\}$ and the \em interaction Hessian \em $H_t: \nspace \times \nspace \rightarrow \{\R^d \rightarrow \R^d\}$ by
    \begin{align}
    D_t(w) &:= \nabla_{\xi_t(w)} \nu(\xi_t(w), \rtmf)P_{\xi_t(w)}\\
    H_t(w, w') &:= P_{\xi_t(w)}\nabla_{\xi_t(w')} \nabla_{\xi_t(w)} K(\xi_t(w), 
    \xi_t(w'))P_{\xi_t(w')}
\end{align}
We will also use the abbreviated notation $ D_t(i) :=  D_t(\bwzi)$, and $ H_t(i, j) :=  H_t(\bwzi, w_j)$. Note that by construction we have that $H_t$ is a PSD operator. 
\end{definition}


We make the following basic regularity assumptions on the activation function and the data. Let $\sigma^{(j)}$ denote the $j$th derivative of $\sigma$.
\begin{assumptionp}{Regularity}[Regularity Assumptions]\label{assm:reg} ~
\begin{enumerate}[{\bfseries{R\arabic{enumi}}}]
\item\label{assm:sigma}  For a constant $\creg$, the activation $\sigma$ satisfies:
    \begin{itemize}
        \item  If $\nspace = \R^d$: $\sigma, \sigma'$ and $\sigma''$ have total variation at most $\creg$. Also, $|\sigma(0)|, |\sigma'(0)|, |\sigma''(0)|, |\sigma'''|_{\infty} \leq \creg$.\footnote{In the case population-data training, where $\hat{\md} = \md$, we can remove the assumption that $\sigma$ has bounded variation; see Remark~\ref{remark:KMD_poploss} and Lemma~\ref{lemma:regtouc}.}
        \item If $\nspace = \sd$: for any subgaussian variable $X$, for $j = 0, 1, 2, 3$, $(\mathbb{E}_{X}|\sigma^{(j)}(X)|^4)^{\frac{1}{4}} \leq \creg$.
    \end{itemize}
\item \label{assm:data} The distribution $\md$ on the data $(x, y) \in \R^d \times \R$ is $\creg$-subgaussian, $\mathbb{E}_{\md}\|x\|^2 = d$, and  $\mathbb{E}_{\md}y^2 \leq \creg$.
\item \label{assm:init} The initialization $\rho_0$ is supported on the bounded set $\nspace \cap \{w : \|w\| \leq \creg\}$.
\end{enumerate}
\end{assumptionp}
Note that in the case where neurons are constrained to the sphere, \ref{assm:sigma} is quite tame: it suffices to have sufficiently fast decay in the coefficients in the polynomial expansion of $\sigma$. In the case where neurons can grow arbitrarily, \ref{assm:sigma} implies that $\sigma$ is smooth and grows no faster than linearly. This includes for example a smoothed ReLU function in the population-data setting. 
\begin{remark}\label{rem:regularity}
For user-friendliness, we have stated our assumptions as above. We remark however that in all our results, \ref{assm:sigma} and \ref{assm:data} can be replaced by the more general Assumptions~\ref{assm:smoothness}, \ref{uc:n}, \ref{uc:m} in Appendix~\ref{apx:dyn}, which are implied by \ref{assm:sigma} and \ref{assm:data} (up to polynomial factors in $\creg$). We believe with some modifications, our proof could be adapted to ReLU activations when $\nspace = \sd$; however when the neurons can get arbitrarily close to $0$, even for Gaussian data, the gradients become too unstable for our techniques to work.
\end{remark}

We introduce the control parameters 
\begin{align*}
    \eps_m :=  \frac{\creg^6 \sqrt{d}\log(dm)}{\sqrt{m}}, \qquad \eps_n := \frac{\creg^5 \sqrt{d}\log(dn)}{\sqrt{n}}, \qquad \eps_{\eta} = 2\creg^2(\eta + \eta^2). 
\end{align*}
We will show in Lemma~\ref{lemma:concentration} that with high probability, the error $\|\nu(\bwti, \rtmf) - \nu(\bwti, \brt)\|$ due to sampling only $m$ neurons is uniformly (over $i$ and $t$) bounded (roughly) by $\eps_m$. Similarly, we will show in Lemma~\ref{lemma:nconcentration} that the error $\|\nu_{\hat{\md}}(\hxiti, \rtm) - \nu(\hxiti, \rtm)\|$ due to using the empirical data distribution $\md$ is uniformly bounded by $\eps_n$.

Recall that we have defined $ K(w, w') := \mathbb{E}_x \sigma(x^\top w)\sigma(x^\top w')$, and now define
\begin{align}
    K_t(w_i, w_j) &= K_t(i, j) := K(\bwti, \bwtj).
\end{align}
Let $\nabla K_t(w, w') := P_{\xi_t(w)}\nabla_{\xi_t(w)}K(\xi_t(w), \xi_t(w'))$ and let $\kappa_t$ be the subgaussian norm of $\|w\|$ for $w \sim \rtmf$.

\begin{restatable}[Parameter-Space Error Dynamics]{lemma}{errdyn}\label{lemma:errdynamicstight}
Suppose Assumption~\ref{assm:reg} holds. With probability $1 - \min(m, n)^{-\Theta(1)}$, for all $t < \infty$ and $i \in [m]$,
$$
\frac{d}{dt}\dit = D_t(i) \dit - \mathbb{E}_{j \sim [m]}H_t(i, j) \djt -  \bm{\beta}_t(i) + {\bm{\epsilon}_{t}(i)},$$
where 
\begin{align}
    \bm{\beta}_t(i) := P_{\bwti}\left(\mathbb{E}_{j \sim [m]} \nabla K_t(i, j) - \mathbb{E}_{w \sim \rho_0}\nabla K_t(w_i, w) \right),
\end{align}
and
$\|{\bm{\eps}_{t}}(i)\| \leq \eps_n + \kappa_t^2\eps_{\eta} + 2\kappa_t\\log(e+t)\|\dit\|\eps_m + 2\creg\left(\|\dit\|^2 + \mathbb{E}_j\|\Delta_j\|^2\right)$, $\|\bm{\beta}_t(i)\| \leq \eps_m \\log(e+t)$. 
\end{restatable}
Note that on constant timescales and with $\eps_n = \eps_{\eta} =  0$, the ${\bm{\eps}_{t}}(i)$ term is on the scale $\eps_m \|\Delta_t(i)\| \approx \eps_m^2$. We will show in Lemma \ref{lemma:concentration_new} that the $\bm{\beta}_t(i)$ term converges to a constant term $\bm{\beta}_{\infty}(i)$, with a difference on the scale $\eps_m \|\bwti - \bwii\|$ which dissipates as the network converges. Because this constant term $\bm{\beta}_{\infty}(i)$ is easier to handle, this improves over the naive bound $\|\bm{\beta}_t(i) + \bm{\eps}_{t}(i)\| \lessapprox \eps_m$.

We prove Lemma~\ref{lemma:errdynamicstight} by decomposing $\frac{d}{dt}\dit \!=\! -\nu(\bwti, \rtmf) \!+\! \nu(\hxiti, \rtm)$ into five differences (see Figure~\ref{fig:dynamics}), and separating the first order terms (in $\Delta_t$) from higher order terms in these differences. We defer the proof of Lemma~\ref{lemma:errdynamicstight} to Appendix~\ref{apx:dyn}.




%% file: technical_approach.tex
\section{Technical Approach}
The first key idea in our proof of Theorem~\ref{thm:informal} --- which has been previously used in \cite{chen2020dynamical} --- is the observation that as the loss decreases, $D_t$ becomes small, and thus the dynamics of $\Delta_t$ in Lemma~\ref{lemma:errdynamicstight} become nearly dissipative. Precisely, we have the following lemma, which is proved via a simple application of Cauchy-Schwarz in Appendix~\ref{sec:uniform}.

\begin{lemma}[See Lemma~\ref{lemma:loss_to_D} for full statement]
For all $w \in \nspace$, we have
 $  \| D_t(w)\| \leq \creg \left(\loss(\rtmf)\right)^{1/2}.$
\end{lemma}
Recall from Lemma~\ref{lemma:errdynamicstight} that --- omitting higher order terms and any error form $\eps_n$ and $\eps_{\eta}$--- we have 
\begin{align}
    \frac{d}{dt}\Delta_t = D_t \odot \Delta_t - H_t \Delta_t + \bm{e}_t,
\end{align}
where $\|\bm{e}_t\| = \|\bm{\eps}_t + \bm{\beta}_t\| \lessapprox \eps_m$. Since $H_t$ is PSD, we have $\lambda_{\max}(D_t - H_t) \leq \creg \sqrt{\loss(\rtmf)}$. Thus if $S := \int_{s = 0}^{\infty} \sqrt{\loss(\rsmf)} < \infty$, by Gronwall's inequality, we attain that for all $t$,
\begin{align}
    \|\Delta_t\| \leq \exp\left(\creg S\right)\int_{s = 0}^t \|\bm{e}_s\|ds \leq  \exp\left(\creg S\right)t \eps_m \approx \frac{t}{\sqrt{m}}.
\end{align}
Unfortunately, because of the linear dependence on $t$, this approach does not suffice to give a uniform bound on $\|\Delta_t\|$. However, our refined dynamics in Lemma~\ref{lemma:errdynamicstight} give more control over the error term $\bm{e}_t$. Up to higher order terms, we have that $\bm{e}_t =  \bm{\beta}_{t} \xrightarrow[t\to\infty]{} \bm{\beta}_{\infty}$, and $\bm{\beta}_{t}^{\top}H^{\dagger}_{t}\bm{\beta}_t \lesssim \eps_m^2$. To see this, observe that with $V_t(x, w) := \sigma'(\xi_t(w)^\top x)x$, we have that $H_t = V_t^{\top}V_t$, and $\bm{\beta}_{t} = V_{t}^{\top}g_{t}$, where matrix/vector multiplications are over $L_2(\md)$, and $g_t = f_{\brt} - f_{\rtmf}$. Since $\brt$ is an i.i.d. sample from $\rtmf$, we have that $\|g_t\| \lesssim 1/\sqrt{m}$, and thus $\bm{\beta}_{t}^{\top}H^{\dagger}_{t}\bm{\beta}_t = g_t V_t (V_t^{\top}V_t)^{\dagger}V_t^{\top}g_t \leq \|g_t\|^2 \lesssim 1/m$.

Nevertheless, one can see from the following example that even if $D_t = 0$, and $\bm{e}_t \equiv \bm{e}$ and $H_t \equiv H$ are constant (which is the case at convergence), and $\bm{e}^{\top}H^{\dagger}\bm{e} \leq \eps^2$ it is still possible that $\|\Delta_t\| \rightarrow \infty$.
\begin{example}[Intuition for $\|\Delta_t\|$ diverging.]
Consider the system $\frac{d}{dt}X_t = - H X_t + \bm{e}$ with $X_0 = 0$. Then for any $M$, there exists a PSD $H$ with $\|H\| \leq 1$ and an $\bm{e}$ with $\bm{e}^{\top}H^{\dagger}\bm{e} \leq \eps^2$ such that for some $t$, $\|X_t\| \geq M$.
\end{example}
\begin{proof}
Let $H = \sum_{i = 1}^B \lambda_i v_iv_i^{\top}$ with $\|v_i\| = 1$ and $\bm{e} := \sum_i v_i e_i$, such that we have the closed form 
\begin{align}\label{eq:X_exp}
    X_t = \int_{s = 0}^t \exp(-(t-s)H)\bm{e} = \sum_{i = 1}^B \int_{s = 0}^t\exp(-\lambda_i (t-s))v_i e_i
\end{align}
Choosing $\lambda_1 > 0$, $e_1 = \sqrt{\lambda_1}\eps$, and the rest of the $e_i = 0$, we have $\bm{e}^{\top}H^{\dagger}\bm{e} = \eps^2$ and
\begin{align}
    X_t = \int_{s = 0}^t\exp(-(t-s)\lambda_1)v_1 e_1 = \frac{1}{\lambda_1}(1 - \exp(-t\lambda_1))v_1 \sqrt{\lambda_1}\eps
\end{align}
Choosing $\lambda_1 = \frac{\eps^2}{4M^2}$, $t = \frac{4M^2}{\eps^2}$ yields $\|X_t\| = 2M(1 - \exp(-1)) \geq M$.
\end{proof}

This obstacle suggests that a strong form of  propagation of chaos -- in which the Wasserstein distance between $\rtmf$ and $\rtm$ is bounded uniformly in time, may not be attainable when the spectrum of $H_{\infty}$ decays to $0$. In settings where $\rho^*$ is non-atomic and thus the landscape is not locally strongly convex, we in fact expect that $H_{\infty}$ will have an infinite spectrum with a positive sequence converging to $0$.\footnote{$H_{\infty}$ necessarily has an infinite spectrum whenever $\sigma$ is analytic and non-polynomial, and the support of $\rimf$ and $\md$ contain a non-empty open set. Further, it is a standard fact that compact PSD operators can only have accumulation points in the spectrum at $0$. Even in the ERM setting, where the data distribution is atomic, the lowest non-zero eigenvalue of $H_{\infty}$ will typically decay in $n$.}

Fortunately, this behavior does not present an obstacle when we consider the weak propagation of chaos, namely $\|f_{\rtmf} - f_{\rtm}\|^2$. We will need the following lemma.
\begin{restatable}[cf. Lemma~14 in \cite{glasgow2025mean}]{lemma}{Hferr}\label{lemma:function_err}
With probability $1 - m^{-\Theta(1)}$ over $\rzm$, for any $t \in [0, \infty)$ we have
\begin{align}
\mathbb{E}_x (f_{\rtmf}(x) - f_{\rtm}(x))^2 &\leq 2\Delta_t^\top  \hpt \Delta_t  + \frac{\kappa_t^4\log(m)}{m}  + O\left(\creg^2 \|\Delta_t\|_4^4\right),
\end{align}
where $\kappa_t$ is the subguassian norm of $\|w\|$ for $w \sim \rtmf$.
\end{restatable}
Using this lemma --- up to higher order terms --- it suffices to bound $\Delta_t^{\top} H_t \Delta_t$ uniformly in $t$. We return to the example above to give intuition for why this is possible when $\bm{e}_t$ is constant.
\begin{example}[Intuition for why $\|\Delta_t^{\top}H_t \Delta_t\|$ can be bounded.]
In any system $\frac{d}{dt}X_t = - H X_t + \bm{e}$ with $X_0 = 0$, where $\|\bm{e}^{\top}H^{-1}\bm{e}\| \leq \eps^2$ and $H$ is PSD, for any $t < \infty$, we have $X_t^{\top}HX_t \leq \eps^2$.
\end{example}
\begin{proof}
Plugging in \eqref{eq:X_exp}, we have
\begin{align}
    X_t^{\top}HX_t = \sum_{i = 1}^B \lambda_i e_i^2 \left(\int_{s = 0}^t \exp(-\lambda_i(t-s))\right)^2 = \sum_{i = 1}^B \lambda_i e_i^2\left(\frac{1}{\lambda_i}\left(1 - \exp(-t\lambda_i)^2 \right)\right) \leq \sum_i \frac{1}{\lambda_i}e_i^2 \leq \eps^2.
\end{align}
\end{proof}
While the idea above is promising, bounding $\Delta_t^{\top} H_t \Delta_t$ uniformly in time is still challenging due to the higher order terms that appear in Lemma~\ref{lemma:function_err}, which we expect will be unbounded as $t \rightarrow \infty$. Fortunately, if the loss decays fast enough (faster than $t^{-6}$), we can show these terms are small enough up to some $t^*_m \approx m^{1/6}$, when the loss $\loss(\rtmf)$ is on order $\frac{1}{m}$. If we can show that the error $\|f_{\hat{\rho}_{t^*_m}^m} - f_{\rho_{t^*_m}^{\textsc{MF}}}\|^2 \approx \frac{1}{m}$ at that time, then using the fact that that the loss $\loss(\rtm)$ is non-increasing, we have that for any $t \geq t^*_m$,
\begin{align}
    \|f_{\rtm} - f_{\rtmf}\| &\leq \|f_{\rtm} - f^*\| + \|f^* - f_{\rtmf}\| \leq \sqrt{\loss(\hat{\rho}_{t^*_m}^m)} + \sqrt{\loss(\rho_{t^*_m}^{\textsc{MF}})}\\
    &\leq \|f_{\hat{\rho}_{t^*_m}^m} - f_{\rho_{t^*_m}^{\textsc{MF}}}\| + 2\sqrt{\loss(\rho_{t^*_m}^{\textsc{MF}})} \approx  1/\sqrt{m}.
\end{align}
If the loss decays faster than $t^{-2}$ but slower than $t^{-6}$, we can perform a similar argument, but the final uniform bound on $\|f_{\rtm} - f_{\rtmf}\|^2$ will be on the order of $\loss(\rho_{t^*_m}^{\textsc{MF}})$, for $t^*_m \approx m^{1/6}$.

We now have all the ideas in place for proving Theorem~\ref{thm:informal}. It remains to handle the higher order terms, and the fact that $H_t$ and $\bm{e}_t$ are not actually constant. Our key idea here is to bound the objects $\|\Delta_t\|$, and $\|\Delta_t\|_{\infty}$, and $\|f_{\rtm} - f_{\rtmf}\|$ through a careful inductive argument. In the following section, we state the formal version of our main theorem and its proof.


%% file: uniform_kmd_body.tex
\section{Formal Statement and Proof of Main Result on Uniform in Time PoC}

\subsection{Formal Theorem Statement}
Before stating the formal version of our main theorem, we remark that whenever $f^*$ is realizable by a distribution $\rho^* \in \mathcal{P}(\nspace)$, i.e., $f^*(x) = \mathbb{E}_{w \sim \rho^*}\sigma(w^{\top}x)$,
our neural network loss is equivalent to that of a Kernel Mean Discrepancy problem~\cite{chizat2026quantitative, arbel2019maximum} with the kernel $K(w, w') = \mathbb{E}_{x \sim \md} \sigma(w^{\top}x)\sigma({w'}^{\top}x)$. 
Let us first verify that under our loss decay assumptions, we are indeed in this realizable setting.
Let $m_{\rho} := \mathbb{E}_{w' \sim \rho} K(\cdot, w')$, and let $\mc H$ be the RKHS generated by the kernel $K$, with inner product $\langle{,}\rangle_{\mc H}$. 


\begin{lemma}[Reduction to Kernel Mean Discrepancy]\label{fact:kmd}
Suppose that $\int_0^\infty \sqrt{\loss_t} dt < \infty$
and Assm.~\ref{assm:reg} or Assm.~\ref{assm:smoothness} hold so that the characteristics in \eqref{eq:MFdyn} are well-posed. 
Then the characteristics  admit a limit $\xi_t(w)  \to \xi_{\infty}(w)$ $\rho_0$-a.e as $t \to \infty$. 
As a result, defining $\rho^* := (\xi_{\infty})_\#\rho_0 $, we have $\rho_t \to \rho^*$ in $W_2$, and $f^* = f_{\rho^*}$ in $L^2(\mathcal{D})$. 
Finally,
we have $\|f_{\rho}-f_{\rho'}\|^2 = \|m_{\rho} - m_{\rho'}\|^2_{\mc H}$, and in particular, 
$\loss(\rho) := \|m_{\rho} - m_{\rho^*}\|^2_{\mc H}$.
\end{lemma}
\begin{proof}
Recall the continuity equation from \eqref{eq:MFdyn}. 
Define 
$\mathcal{B}(t):= \| \nu(\cdot, \rtmf) \|_{L^2(\rtmf)}$. Since $\nu(w, \rtmf)$ is the gradient of the first-variation of $\mathcal{L}$, we have the energy dissipation 
$$\frac{d}{dt} \mathcal{L}_t = - 2 \mathcal{B}(t)^2~,$$
and thus
$$\int_a^b \mathcal{B}(t)^2 dt = \frac12(\loss_a - \loss_b)$$ 
for any $a < b$. 
We now claim that 
\begin{align}
\label{eq:dyadic}
\int_0^\infty \mathcal{B}(t) dt < \infty~.    
\end{align}
Indeed, observe that 
\begin{align}
    \int_0^\infty \mathcal{B}(t) dt & = \sum_{k=0}^\infty \int_k^{k+1} \mathcal{B}(t) dt \\
    &\leq \sum_k \left( \int_k^{k+1} \mathcal{B}(t)^2 dt \right)^{1/2} = \sum_k \sqrt{\loss_{k} - \loss_{k+1}} \\
    & \leq \sqrt{\loss_0} + \int \sqrt{\loss_t} dt  < \infty~.
\end{align}


Now, from the characteristic flow representation 
$\rho_t = (\xi_t)_\#\rho_0 $, where $\xi_t$ solves 
$\dot{\xi}_t(w) = \nu( w, \rtmf)$, we have, for any $s < t$, 
\begin{align}
    \| \xi_t(w) - \xi_s(w) \| & \leq \int_s^t \| \nu(\xi_r(w); \rho_r^{\textsc{MF}}) \| dr~.
\end{align}
Taking expectations w.r.t. $\rho_0$ then yields
\begin{align}
\| \xi_s - \xi_t \|_{L^2(\rho_0)} & \leq \int_s^t \mathcal{B}_r dr~.
\end{align}
From the previous argument, we have that $(\xi_s)$ is a Cauchy sequence in $L^2(\rho_0)$, and therefore there exists $\xi_\infty \in L^2(\rho_0)$ such that $\xi_t \to \xi_\infty$. 
Moreover, by Fubini, we have 
\begin{align}
    \int \left( \int_0^\infty \left| \dot{\xi}_t(w) \right| dt \right)\rho_0(dw) &= \int_0^\infty \int | \nu(\xi_t(w), \rtmf) | \rho_0(dw) dt \\
    &= \int_0^\infty \int |\nu(w, \rtmf) | \rho_t(dw) dt \\
    &\leq \int_0^\infty \mathcal{B}(t) dt~,
\end{align}
which shows that the characteristics have finite excursion $\rho_0$-a.e., and thus $\xi_t(w) \to \xi_\infty(w)$ $\rho_0$-a.e. 

Then, defining $\rho^*:= (\xi_\infty)_\# \rho_0$, using the coupling $(\xi_t(w), \xi_\infty(w))$ we have 
$$W_2( \rtmf, \rimf) \leq \int_t^\infty \mathcal{B}(r) dr \to 0 $$
as $t \to \infty$, showing that $\rtmf \to \rho^*$ in $W_2$. 
Since $\mathcal{L}_t \to 0$ by assumption, we also have $f^* = f_{\rho^*}$ in $L^2(\mathcal{D})$. 
Finally, we have
\begin{align}
    \|f_{\rho}-f_{\rho'}\|^2
    &=
    \mathbb{E}_{x\sim\md}
    \left[
        \left(
            \mathbb{E}_{w\sim\rho}\sigma(w^\top x)
            -
            \mathbb{E}_{w\sim\rho'}\sigma(w^\top x)
        \right)^2
    \right] \\
    &=
    \mathbb{E}_{w,w'\sim\rho}K(w,w')
-2\mathop{\mathbb{E}}\limits_{\substack{w\sim\rho\\ w'\sim\rho'}}K(w,w')
    +\mathbb{E}_{w,w'\sim\rho'}K(w,w').
\end{align}
On the other hand, since $m_{\rho} = \mathbb{E}_{w\sim\rho}K(w,\cdot)$,
the reproducing property gives
\begin{align}
    \|m_{\rho}-m_{\rho'}\|_{\mc H}^2
    &=
    \mathbb{E}_{w,w'\sim\rho}K(w,w')
    -2\mathop{\mathbb{E}}\limits_{\substack{w\sim\rho\\ w'\sim\rho'}}K(w,w')
    +\mathbb{E}_{w,w'\sim\rho'}K(w,w').
\end{align}
The two expressions are identical, and therefore we have the desired result $\|f_{\rho}-f_{\rho'}\|^2  = \|m_{\rho}-m_{\rho'}\|_{\mc H}^2$. The final line follows because $\loss(\rho) = \mathbb{E}_{(x, y) \sim \md}(y - f_{\rho}(x))^2 - \mathbb{E}_{(x, y) \sim \md}(y - f^*(x))^2 = \mathbb{E}_{x \sim \md}(f_{\rho}(x) - f^*(x))^2$.
\end{proof}

We note that the integrability condition $\int \sqrt{\loss_t}dt < \infty$ can be relaxed  
to a tail decay assumption $\loss_t = O(t^{-c})$, $c>1$ \footnote{using a dyadic argument to control $\int_T^\infty B(t)dt$ from $\sum_k T^{1/2}2^{k/2}\left(\int_{2^k T}^{2^{k+1}T} B(t)^2 dt\right)^{1/2}= O(T^{(1-c)/2})$}, which is weaker than what we require in Theorem \ref{prop:uniform}.

We now state the full version of Theorem~\ref{thm:informal}. Instead of assuming Assumption~\ref{assm:reg}, we assume Assumptions~\ref{assm:smoothness}, \ref{uc:n}, and \ref{uc:m}. In Appendix~\ref{apx:dyn}, we show that up to a polynomial factor in $\creg$, Assumption~\ref{assm:reg} implies these three assumptions, and that Lemma~\ref{lemma:errdynamicstight} holds with high probability under these three assumptions (See Lemmas~\ref{lemma:regtouc} and \ref{lemma:errdynamicstight_assm}). Let $\bar{\nspace}$ denote the convex hull of the space $\nspace$. Thus, in the case where $\nspace = \sd$, we have $\bar{\nspace} = \{w \in \R^d: \|w\| \leq 1\}$, and when $\nspace = \R^d$, then $\bar{\nspace} = \R^d$.

\begin{restatable}[Kernel Smoothness]{assumption}{assmsmooth}\label{assm:smoothness}
For some $\creg \geq 1$, the kernel $K$ satisfies for all $w, w' \in \bar{\nspace}$:
\begin{align}
    K(w, w) &\leq \creg(1 + \|w\|)^2~,\\
    \left\|\nabla_w K(w, w')\right\|, \left\|\nabla^2_{w} K(w, w')\right\|, \left\|\nabla^3_w K(w, w')\right\| &\leq \creg/11(1 + \|w'\|)~,\\
    \left\|\nabla_w \nabla_{w'}  K(w, w')\right\|,\left\|\nabla^2_w \nabla_{w'} K(w, w')\right\|, \left\|\nabla^2_w \nabla^2_{w'} K(w, w')\right\| &\leq \creg/11.
\end{align}
Further, 
\begin{align}
    \left\|\nabla_w F(w)\right\|, \left\|\nabla^2_{w} F(w)\right\|, \left\|\nabla^3_{w} F(w)\right\| \leq \creg/11,
\end{align}
or alternatively, for the case of kernel mean discrepancy, $\|w\|$ for $w \sim \rho^*$ is subgaussian. Assume also that the initialization $\rho_0$ is supported on the bounded set $\nspace \cap \{w : \|w\| \leq \creg\}$.
\end{restatable}
It is straightforward to show using Holder's inequality that Assumption~\ref{assm:smoothness} holds under Assumption~\ref{assm:reg}, up to a polynomial factor in $\creg$. As an example, for $\nabla_w^3 K(w, w')$, we have
\begin{align}
   \frac{1}{1 + \|w'\|}\left\|\nabla^3_{w}K'(w, w')\right\|_{op} &\leq  \frac{1}{1 + \|w'\|}\sup_{v_1, v_2, v_3 \in \mathbb{S}^{d-1}}\mathbb{E}_x \sigma'''(w^\top x)\sigma(w'^\top x)(v_1^\top x)(v_2^\top x)(v_3^\top x)\\
    &\leq \sup_{z, z' \in \bar{\nspace}}\left(\mathbb{E}_x |\sigma'''(z^\top x)|^4\right)^{1/4}\frac{\left(\mathbb{E}_x |\sigma(z'^\top x)|^4\right)^{1/4}}{1 + \|z'\|}\sup_{v \in \mathbb{S}^{d-1}}\left(\mathbb{E}_x|(v^\top x)|^6\right)^{1/2}\\
    &\leq O(\creg^5).
\end{align}
The final line follows directly from \ref{assm:init}.

\begin{restatable}[Uniform Convergence over empirical data sample]{assumption}{ucn}\label{uc:n}
Let
\begin{align}
    \mathcal{F} &:= \left\{f_{w, w', v}((x, y)):= \frac{1}{1 + \|w'\|}(y - \sigma(w'^{\top}x))\sigma'(w^{\top}x)(x^{\top}v) : w, w' \in \nspace, v \in \sd\right\}~,\\
    \mathcal{F}' &:= \left\{f_{w, w'}((x, y)):= (y - \sigma(w^{\top}x))(y - \sigma(w'^{\top}x))  : w, w' \in \nspace\right\}~.
\end{align}
We have the following uniform convergence bounds with probability $1 - n^{-\Theta(1)}$ over $x_i, y_i \sim \md$ i.i.d.,
\begin{align}
    \sup_{f = f_{w, w', v} \in \mc{F}} \left|\mathbb{E}_{(x, y) \sim \md} f((x, y)) - \frac{1}{n}\sum_{i= 1}^n f((x_i, y_i)) \right| &\leq \eps_n~, \\
    \sup_{f = f_{w, w'} \in \mc{F}'} \left|\mathbb{E}_{(x, y) \sim \md} f((x, y)) - \frac{1}{n}\sum_{i= 1}^n f((x_i, y_i)) \right| &\leq \eps_n~,\\
    \left|\mathbb{E}_{(x, y) \sim \md}(f^*(x)-y)^2 - \frac{1}{n}\sum_{i= 1}^n (f^*(x_i) - y_i)^2 \right| &\leq \eps_n~.
\end{align}
Also with probability $1 - n^{-\Theta(1)}$, $\forall w, w' \in \nspace$, $\|\nabla^2_w K_{\hat{\md}}(w, w')\| \leq \creg/11$, $\|\nabla_w \nabla_{w'} K_{\hat{\md}}(w, w')\| \leq \creg/11$, and $\|\nabla^2_w F_{\hat{\md}}(w, w')\| \leq \creg/11$.
\end{restatable}

\begin{restatable}[Uniform Convergence over sample of neurons]{assumption}{ucm}\label{uc:m}
Let
\begin{align}
    \mathcal{F} &:= \{f_{w, v}(z):= v^{\top}\nabla_w K(w, z) : w \in \nspace, v \in \sd\}~,\\
    \mathcal{F}' &:= \{f_{w, v}(z):= \langle{v, \nabla_w P_w \nabla_w K(w, z)\rangle} : w \in \nspace, v \in (\sd)^{\otimes 2}\}~.
\end{align}
For any distribution $\rho \in \mathcal{P}(\nspace)$, with $z \sim \rho$ is $\kappa_z$ subgaussian, we have the following uniform convergence bound with probability $1 - \eps_m^2 m^{-\Theta(1)}/(1 + t)^2$ over $z_i \sim \rho$:
\begin{align}
    \sup_{f \in \mc{F} \cup \mc{F}'} \left|\mathbb{E}_{z \sim \rho} f(z) - \frac{1}{m}\sum_{i=1}^m f(z_i) \right| \leq \eps_m \kappa_z \log(e + t).
\end{align}
\end{restatable}

\begin{theorem}[Uniform PoC for Polynomial Convergence Rates]\label{prop:uniform}
Suppose Assumptions~\ref{assm:smoothness}, \ref{uc:n}, and \ref{uc:m} hold, $\eta \leq 0.1/\creg$, and write $\loss_t := \loss(\rtmf)$. Define $S := 1 + \int_{t = 0}^{\infty} \sqrt{\loss_t} dt$, and $\tilde{S}(t) := 1 + \min\left(t, \int_{s = 0}^{\infty} \min(t, s)\sqrt{\loss_s} ds\right)$, and assume that $S < \infty$. Let $\rtmf$ be given by \eqref{eq:MFdyn}, and $\rtm$ by \eqref{eq:finite_dynamics}. Then with probability at least $1 - \min(m, n)^{-\Theta(1)}$ over $\rzm$ and the data sample, for any $t < \infty$,
\begin{align}
     \|m_{\rtm} - m_{\rtmf}\|^2_{\mc H} \leq 4\mathcal{L}_R(m, n, \eta)
\end{align}
where 
$t^*(m, n, \eta) := \frac{1}{64\creg^{3}\, S^{4/3}\, e^{5\creg S/3}}\cdot
\min\!\left(\beps_m^{-1/3},\; (\eps_n + \eps_{\eta})^{-1/4}\right) \leq m $, and 
\begin{align}
    \mathcal{L}_R(m, n, \eta) := \inf_{t \leq t^*(m, n, \eta)} 2\loss_t +  O\left(\creg^6S^2\exp(2\creg S)(\beps_m^2\tilde{S}(t)^2 + (\eps_n + \eps_{\eta})^{1/2})\right),
\end{align}
with $\beps_m := 2\creg(1+S)\log(m)\eps_m$.
\end{theorem}

\begin{corollary}\label{corr:poly_rates}
Suppose Assumptions~\ref{assm:smoothness} and \ref{uc:m} hold, $\eta = 0$, and $\md = \hat{\md}$ such that $\eps_{\eta} = \eta_n = 0$. If $\loss_t \leq \max(1, t + 1 - B)^{-c}$ for some burn-in time $1 \leq B < d$, $c > 2$
and $m \geq \Omega_{\creg, B, c}(d)$  then with probability $1 - m^{-\Theta(1)}$, 
\begin{align}\label{eq:cor_bound}
    \sup_{0 \leq t < \infty} \|m_{\rtm} - m_{\rtmf}\|^2_{\mc H}
    \;\leq\;
    \tilde{O}_{\creg}\!\left(
    \left(\frac{\exp(20 \creg (B + \frac{c}{c-2}))\, d}{m} \right)^{\min(c, 6)/6}
    \right),
\end{align}
where the $\tilde{O}(\cdot)$ hides logarithmic factors in $d$ and $m$. 
\end{corollary}
\begin{remark}\label{remark:KMD_poploss}
For the Wasserstein gradient flow on the kernel mean discrepancy problem, we do not need to assume Assumption~\ref{uc:n} since we have no data, and can treat $n$ as $\infty$. (In general, this assumption can also be omitted if we define $\rtmf$ to be the trajectory on the empirical loss, ie. setting $\md = \hat{\md}$).    
\end{remark}

\begin{remark}[Comparison to  \cite{chen2020dynamical}]\label{rem:asymtotic_comp}
The closest existing result to this is Theorem 3.5 in \cite{chen2020dynamical}, which yields uniform-in-time asymptotic PoC under the assumption $\mathcal{L}_t = O(t^{-4})$ and in the ERM setting. An important difference is that in their setting, the order of the limits in $t$ and $m$ is exchanged: the authors establish that $\lim_{t \to \infty} \lim_{m \to \infty} m \mathcal{E}(\rtmf, \rtm) \leq C_{\rho^*}$ \footnote{More precisely, they show that time-averages of the renormalized errors converge.}. 
Taking the $m$-limit first eliminates the need to handle high-order terms in the coupling expansion \eqref{eq:dyn_intro}, precluding a non-asymptotic (in $m$) control. Additionally, the asymptotic-in-$t$ result hides dependence on the smallest non-zero eigenvalue of $\lim_{t \rightarrow \infty} H_t$, which may depend on the dimension $d$, or, in an ERM setting, the number of data points $n$. Our $O(1/m)$ result gives explicit dependence on $d$ and makes no assumption on $\lim_{t \rightarrow \infty} H_t$, but requires $\loss_t = O(t^{-6})$, due to accounting for higher order terms that appear in the non-asymptotic analysis. Our result also goes beyond \cite{chen2020dynamical} in that it holds for slower rates $\loss_t = o(t^{-2})$, though in this case we attain a slower than $1/m$ PoC guarantee.
\end{remark}

\begin{remark}[Second Layer Weights]\label{rem:second_layer}
While training both second and first layer weights can obstruct propagation of chaos in settings where the weights grow exponentially large (see Section 5.3 in \cite{glasgow2025mean}), under our convergence rate assumption $\int_{t = 0}^{\infty} \sqrt{\loss_t}dt < \infty$ and standard regularity assumptions, we can show that the weights stay bounded. Thus we believe that Theorem~\ref{prop:uniform} should hold in this case too but have not verified all the details; see Appendix~\ref{apx:second_layer} for more discussion.

\end{remark}

\begin{remark}[The critical decay rate \(c=2\)]
We can prove an analog of Corollary~\ref{corr:poly_rates} in the case that $c = 2$, and thus we have the decay rate $\loss_t \leq (t + 1 - B)^{-2}$. Here, for
$m\ge \widetilde\Omega_{\creg}\!\left(
    e^{10\creg B}(1+B)^{18+10\creg}\,d
\right)$
with probability \(1-m^{-\Theta(1)}\), we can show that
\[
\sup_{t\ge0}
\left\|
    m_{\widehat\rho_t^m}-m_{\rho_t^{\mathrm{MF}}}
\right\|_{\mathcal H}^{2}
\le
\widetilde O_{\creg}\!\left(
    \left[
        \frac{e^{10\creg B}(1+B)^{12}\,d}{m}
    \right]^{\frac{1}{3+5\creg}}
\right).
\]
To prove this, we need to modify the proof of Theorem~\ref{prop:uniform} to accommodate the fact that $\int \sqrt{\loss_t} dt$ does not converge. We run the first part of the proof on a deterministic interval \([0,T]\), with
\(2B\le T\le m\), replacing \(S\) by
\(S_T:=1+\int_0^T\sqrt{L_s}\,ds\le1+B+\log(e + t)\), and $\tilde{S}(t)$ by $\widetilde S_T(t) :=1+t$. In Claim~\ref{claim:func_err_integral}, we replace $g_{\infty}$ by $g_T$. Appropriate recalculations give the above result.
\end{remark}

\begin{remark}[Burn in time growing with $d$] In many feature learning problems in high dimensions, there is a ``burn-in'' time, or search phase, of order $\log(d)$, before the loss begins to decay rapidly. This includes for example single-index models or multi-index models with information exponent or leap complexity $2$ (see eg. \cite{benarous2021online, abbe2023sgd}). In these cases, so long as the loss decays like $t^{-2}$ or faster \em after \em the burn-in time, Corollary~\ref{corr:poly_rates} guarantees a uniform-in-time PoC bound of $\on{poly}(d)/m^{1/6}$. This guarantees that networks of width $\on{poly}(d/\eps)$ can attain a loss of $\eps$.
\end{remark}

\subsection{Preliminary Lemmas and Notation}\label{sec:uniform}

\paragraph{Notation}
For a tensor $M \in (\R^d)^{\otimes k}$, let $\|M\| := \sup_{u \in (\sd)^{\otimes k}}\langle{M, u\rangle}$ denote it operator norm. Let $D \Phi [u]$ denote the directional derivative of $\Phi$ in the direction $u$. For functions of two arguments, we will sometimes use $\nabla_i$ to denote the gradient with respect to the $i$th argument; by default we will use $\nabla = \nabla_1$, or otherwise $\nabla_w$ to denote gradient with respect to $w$. For a bounded linear operator $A$, we use $A^\ast$ to denote its adjoint. 

\paragraph{Definitions} Let $\Phi(w):=K(w,\cdot)\in \mc H$. For $w \in \nspace$, define the operator $V_t(w):\R^d\to \mc H$ by
$V_t(w)u := D\Phi(\xi_t(w))[P_{\xi_t(w)}u]$. Then $H_t(w,w')=V_t(w)^\ast V_t(w')$ as an operator on $\R^d$. Indeed, for any $u,u'\in\R^d$, by the reproducing kernel property, and the fact that derivatives commute with inner products, we have
\begin{align}\label{eq:HVV}
\langle{V_t(w)u, V_t(w')u'}\rangle_{\mc H}
&= \left\langle D\Phi(\xi_t(w))[P_{\xi_t(w)}u], D\Phi(\xi_t(w'))[P_{\xi_t(w')}u']\right\rangle_{\mc H}\\
&= u^{\top}P_{\xi_t(w)}\nabla_1 \nabla_2 K(\xi_t(w), \xi_t(w'))P_{\xi_t(w')}u'\\
&= u^\top H_t(w,w')u'.
\end{align}

Let $V_t := \begin{pmatrix}V_t(w_1)| \cdots
|V_t(w_m)\end{pmatrix}$, viewed as an operator from $\ell_2([m];\R^d)$ with normalized inner product to $\mc H$.
We can expand $\bm{\beta}_{t} = V_t^\ast g_{t}$, where $g_t = m_{\brt}- m_{\rtmf}$:
\begin{align}
    \bm{\beta}_{t}(i) &= P_{\bwti}\left(\mathbb{E}_{j \sim [m]} \nabla_1 K(\bwti, \bwtj) - \mathbb{E}_{w \sim \rtmf}\nabla_1 K(\bwti, w)\right)\\
    &= P_{\bwti}\left(\nabla m_{\brt}(\bwti) - \nabla m_{\rtmf}(\bwti)\right)\\
    &= P_{\bwti}\nabla \left\langle K(\bwti, \cdot), m_{\brt}-m_{\rtmf}\right\rangle_{\mc H}\\
    &= (V_t^\ast g_t)(i).
\end{align}

Define $\kappa_{\infty}$ to be the the sub-Gaussian norm of $\sup_{t < \infty} \|\xi_t(w)\|$ for $w \sim \rho_0$. Note that this supremum is guaranteed to exist under the assumptions of Theorem~\ref{prop:uniform} because as is clear from Lemma~\ref{lemma:loss_to_D} below, $\xi_t(w)$ converges to some limit $\xi_{\infty}(w)$.

We need several preliminary lemmas to prove Theorem~\ref{prop:uniform}. They are straightforward, so we defer their proofs until after the main proof. We assume in all these lemmas that the assumptions of Theorem~\ref{prop:uniform} hold for $\creg \geq 1$, including the assumption $\int_{t = 0}^{\infty} \sqrt{\loss_t(\rtmf)}dt < \infty$. In the proof of Corollary~\ref{corr:poly_rates}, which includes a case where $\int_{t = 0}^{\infty} \sqrt{\loss_t(\rtmf)}dt$ does not converge, we show how to adapt the results and proofs of these lemmas.
\begin{lemma}[cf. Lemma 15 in \cite{glasgow2025mean}]\label{lemma:smoothness}
We have the following for any $w, w' \in \bar{\nspace}$:
\begin{enumerate}[{\bfseries{S\arabic{enumi}}}]
    \item\label{S1} $\left\|\nabla^2_{w'}P_w \nabla_w K(w, w')\right\| \leq \creg$
    \item\label{S3} $\left\|\nabla_{w'}\nabla_wP_w \nabla_w K(w, w')\right\| \leq \creg$
    \item\label{S5} For $\rho \in \mathcal{P}(\nspace)$, we have $\left\|\nabla^2_{w} \nu(w, \rho)\right\|_{op}, \left\|\nabla_{w} \nu(w, \rho)\right\|_{op}, \left\| \nu(w, \rho)\right\|_{op} \leq (2 + \kappa_{\rho})\creg$, where $\|w'\|$ for $w' \sim \rho$ is $\kappa_{\rho}$ subgaussian.
    \item \label{S0} $\left\|K(w, \cdot)\right\|_{\mc H} \leq \creg (1 + \|w\|)$
    \item\label{SD} $\left\|u^{\top}\nabla K(w, \cdot)\right\|_{\mc H} \leq \creg$ for any $u \in \sd$.
    \item\label{SDD}$ \left\|u^{\top}\nabla^2 K(w, \cdot)v\right\|_{\mc H} \leq \creg$ for any $u, v \in \sd$.
\end{enumerate}
\end{lemma}

\begin{lemma}[Controlling $D_t$, $\xi_t - \xii$ and $\dot{V}_t$ from the loss]\label{lemma:loss_to_D}
For all $w \in \nspace$,
\begin{equation}
   \| D_t(w)\| \leq \creg \left(\loss(\rtmf)\right)^{1/2}.
\end{equation}
Further, under the gradient flow dynamics $\rtmf$, for all $0 \leq t \leq r \leq \infty$ and $w \in \nspace$,
\begin{align}
    \|\xi_t(w) - \xi_r(w)\| \leq \creg \int_{s = t}^{r} \sqrt{\loss(\rsmf)}ds
    \qquad \|\dot{\xi}_t\| \leq \creg \sqrt{\loss_t}, \qquad \|\xi_t(w)\| \leq \creg\left(1 + \int_{s = 0}^t \sqrt{\loss_s}ds\right).
\end{align}

Finally, we have
\begin{align}
    \|\dot{V_t}\|_{2\to\mc H} \leq \creg^2 \sqrt{\loss(\rtmf)}.
\end{align}
    
\end{lemma}

\begin{lemma}[Concentration of $g_t - g_{\infty}$]\label{lemma:concentration_new}
Let $g_t := m_{\brt} - m_{\rtmf} \in \mc H$. With probability $1 - 3/m$, uniformly over $t < \infty$, we have 
\begin{align}
    \|g_t - g_{\infty}\|_{\mc H} \leq 2\creg^2 \sqrt{\log(e+t)} \eps_m \min\left(1 + \kappa_{\infty}, \int_{s = t}^{\infty} \sqrt{\loss(\rsmf)}ds\right),
\end{align}
and 
\begin{align}
    \|g_{\infty}\|_{\mc H} \leq \creg\kappa_{\infty}\eps_m.
\end{align}
\end{lemma}

We now state Lemma~\ref{lemma:function_err}, restated here for the reader's convenience. Note that here thanks to Observation~\ref{fact:kmd}, we replace $\|f_{\rtm} - f_{\rtmf}\|^2$ by $\|m_{\rtm} - m_{\rtmf}\|^2_{\mc H}$.
\begin{lemma}[cf. Lemma~14 in \cite{glasgow2025mean}]\label{lemma:kmmd_error}
With probability $1 - m^{-\Theta(1)}$ over $\rzm$, for any $t \in [0, \infty)$ we have
\begin{align}
\|m_{\rtm} - m_{\rtmf}\|^2_{\mc H} &\leq 2\Delta_t^\top  H_t \Delta_t  + \frac{\creg^4\kappa_{\infty}^2\log(m(e + t))}{m}  + O\left(\creg^2 \|\Delta_t\|_4^4\right).
\end{align}
\end{lemma}

\subsection{Proof of Theorems~\ref{thm:informal} and \ref{prop:uniform}}\label{sec:uniform_proof}
In Figure~\ref{fig:thm_deps}, we show a graphic of the proof dependencies in proving Theorem~\ref{prop:uniform} and its corollaries.

\begin{figure}
    \centering
\begin{tikzpicture}[x=1cm,y=1cm]
\node[deptitle] at (6.75,1.95) {A. Main theorem: proof dependencies};

\node[depassumption] (a1) at (0,0)
  {\textbf{Assumption~\ref{assm:smoothness}}\\Kernel smoothness\\and initialization};
\node[depcondition] (loss) at (4.5,0)
  {\textbf{Loss integrability}\\$\displaystyle\int_0^\infty\sqrt{L_s}\,ds<\infty$};
\node[depassumption] (a3) at (9,0)
  {\textbf{Assumption~\ref{uc:m}}\\Concentration over\\sampled neurons};
\node[depassumption] (a2) at (13.5,0)
  {\textbf{Assumption~\ref{uc:n}}\\Concentration over\\training data};

\node[deplemma] (l7) at (0,-2.15)
  {\textbf{Lemma~\ref{lemma:smoothness}}\\Smoothness of $\nu,\Phi$\\and their derivatives};
\node[deplemma] (l5) at (4.5,-2.15)
  {\textbf{Lemma~\ref{fact:kmd}}\\Limit $\rho^*$ and\\KMD representation};
\node[deplemma,text width=3.9cm] (l2) at (11.25,-2.15)
  {\textbf{Lemma~\ref{lemma:errdynamicstight} / Lemma~\ref{lemma:errdynamicstight_assm}}\\Dynamics of $\Delta_t$\\via Lemmas 20--21};

\node[deplemma] (l8) at (4.5,-4.35)
  {\textbf{Lemma~\ref{lemma:loss_to_D} (cf.\ Lemma 3)}\\Loss controls $D_t$,\\$\xi_t-\xi_r$, and $\dot V_t$};

\node[deplemma] (l10) at (0,-6.55)
  {\textbf{Lemma~\ref{lemma:kmmd_error} (cf.\ Lemma 4)}\\Function error from\\$\Delta_t^\top H_t\Delta_t$ and remainder};
\node[deplemma] (l9) at (4.5,-6.55)
  {\textbf{Lemma~\ref{lemma:concentration_new}}\\Concentration of\\$g_t-g_\infty$ and $g_\infty$};
\node[depclaim] (c11) at (9,-6.55)
  {\textbf{Claim~\ref{claim:func_err_integral}}\\Self-referential bound\\for $\|V_t\Delta_t\|_{\mathcal H}$};

\node[depclaim,text width=4.1cm] (boot) at (9,-8.75)
  {\textbf{Claims~\ref{claim:delta_t_lin}--\ref{claim:b_eps_bd} bootstrap}\\Joint control of $\|\Delta_t\|_2$,\\
   $\|\Delta_t\|_\infty$, and $\|V_t\Delta_t\|_{\mathcal H}$};

\node[depclaim,text width=4.1cm] (finite) at (6.75,-10.95)
  {\textbf{Equation~\ref{eq:lesststar}}\\Finite-horizon PoC\\for $0\le t\le t^*$};
\node[depclaim] (c14) at (13.5,-10.95)
  {\textbf{Claim~\ref{claim:non-increasing-loss}}\\Near-monotonicity of\\the finite-network loss};

\node[depresult,text width=5.8cm,minimum height=1.23cm] (thm2) at (9,-13.15)
  {\textbf{Theorem~\ref{prop:uniform}}\\Uniform-in-time weak propagation of chaos\\
   $\displaystyle\sup_{t\ge0}\|m_{\hat\rho_t^m}-m_{\rho_t^{\mathrm{MF}}}\|_{\mathcal H}^2
   \le 4L_R(m,n,\eta)$};

\node[depresult,text width=4.4cm] (thm1) at (5.5,-15.60)
  {\textbf{Theorem~\ref{thm:informal}}\\Explicit $m^{-1/3}$ bound;\\$m^{-1}$ under a weighted\\integrability condition};
\node[depresult,text width=4.4cm] (cor6) at (12.5,-15.60)
  {\textbf{Corollary~\ref{corr:poly_rates}}\\Polynomial $t^{-c}$ MF loss decay, $c>2$\\gives rate $m^{-\min(c,6)/6}$};

\begin{scope}[on background layer]
\draw[deparrow] (a1) -- (l7);
\draw[deparrow] (a1.north) -- (0,1.02) -- (11.25,1.02) -- (l2.north);
\draw[deparrow] (loss) -- (l5);
\draw[depaux] (a1.south east) -- (l5.north west);
\draw[deparrow] (a3.south) -- (l2.north west);
\draw[deparrow] (a2.south) -- (l2.north east);
\draw[depaux] (l7.east) -- (2.35,-2.15) -- (2.35,-3.33)
  -- (11.25,-3.33) -- (l2.south);
\draw[deparrow] (l5) -- (l8);
\draw[depaux] (l7.south) -- (0,-4.35) -- (l8.west);
\draw[deparrow] (l8) -- (l9);
\draw[deparrow] (l8.south west) -- (l10.north east);
\draw[deparrow] (l2.south east) -- (13.35,-3.65)
  -- (13.35,-5.35) -- (9,-5.35) -- (c11.north);
\draw[deparrow] (l9.east) -- (c11.west);
\draw[deparrow] (c11) -- (boot);
\draw[deparrow] (boot.south) -- (finite.north east);
\draw[deparrow] (l10.south) -- (0,-10.95) -- (finite.west);
\draw[depaux] (a2.east) -- (15.85,0) -- (15.85,-10.95) -- (c14.east);
\draw[depaux] (l7.west) -- (-2.35,-2.15) -- (-2.35,-9.86)
  -- (13.5,-9.86) -- (c14.north);
\draw[deparrow] (finite.south) -- (thm2.north west);
\draw[deparrow] (c14.south) -- (thm2.north east);
\draw[deparrow] (thm2.south west) -- (thm1.north);
\draw[deparrow] (thm2.south east) -- (cor6.north);
\end{scope}


\end{tikzpicture}
    \caption{Proof Dependencies of Theorem~\ref{prop:uniform} and its corollaries. Solid arrows show main proof inputs; dashed arrows show additional inputs.}
    \label{fig:thm_deps}
\end{figure}

For the proof of these theorems, we condition on the event that events in Assumptions~\ref{uc:n} hold for the data sample $\hat{\md}$, and that the event in Lemma~\ref{lemma:errdynamicstight} hold for $\hat{\md}$ and the initial sample of neurons $\rzm$. These assumptions hold with probability $1 - \min(m, n)^{-\Theta(1)}$.

\begin{proof}[Proof of Theorem~\ref{prop:uniform} and Theorem~\ref{thm:informal}]
First observe that by Lemma~\ref{lemma:loss_to_D} for all neurons, for $t \leq m$, 
\begin{align}
    \|\xi_t(w)\| &\leq \|w\| + \creg \int_{s = 0}^{t}\sqrt{\loss_s}ds \leq \creg(1 + S),
\end{align}
so $\kappa_{m} \leq \creg(1 + S)$.
By Lemma~\ref{lemma:kmmd_error}, we have for all $t \leq m$,
\begin{align}\label{eq:func_err}
\|m_{\rtm} - m_{\rtmf}\|^2_{\mc H} &\leq 2\Delta_t^\top  \hpt \Delta_t  + O\left(\frac{\creg^6(1 + S)^2\log(m)}{m}\right)  + O\left(\creg^2 \|\Delta_t\|_4^4\right).
\end{align}
Thus our main goal in this proof will be bounding $\Delta_t^{\top}H_t\Delta_t$ for $t \leq m$. Define 
\begin{align}
    \beps_m &:= 2\creg (1 + S) \log(m)\eps_m\\
    \beps &:= \creg(1 + S)\eps_n + 4\creg (1 + S)^2 \eps_{\eta}.
\end{align}

Recall that by Lemma~\ref{lemma:errdynamicstight}, for all $t$, we have
\begin{align}\label{eq:dynamics_full}
    \frac{d}{dt}\Delta_t = D_t \Delta_t - H_t \Delta_t - \bm{\beta}_t  + \bm{\eps}_{t}, 
\end{align}
where
\begin{align}
    \bm{\beta}_t(i) = \left(\mathbb{E}_{j \sim [m]} \nabla K_t(i, j) - \mathbb{E}_{w \sim \rho_0}\nabla K_t(w_i, w) \right),
\end{align}
and for all $t \leq m$, we have the bounds
\begin{align}\label{eq:eps_b_bounds}
    \|\bm{\eps}_t\|_{\infty} &\leq \beps_m \|\Delta_t\|_{\infty} +  4\creg(\|\Delta_t\|^2_{\infty}) + \beps\\
    \|\bm{\eps}_t\|_{2} &\leq \beps_m\|\Delta_t\|_2 +  4\creg(\|\Delta_t\|^2_{4}) + \beps\\
    \|\bm{\beta}_t\|_{\infty} &\leq \beps_m.
\end{align}

For $w \in \nspace$, recall that $V_t(w):\R^d\to\mc H$ is defined by $V_t(w)u=D\Phi(\xi_t(w))[P_{\xi_t(w)}u]$, so that $H_t(w,w')=V_t(w)^\ast V_t(w')$ (see \eqref{eq:HVV}). Let $V_t := \begin{pmatrix}V_t(w_1)| \cdots
|V_t(w_m)\end{pmatrix}$, viewed as an operator from $\ell_2([m];\R^d)$ with normalized inner product to $\mc H$. Recall also that we have $\bm{\beta}_{t} = V_t^\ast g_{t}$, where $g_t = m_{\brt}- m_{\rtmf}$, and $m_{\rho} := \mathbb{E}_{w' \sim \rho} K(\cdot, w')$.

Our first claim proves a bound on $\sqrt{\Delta_t^{\top}H_t \Delta_t} = \|V_t\Delta_t\|_{\mc H}$ which is self-referential.
\begin{claim}\label{claim:func_err_integral}
For any $0 \leq t \leq m$, we have
\begin{align}
    \|V_t\Delta_t\|_{\mc H} &\leq 2\creg^3 \beps_m \tilde{S}(t) + \creg^3\sqrt{d} \int_{s = 0}^{t} \sqrt{\loss_s}\|\Delta_s\|ds + \creg \int_{s = 0}^t \left(\beps_m \|\Delta_s\|_2 +  4\creg\|\Delta_s\|\|\Delta_s\|_{\infty} + \beps\right) ds.
\end{align}
\end{claim}
\begin{proof}
Lets track the evolution of $Y_t := V_t\Delta_t$. We have
\begin{align}
    \dot{Y}_t &= V_t\dot{\Delta}_t + \dot{V}_t \Delta_t = - V_tV_t^\ast Y_t - V_tV_t^\ast g_t + V_tD_t \Delta_t + V_t \bm{\eps}_t + \dot{V}_t \Delta_t,
\end{align}
so 
\begin{align}
    \frac{d}{dt} (Y_t + g_{\infty}) = - V_tV_t^\ast(Y_t + g_{\infty}) + V_tV_t^\ast(g_{\infty} - g_{t}) +  V_tD_t \Delta_t + V_t \bm{\eps}_t + \dot{V}_t \Delta_t.
\end{align}
Thus 
\begin{align}
    &\|V_t\Delta_t\|_{\mc H} = \|Y_t\|_{\mc H} \leq \|Y_t + g_{\infty}\|_{\mc H} + \|g_{\infty}\|_{\mc H} \\
    &\quad\leq \int_{s  = 0}^t \|V_sV_s^\ast\|_{\mc H \to \mc H}\|g_s - g_{\infty}\|_{\mc H} ds + \int_{s  = 0}^t \|V_s\|_{2 \rightarrow \mc H}(\|D_s\|\|\Delta_s\| + \|\bm{\eps}_s\|)ds + \int_{s  = 0}^t \|\dot{V}_s\|_{2 \to \mc H}\|\Delta_s\| ds + \|g_{\infty}\|_{\mc H}.
\end{align}
Now $\|V_t^\ast\|_{\mc H \to 2} = \|V_t\|_{2 \to \mc H} = \sqrt{\|V_t^\ast V_t\|} = \sqrt{\|H_t\|} \leq \sqrt{\creg}$. Further, using Lemma~\ref{lemma:loss_to_D}, we have that $\|D_t\|_{\infty} \leq \creg\sqrt{\loss_t}$  and $\|\dot{V_t}\|_{2 \rightarrow \mc H} \leq \creg^2 \sqrt{\loss_t}$. Finally, Lemma~\ref{lemma:concentration_new} yields $\|g_t - g_{\infty}\|_{\mc H} \leq 2\creg^2 \beps_m \min\left(1, \int_{u = t}^{\infty}\sqrt{\loss_u}du\right)$ and $\|g_{\infty}\|_{\mc H} \leq \creg \beps_m$.

It follows that 
\begin{align}
    \|V_t\Delta_t\|_{\mc H} &\leq \creg \int_{s  = 0}^t 2\creg^2 \beps_m \min\left(1, \int_{u = s}^{\infty}\sqrt{\loss_u}du\right) ds\\
    &\qquad + (\sqrt{\creg} + 1)\int_{s  = 0}^t \creg^2  \sqrt{\loss_s}\|\Delta_s\|ds + \sqrt{\creg}\int_{s = 0}^t \|\bm{\eps}_s\| ds + \creg\beps_m\\
    &\leq 2\creg^3 \beps_m \int_{s  = 0}^t \min\left(1, \int_{u = s}^{\infty}\sqrt{\loss_u}du\right) ds + \creg^3 \int_{s = 0}^{t} \sqrt{\loss_s}\|\Delta_s\|ds + \creg \int_{s = 0}^t \|\bm{\eps}_s\| ds\\
    &\leq 2\creg^3 \beps_m \tilde{S}(t) + \creg^3 \int_{s = 0}^{t} \sqrt{\loss_s}\|\Delta_s\|ds + \creg \int_{s = 0}^t \|\bm{\eps}_s\| ds.
\end{align}
Now 
\begin{align}
    \int_{s = 0}^t \|\bm{\eps}_s\|ds &\leq \int_{s = 0}^t\left(\beps_m \|\Delta_s\|_2 +  4\creg\|\Delta_s\|^2_4 + \beps\right)ds\\
    &\leq \int_{s = 0}^t\left(\beps_m \|\Delta_s\|_2 +  4\creg\|\Delta_s\|\|\Delta_s\|_{\infty} + \beps\right)ds.
\end{align}
Plugging this in yields the claim.
\end{proof}

The following claim builds upon Claim~\ref{claim:func_err_integral} to give bounds on $\|V_t\Delta_t\|_{\mc H}$ and $\|\Delta_t\|$ by induction.
\renewcommand{\tstar}{\min\left((8 \creg \ctwo \cinf^2 S^2 \teps_m)^{-1/3}, (8\creg^3 \ctwo^2 S \beps)^{-1/4}\right)}

\begin{claim}\label{claim:delta_t_lin}
For $t \leq t^* := \tstar$, for $m \geq \Theta_{\creg}(1)$, we have
\begin{align}
    \|\Delta_t\| &\leq \ctwo (t+1)\teps_m;\\
    \|V_t\Delta_t\|_{\mc H} &\leq \cfunc \tilde{S}(t)\beps_m + \creg t \beps \\
    \|\Delta_t\|_{\infty} & \leq \cinf \tilde{S}(t) (t+1) \teps_m + \ctwo \creg t^2 \beps \leq 1.
\end{align}
for some constants $\ctwo = 4\exp(\creg S)$, $\cfunc = 4\creg^3 \ctwo S$, and $\cinf = 4\exp(\creg S) \creg \cfunc$, and with $\teps_m := \beps_m + \beps$.
\end{claim}

\begin{proof}[Proof of Claim~\ref{claim:delta_t_lin}]
We prove this by real induction. Note that it holds trivially for $t = 0$ because $\Delta_0 = 0$. Suppose that the claim holds up to time $t$. We must show that for some $t' > t$, the claim continues to hold up to time $t'$. We will begin with the following observation.
\begin{claim}\label{claim:b_eps_bd}
If the inductive hypothesis holds up to time $t$, then for all $s \leq t$, we have
\begin{align}
    \|\bm{\beta}_s\|_{\infty} + \|\bm{\eps}_s\|_{\infty} &\leq 3\tilde{S}(s)\teps_m\\
    \|\bm{\beta}_s\|_2 + \|\bm{\eps}_s\|_2 &\leq 3\teps_m.
\end{align}
\end{claim}
\begin{proof}
From \eqref{eq:eps_b_bounds} and the definition of $\beps_m$, we have
\begin{align}
    \|\bm{\eps}_s\|_{\infty} + \|\bm{\beta}_s\|_{\infty} \leq \beps_m \|\Delta_s\|_{\infty} + 4\creg\|\Delta_s\|^2_{\infty} + \beps_m + \beps \leq \teps_m + 4\creg\|\Delta_s\|^2_{\infty}.
\end{align}
By the induction hypothesis, and the fact that $t \leq \tstar$, we have that for all $s \leq t$,
\begin{align}
    \|\bm{\eps}_s\|_{\infty} + \|\bm{\beta}_s\|_{\infty} &\leq \teps_m + 8\creg \cinf^2 (\tilde{S}(s))^2(s + 1)^2 \teps_m^2 + 8\creg^3 \ctwo^2 s^4 \beps^2\\
    &\leq \teps_m + 2\teps_m \tilde{S}(s) + \beps \leq 3\tilde{S}(s)\teps_m.
\end{align}
Indeed, $(s + 1)^2\tilde{S}(s) \leq (s + 1)^3 S \leq (8 \creg \ctwo \cinf^2 \teps_m)^{-1}$. Similarly, we have
\begin{align}
    \|\bm{\eps}_s\| + \|\bm{\beta}_s\| \leq \beps_m \|\Delta_s\| +  4\creg\|\Delta_s\|^2_4 + \beps_m + \beps \leq 2\teps_m + 4\creg\|\Delta_s\|\|\Delta_s\|_{\infty}.
\end{align}
Again by the induction hypothesis, and the fact that $t \leq t^*$ we have that for all $s \leq t$,
\begin{align}
    \|\bm{\eps}_s\| + \|\bm{\beta}_s\| &\leq 2\teps_m + 4\creg \cinf \ctwo \tilde{S}(s)(s + 1)^2 \teps_m^2 + \beps \\
    &\leq 2 \beps_m + \teps_m + \beps \leq 3\teps_m.
\end{align}
\end{proof}

Now lets prove the inductive step for $\|\Delta_{t'}\|$. Employing \eqref{eq:dynamics_full} with the bound on both $\|\bm{\beta}_s\| + \|\bm{\eps}_s\| \leq 3\beps_m + \beps$ for $t \leq s$ from Claim~\ref{claim:b_eps_bd}, along with Lemma~\ref{lemma:loss_to_D}, yields that for all $s \leq t$,
\begin{align}
    \frac{d}{ds}\|\Delta_s\| \leq \creg\sqrt{\loss_s}\|\Delta_s\| + 3\teps_m.
\end{align}
Let $Q(t, s) := \exp\left(\int_{r = s}^t \sqrt{\loss_r}dr\right)$. Then by Gronwall's inequality and Duhamel, we have that 
\begin{align}
\|\Delta_{t'}\| &\leq Q(t', 0)^{\creg}\|\Delta_0\| + 3\teps_m\int_{s = 0}^{t} Q(t', s)^{\creg}ds + \int_{s = t}^{t'} Q(t', s)^{\creg}(\|\bm{\eps}_s\| + \|\bm{\beta}_s\|) ds\\
&\leq 3\exp(\creg S)t\teps_m + \exp(\creg S)(t' - t)\sup_{s \in [t, t']}(\|\bm{\eps}_s\| + \|\bm{\beta}_s\|)\\
&\leq 4\exp(\creg S)(t' + 1)\teps_m,
\end{align}
for $t' - t$ small enough. This proves the inductive step for $\|\Delta_t\|$ since $\ctwo =  4\exp(\creg S)$.

Now we prove the inductive step for $\|\Delta_{t'}\|_{\infty}$. We can write
\begin{align}
    \Delta_{t'} = - \int_{s = 0}^{t'} J_{t', s}(H_s \Delta_s - \bm{\eps}_s - \bm{\beta}_s)ds,
\end{align}
so 
\begin{align}\label{eq:4norm}
    \|\Delta_{t'}\|_{\infty} &\leq \int_{s = 0}^{t'} \jmax(t') (\|V_s^\ast V_s\Delta_s\|_{\infty} + \|\bm{\eps}_s\|_{\infty} + \|\bm{\beta}_s\|_{\infty})ds \\
    &\leq \exp(\creg S) \int_{s = 0}^{t'} (\|V_s^\ast\|_{\mc H \to \infty} \|V_s\Delta_s\|_{\mc H} + \|\bm{\eps}_s\|_{\infty} + \|\bm{\beta}_s\|_{\infty})ds\\
    &\leq \exp(\creg S) \int_{s = 0}^{t'} (\creg \|V_s\Delta_s\|_{\mc H} + \|\bm{\eps}_s\|_{\infty} + \|\bm{\beta}_s\|_{\infty})ds,
\end{align}
where here we have used the fact that by Lemma~\ref{lemma:loss_to_D}, $\jmax(t') \leq \exp(\int_{s = 0}^{t'} \|D_s\|ds) \leq \exp(\creg \int_{s = 0}^{t'} \sqrt{\loss_s}ds) \leq \exp(\creg S)$, and by Lemma~\ref{lemma:smoothness}~\ref{SD},
\begin{align}
    \|V_s^\ast\|_{\mc H \to\infty}
    &\leq
    \sup_{\xi\in\nspace}
    \sup_{u\in\sd}
    \sup_{\|g\|_{\mc H}=1}
    \left\langle
        g,\,
        u^\top P_\xi \nabla K(\xi,\cdot)
    \right\rangle_{\mc H} \\
    &\leq
    \sup_{\xi\in\nspace}
    \sup_{u\in\sd}
    \left\|
        u^\top P_\xi \nabla K(\xi,\cdot)
    \right\|_{\mc H} \\
    &\leq \creg.
\end{align}

Let $A_{t, t'} :=\exp(\creg S) \int_{s = t}^{t'} (\creg \|V_s\Delta_s\|_{\mc H} + \|\bm{\eps}_s\|_{\infty} + \|\bm{\beta}_s\|_{\infty})ds$, which goes to zero as $t' - t \rightarrow 0$. Plugging Claim~\ref{claim:b_eps_bd} into \eqref{eq:4norm} along with the inductive hypothesis that $\|V_s\Delta_s\|_{\mc H} \leq \cfunc \tilde{S}(s) \teps_m + \creg s \beps$ for $s \leq t$ yields 
\begin{align}
    \|\Delta_{t'}\|_{\infty} &\leq  \exp(\creg S) \int_{s = 0}^{t}(\creg \cfunc \tilde{S}(s)\teps_m + 3\tilde{S}(s)\teps_m + \creg^2 s \beps)ds + A_{t, t'}\\
    &\leq 3\exp(\creg S)\left(t \creg \cfunc \tilde{S}(t)\teps_m +\creg^2 t^2 \beps\right) + A_{t, t'}\\
    &\leq 4\exp(\creg S) \left(\creg \cfunc \tilde{S}(t)(t' + 1)\teps_m + \creg^2{t'}^2\beps\right)
\end{align}
for $t' - t$ small enough. This proves the inductive step for $\|\Delta_t\|_{\infty}$ since $\cinf =  4\exp(\creg S) \creg \cfunc$.

Finally for $\|V_{t'}\Delta_{t'}\|_{\mc H}$, we have by Claim~\ref{claim:func_err_integral} that 
\begin{align}
    \|V_{t'}\Delta_{t'}\|_{\mc H} &\leq 2\creg^3 \beps_m \tilde{S}(t') + \creg^3  \int_{s = 0}^{t'} \sqrt{\loss_s}\|\Delta_s\|ds + \creg \int_{s = 0}^{t'} \left(\beps_m \|\Delta_s\|_2 +  4\creg\|\Delta_s\|\|\Delta_s\|_{\infty} + \beps\right) ds.
\end{align}    
Now by the inductive hypothesis,
\begin{align}
    \int_{s = 0}^{t} \sqrt{\loss_s}\|\Delta_s\|ds &\leq \int_{s = 0}^t \sqrt{\loss_s}(s + 1)\ctwo \teps_m ds \leq S\tilde{S}(t)\ctwo\teps_m.
\end{align}
Next,
\begin{align}
    \int_{s = 0}^{t} \beps_m \|\Delta_s\|_2 ds &\leq \beps_m \int_{s = 0}^{t} (s + 1)\ctwo \teps_m \leq \beps_m (t + 1)^2 \ctwo \teps_m \leq \beps_m,
\end{align}
by the upper bound on $t$. Finally, 
\begin{align}
    \int_{s = 0}^{t} &4\creg\|\Delta_s\|\|\Delta_s\|_{\infty} ds\\
    &\leq \int_{s = 0}^{t} 4\creg (s + 1)^2 \tilde{S}(s) \ctwo \cinf \teps_m^2 ds + \int_{s = 0}^{t} 4\creg (s + 1)^3 \ctwo^2 \creg \teps_m \beps ds\\
    &\leq 4\creg(t + 1)^3 \ctwo \cinf \tilde{S}(t) \teps_m^2 +  (t + 1)^4\creg^2 \ctwo^2 \teps_m \beps \\
    &\leq \tilde{S}(t) \teps_m 
\end{align}
since $t \leq \tstar$.

Thus for $t - t'$ small enough, since $V_t\Delta_t$ is continuous in $t$, we have 
\begin{align}
    \|V_{t'}\Delta_{t'}\|_{\mc H} \leq \left(2\creg^3 \teps_m \tilde{S}(t) + \creg^3 S\tilde{S}(t')\ctwo \teps_m + \tilde{S}(t') \teps_m + \creg t' \beps\right) \leq 2\creg^3 \ctwo S \tilde{S}(t') \teps_m + \creg t' \beps.
\end{align}
This proves the inductive step for $\|V_t\Delta_t\|_{\mc H}$ since $\cfunc = 2\creg^3 \ctwo S$. 
This completes the proof.
\end{proof}
Returning to \eqref{eq:func_err} and again leveraging the bound on $\|\Delta_t\|_4^4 \leq \|\Delta_t\|^2 \|\Delta_t\|^2_{\infty}$ from Claim~\ref{claim:delta_t_lin}, we have that for all $t \leq \tstar$,

\begin{align}\label{eq:lesststar}
    \|m_{\rtm} - m_{\rtmf}\|_{\mc H}^2 &\leq 2\cfunc^2 \tilde{S}(t)^2\teps_m^2 + 2\creg^2 t^2 \beps^2 + O\left(\frac{\creg^6 (1 + S)^2\log(m)}{m}\right) + O\left(\creg^2 \|\Delta_t\|_4^4\right)\\
    &\leq O\left(\creg^6S^2\exp(2\creg S)(\beps_m^2\tilde{S}(t)^2 + \beps^{1/2})\right).
\end{align}


Now because the above bounds hold only for $t = O\left(\min(\beps_m^{-1/3}, \beps^{-1/4})\right)$, we will leverage the fact that this PoC error can also be bounded from the loss of $\rtm$, which is nearly non-increasing, as per the claim below.

\begin{claim}\label{claim:non-increasing-loss}
For any $t \geq s$, we have 
\begin{align}
    \loss(\rtm) \leq \loss(\rsm) + 4\eps_n.
\end{align}
\end{claim}
\begin{proof}
In the case that $\eta = 0$ and $n = \infty$ (or $\md = \hat{\md}$), this is immediate, since we are running gradient flow on the population loss, so it can never increase. If $\md = \hat{\md}$ but we have a non-zero step size $\eta$, it suffices to show that $\eta$ is smaller than the inverse Lipschitzness of the gradient. Indeed by Lemma~\ref{lemma:smoothness}, the Lipschitzness of $\nu$ with respect to any parameter is at most $\sup_{w, w' \in \nspace}\left(\|\nabla^2 K(w, w')\| + \|\nabla_1 \nabla_2 K(w, w')\|\right) \leq 2\creg$, and we have assumed $\eta < 0.1/\creg$.

Finally, for the case that  $\loss_\md \neq \loss_{\hat{\md}}$, we have that $\loss_{\hat{\md}}$ is non-increasing. Indeed, this is immediate if $\eta = 0$, and for $0 < \eta \leq 0.1/\creg$, it follows from the Lipschitzness bounds on $K_{\hat{\md}}, F_{\hat{\md}}$ and its derivatives in Assumption~\ref{uc:n}. Thus we have 
\begin{align}
    \loss_{\hat{\md}}(\rtm) \leq \loss_{\hat{\md}}(\rsm),
\end{align}
Now we will show that the event in Assumption~\ref{uc:n} holding, implies that
\begin{align}
    \loss_{\md}(\rtm) \leq \loss_{\hat{\md}}(\rtm) + 2\eps_n \leq \loss_{\hat{\md}}(\rsm) + 2\eps_n \leq \loss_{\md}(\rsm) + 4\eps_n.
\end{align}
Indeed, note that the event on $\mathcal{F}'$ in Assumption~\ref{uc:n} holding implies that for any distribution $\rho$ over neurons, 
\begin{align}
    |\mathbb{E}_{(x, y) \sim \md}& (f_{\rho}(x) - y))^2 -  \mathbb{E}_{(x, y) \sim \hat{\md}} (f_{\rho}(x) - y))^2|\\
    &= |\mathbb{E}_{w \sim \rho}\mathbb{E}_{w' \sim \rho} \left(\mathbb{E}_{(x, y)\sim \md}(\sigma(w^{\top}x) - y)(\sigma({w'}^\top x) - y) - \mathbb{E}_{(x, y)\sim \hat{\md}}(\sigma(w^{\top}x) - y)(\sigma({w'}^\top x) - y)  \right)| \\
    &\leq \sup_{w, w' \in \nspace} |\mathbb{E}_{(x, y)\sim \md}(\sigma(w^{\top}x) - y)(\sigma({w'}^\top x) - y) - \mathbb{E}_{(x, y)\sim \hat{\md}}(\sigma(w^{\top}x) - y)(\sigma({w'}^\top x) - y)|\\
    &\leq \eps_n.
\end{align}
Combining with the additional event on $f^*$ in Assumption~\ref{uc:n} implies that uniformly over $\rho$,
\begin{align}
    |\loss_{\md}(\rho) - \loss_{\hat{\md}}(\rho)| = |\mathbb{E}_{x, y \sim \md}(f_{\rho}(x) - y))^2 - \mathbb{E}_{x, y \sim \hat{\md}}(f_{\rho}(x) - y))^2 - (f^*(x) - y))^2 + \mathbb{E}_{x, y \sim \hat{\md}}(f^*(x) - y))^2| \leq 2\eps_n.
\end{align}

\end{proof}

From \eqref{eq:lesststar}, we have that for $t \leq t^*$,
\begin{align}
\|m_{\rtm} - m_{\rho^*}\|_{\mc H}^2 &\leq 2\left(\loss_t + \|m_{\rtm} - m_{\rtmf}\|_{\mc H}^2\right)\\
&\leq 2\loss_t +  O\left(\creg^6 S^2\exp(2\creg S)(\beps_m^2\tilde{S}(t)^2 + \beps^{1/2})\right),
\end{align}

Employing Claim~\ref{claim:non-increasing-loss} above, for all $t \geq t^*$, we have $\loss(\rtm) \leq \mathcal{L}_R(m, n, \eta)$, where
\begin{align}\label{eq:inf_L}
    \mathcal{L}_R(m, n, \eta) := \inf_{t \leq \tstar} 2\loss_t +  O\left(\creg^6 S^2\exp(2\creg S)(\beps_m^2\tilde{S}(t)^2 + \beps^{1/2})\right)
\end{align}


Thus for any $t$ greater than the argmin of \eqref{eq:inf_L} we have
\begin{align}
   \|m_{\rtm} - m_{\rtmf}\|_{\mc H}^2 &\leq 2\left( \mathbb{E}_x (f_{\rtmf}(x) - f^*(x))^2 + \loss(\rtm)\right)\\
    &\leq 2\loss_t + 2\mathcal{L}_R(m, n, \eta) \leq 4\mathcal{L}_R(m, n, \eta),
\end{align}
while for $t$ less than the argmin, we have
\begin{align}
\|m_{\rtm} - m_{\rtmf}\|_{\mc H}^2 \leq O\left(\creg^6\exp(2\creg S)(\beps_m^2\tilde{S}(t)^2 + \beps^{1/2})\right) \leq \mathcal{L}_R(m, n, \eta).
\end{align}
 The result now follows. 
 \end{proof}

\begin{proof}[Proof of Theorem~\ref{thm:informal} from Theorem~\ref{prop:uniform}]
The bound of Theorem~\ref{prop:uniform} holds on a single event for all $t$. It
therefore suffices to bound $\mathcal{L}_R$, and by its definition,
\begin{align}\label{eq:LR}
    \mathcal{L}_R \lesssim \loss_{t'} + e^{2\creg S}\, \beps_m^2\, \tilde S(t')^2
    \qquad \text{for every } t' \leq t^* .
\end{align}

We collect the facts used in both cases:
\begin{itemize}
    \item Since $\eps_n = \eps_\eta = 0$, we have $t^* = \beps_m^{-1/3}/K$ with
    $K := 64\, \creg^3 S^{4/3} e^{5\creg S/3}$. In particular $K \geq 1$,
    $K = \on{poly}(\creg S \log(dm)) e^{5\creg S/3}$, and $t^* \leq \beps_m^{-1/3}$.
    \item By Theorem~\ref{prop:uniform}, $t^* \leq m$.
    \item From the definition of $\eps_m$, $\beps_m = \on{poly}(\creg S \log(dm)) \sqrt{d/m}$. Since
    $m \geq d$, this gives $\beps_m^a = \on{poly}(\creg S \log(dm))\left(\frac{d}{m}\right)^{1/3}$ for every
    $a \in [2/3, 2]$.
    \item On any interval $[u, v]$, some $t'$ satisfies
    $\sqrt{\loss_{t'}} \leq \frac{1}{v-u}\int_u^v \sqrt{\loss_s}\, ds$, since a
    function cannot exceed its average everywhere on an interval.
\end{itemize}

\paragraph{General case.}
Take $t' \in [0, t^*]$ given by the averaging fact. Since $t^* \leq m$, the
integral over $[0, t^*]$ is at most $S$, so
\begin{align*}
    \loss_{t'} \leq \frac{S^2}{(t^*)^2} = S^2 K^2 \beps_m^{2/3} .
\end{align*}
Also $\tilde S(t') \leq 1 + t^*$, so
\begin{align*}
    \beps_m^2\, \tilde S(t')^2 \lesssim \beps_m^2 + \beps_m^{4/3}.
\end{align*}
Substituting into \eqref{eq:LR},
\begin{align*}
    \mathcal{L}_R
    = \on{poly}(\creg S \log(dm))e^{10\creg S/3} \beps_m^{2/3} + e^{2\creg S}\big(\beps_m^{4/3} + \beps_m^2\big)
   = \on{poly}(\creg S \log(dm))e^{4\creg S} (d/m)^{1/3} .
\end{align*}

\paragraph{Case $S' < \infty$.}
Take $t' \in [t^*/2, t^*]$ given by the averaging fact. On this interval
$s^2 \geq (t^*)^2/4$, so
\begin{align*}
    \sqrt{\loss_{t'}}
    \leq \frac{2}{t^*}\int_{t^*/2}^{t^*} \sqrt{\loss_s}\, ds
    \leq \frac{8S'}{(t^*)^3},
    \qquad\text{hence}\qquad
    \loss_{t'} = O(K^6 \beps_m^2) .
\end{align*}
For the fluctuation term, use $\min(t', s) \leq s$, then $s \leq 1$ on $[0,1]$
and $s \leq s^2$ on $[1, \infty)$:
\begin{align*}
    \tilde S(t')
    \leq 1 + \int_0^1 \sqrt{\loss_s}\, ds + \int_1^\infty s^2 \sqrt{\loss_s}\, ds
    \leq S + S' .
\end{align*}
Substituting into \eqref{eq:LR},
\begin{align*}
    \mathcal{L}_R
    \lesssim \big(e^{10\creg S} + e^{2\creg S}\big)\, \beps_m^2
  =  \on{poly}(C S \log(dm))e^{10\creg S}\, \frac{d}{m} .
\end{align*}

\end{proof}

\begin{proof}[Proof of Corollary~\ref{corr:poly_rates}]
\textbf{Notation.} Write
\begin{align*}
    S_0 := B + \tfrac{c}{c-2}, \qquad
    \gamma := \min(c, 6), \qquad
    \Lambda_m := \log(m)\log(dm).
\end{align*}
Throughout, $\lesssim$ hides absolute constants and polynomial factors of $\creg$.
Theorem~\ref{prop:uniform} holds on a single event simultaneously for all $t$. It
therefore suffices to bound $\mathcal{L}_R(m, \infty, 0)$, and we do so by taking
$t = t^*$ in the infimum.

\textbf{Bounds on $S$, $\beps_m$ and $t^*$.}
We have $\loss_t \leq 1$ for $t \leq B$ and $\loss_t \leq (t + 1 - B)^{-c}$
otherwise. Hence
\begin{align*}
    S \leq 1 + B + \int_1^\infty u^{-c/2}\, du = S_0 .
\end{align*}
Consequently,
\begin{align*}
    \beps_m \leq \bar\beta_m := 2\creg(1 + S_0)\log(m)\,\eps_m
    \lesssim S_0 \sqrt{d \Lambda_m^2 / m}.
\end{align*}
Since $\eps_n + \eps_\eta = 0$, the minimum in the definition of $t^*$ equals
$\beps_m^{-1/3}$. Combining this with $S \leq S_0$ and $\beps_m \leq \bar\beta_m$,
\begin{align*}
    \frac{\bar\beta_m^{-1/3}}{K_0} \leq t^* \leq \beps_m^{-1/3},
    \qquad
    K_0 := 64 \creg^3 S_0^{4/3} e^{5 C S_0/3}.
\end{align*}
For $m \geq \widetilde{\Omega}_{\creg, B, c}(d)$ we have
$\bar\beta_m \leq (2BK_0)^{-3} \leq 1$. Hence $t^* \geq 2B$, and also
$t^* \leq m$.

\textbf{Loss term.}
Since $t^* \geq 2B$, we have $t^* + 1 - B \geq t^*/2 \geq 1$. Together with
$\gamma \leq c$, this gives
\begin{align*}
    2\loss_{t^*} \leq 2 (t^*/2)^{-\gamma} \lesssim K_0^\gamma\, \bar\beta_m^{\gamma/3}.
\end{align*}

\textbf{Bound on $\tilde S$.}
Let $\alpha := \max\{0, (\gamma - 4)/2\} \in [0, 1]$. We use three facts:
\begin{itemize}
    \item $\min(t, s) \leq t^{1-\alpha} s^\alpha$ for all $s, t \geq 0$;
    \item the integral can be split at $s = B$;
    \item after substituting $s = u + B - 1$, subadditivity gives
    $(u + B - 1)^\alpha \leq u^\alpha + B^\alpha$.
\end{itemize}
Together these yield
\begin{align*}
    \tilde S(t)
    \leq 1 + t^{1-\alpha} \int_0^\infty s^\alpha \sqrt{\loss_s}\, ds
    \leq 1 + t^{1-\alpha}\Big( B^{1+\alpha} + \frac{1}{c/2 - 1 - \alpha} + \frac{2B^\alpha}{c-2} \Big).
\end{align*}
By the choice of $\alpha$, we have $c/2 - 1 - \alpha \geq \min\{1, (c-2)/2\}$.
The bracket is therefore $\lesssim B^\alpha S_0$, and
\begin{align*}
    \tilde S(t^*) \lesssim B^\alpha S_0 (t^*)^{1-\alpha}.
\end{align*}

\textbf{Fluctuation term.}
We use three ingredients: the upper bound $t^* \leq \beps_m^{-1/3}$; the
identity $(4 + 2\alpha)/3 = \max(4, \gamma)/3 \geq \gamma/3$; and
$\beps_m \leq \bar\beta_m \leq 1$. These give
\begin{align*}
    \creg^6 S^2 e^{2\creg S} \beps_m^2\, \tilde S(t^*)^2
    \lesssim B^{2\alpha} S_0^4 e^{2\creg S_0}\, \beps_m^{(4 + 2\alpha)/3}
    \leq B^{2\alpha} S_0^4 e^{2\creg S_0}\, \bar\beta_m^{\gamma/3}.
\end{align*}

\textbf{Conclusion.}
Combine the two terms and insert
$\bar\beta_m^{\gamma/3} \lesssim S_0^{\gamma/3} (d\Lambda_m^2/m)^{\gamma/6}$.
Then use $B \leq S_0$ and $S_0^p \leq e^{p\creg S_0}$, which is valid because
$\creg \geq 1$. This gives
\begin{align*}
    4\mathcal{L}_R(m, n, 0)
    &\lesssim \Big( S_0^{5\gamma/3} e^{5\gamma \creg S_0/3}
        + S_0^{4 + 2\alpha + \gamma/3} e^{2\creg S_0} \Big)
        \Big(\frac{d\Lambda_m^2}{m}\Big)^{\gamma/6} \\
    &\leq \Big( e^{10\gamma \creg S_0/3} + e^{(6 + 2\alpha + \gamma/3) \creg S_0} \Big)
        \Big(\frac{d\Lambda_m^2}{m}\Big)^{\gamma/6}.
\end{align*}
It remains to show that the second exponential is dominated by the first, that
is, $6 + 2\alpha + \gamma/3 \leq 10\gamma/3$, or equivalently
$6 + 2\alpha \leq 3\gamma$.
\begin{itemize}
    \item If $\gamma \leq 4$, then $\alpha = 0$ and the condition reads
    $6 \leq 3\gamma$, which holds since $\gamma > 2$.
    \item If $\gamma > 4$, then $2\alpha = \gamma - 4$ and the condition reads
    $\gamma + 2 \leq 3\gamma$, which holds trivially.
\end{itemize}
Hence
\begin{align*}
    4\mathcal{L}_R(m, n, 0)
    \lesssim \Big( e^{20 \creg S_0}\, d \Lambda_m^2 / m \Big)^{\gamma/6}.
\end{align*}
The factor $\Lambda_m^{2\gamma/6}$ is absorbed into $\tilde{O}$, which the desired result in Corollary~\ref{corr:poly_rates}.
\end{proof}

\subsection{Proof of Preliminary Lemmas}

\begin{proof}[Proof of Lemma~\ref{lemma:smoothness}]
First note that the operator norm of the first and second derivatives of $P_w$ is at most $2$. Thus for any vector-valued function $\xi(w)$, by chain rule, we have
\begin{align}
    \left\|\nabla_w (I - ww^\top )\xi(w)\right\| &\leq \left\|\nabla_w \xi(w)\right\| + 2\left\|\xi(w)\right\|~, \\
    \left\|\nabla^2_{w} (I - ww^\top )\xi(w)\right\| &\leq 3\left\|\nabla^2_{w} \xi(w)\right\| + 8\left\|\nabla_w \xi(w)\right\|.
\end{align}
Thus it is easy to see from Assumption~\ref{assm:smoothness} that \ref{S1} and \ref{S3} hold. For \ref{S5}, since $\nu(w, \rho) = P_w \nabla_w F(w) - \mathbb{E}_{w' \sim \rho}P_w\nabla_w K(w, w')$, we have (for example) $\|\nabla_w^2 P_w \nabla_w F(w)\| \leq 3\|\nabla_w^3 F(w)\| + 8\|\nabla_w^2 F(w)\| \leq \creg$, and similarly $\|\nabla^2_w P_w\nabla_w K(w, w')\| \leq \creg(1 + \|w'\|)$. Clearly $\mathbb{E}_{w' \sim \rho}\|w\| \leq \kappa_{\rho}$, so this yields the bound on $\left\|\nabla^2_{w} \nu(w, \rho)\right\|_{op}$ in \ref{S5}.

In the kernel mean discrepancy setting, we have $$\nu(w, \rho) = P_w \nabla_w \mathbb{E}_{w' \sim \rho^*}P_w\nabla_w K(w, w') - \mathbb{E}_{w' \sim \rho}P_w\nabla_w K(w, w'),$$
and thus the assumption that $\|w\|$ for $w \sim \rho^*$ is subgaussian suffices.

For \ref{S0} by the reproducing kernel property and Assumption~\ref{assm:smoothness},
\begin{align}
    \left\|K(w, \cdot)\right\|_{\mc H}^2
    &= K(w,w) \leq \creg (1 + \|w\|)^2.
\end{align}

Now for \ref{SD}, by the reproducing kernel property and Assumption~\ref{assm:smoothness},
\begin{align}
    \sup_{w \in \nspace,\; u \in \sd}
    \left\|u^{\top}\nabla K(w, \cdot)\right\|_{\mc H}^2
    &=
    \sup_{w \in \nspace,\; u \in \sd}
    u^{\top}
    \nabla_w\nabla_{w'}K(w,w')
    \big|_{w'=w}
    u \leq
    \sup_{w,w' \in \nspace}
    \left\|
        \nabla_w\nabla_{w'}K(w,w')
    \right\| \leq \creg,
\end{align}

Similarly, for \ref{SDD}, with \(\partial_{i,u}\) denoting the directional derivative in the $i$th argument in direction $u$, we have
\begin{align}
    \left\|D^2K(w, \cdot)[u,v]\right\|_{\mc H}^2
    &=
    \left\langle
        \partial_{1,u}\partial_{1,v}K(w,\cdot),
        \partial_{1,u}\partial_{1,v}K(w,\cdot)
    \right\rangle_{\mc H} \\
    &=
    \partial_{1,u}\partial_{1,v}
    \partial_{2,u}\partial_{2,v}
    K(w,w) \\
    &\leq
    \creg^2 \|u\|^2\|v\|^2.
\end{align}

\end{proof}

\begin{proof}[Proof of Lemma~\ref{lemma:loss_to_D}]
Let $h_{\rho}
    :=
    m_{\rho}-m_{\rho^\ast}
    \in \mc H$,
so that under our normalization,
\[
    \loss(\rho)=\|h_{\rho}\|_{\mc H}^2.
\]
The first variation of the loss can be written as
\[
    U(\xi,\rho)
    =
    \left\langle
        \Phi(\xi),
        h_{\rho}
    \right\rangle_{\mc H}.
\]

We first control \(D_t\). Differentiating twice with respect to \(\xi\), we have  for any \(v,v'\in \sd\),
\begin{align}
    v^{\top}\nabla_{\xi}^{2}U(\xi,\rho)v'
    &=
    \left\langle
        D^{2}\Phi(\xi)[v,v'],
        h_{\rho}
    \right\rangle_{\mc H}.
\end{align}
Therefore, by Cauchy-Schwarz and Lemma~\ref{lemma:smoothness}~\ref{SDD},
\begin{align}
   \left\langle
        D^{2}\Phi(\xi)[v,v'],
        h_{\rho}
    \right\rangle_{\mc H}
    &\leq
    \left\|
        D^{2}\Phi(\xi)[v,v']
    \right\|_{\mc H}
    \|h_{\rho}\|_{\mc H} \leq
    \creg \sqrt{\loss(\rho)}.
\end{align}
Now recall that
\[
    D_t(w)
    =
    \nabla_{\xi_t(w)}
    \left(
        P_{\xi_t(w)}
        \nabla_{\xi_t(w)}
        U(\xi_t(w),\rtmf)
    \right).
\]
When \(\nspace=\mathbb R^d\), we have \(P_{\xi}=I\), so the previous bound directly yields $\|D_t(w)\|
    \leq
    \creg \sqrt{\loss(\rtmf)}$.
When \(\nspace=\sd\), differentiating the projection $P_{\xi}=I-\xi\xi^{\top}$
produces an additional term controlled by
$2\|\nabla_{\xi}U(\xi,\rho)\|$, but here we have 
$u^\top \nabla_\xi U(\xi,\rho) =
\left\langle
D\Phi(\xi)[u],
h_\rho
\right\rangle_{\mc H} \leq \|D\Phi(\xi)[u]\|_{\mc H}
\|h_{\rho}\|_{\mc H} \leq \creg \sqrt{\loss(\rho)}$, again by Lemma~\ref{lemma:smoothness}~\ref{SD}. This yields the first conclusion.

We next control the velocity of the characteristics. The mean-field velocity satisfies $\nu(\xi,\rho)
    = -P_{\xi}\nabla_{\xi}U(\xi,\rho)$.
Thus for any $\xi, \rho$, we have
\begin{align}
    \|\nu(\xi,\rho)\|
    &\leq
    \|\nabla_{\xi}U(\xi,\rho)\| =
    \|D\Phi(\xi)^{\ast}h_{\rho}\| \leq
    \|D\Phi(\xi)\|_{\op}
    \|h_{\rho}\|_{\mc H} 
    \leq
    \creg\sqrt{\loss(\rho)}.
\end{align}
Thus
\[
    \|\dot\xi_t(w)\|
    \leq
    \creg\sqrt{\loss(\rtmf)}.
\]
This, along with the final line of Assumption~\ref{assm:smoothness}, yields the second line of conclusions in the lemma. 

Finally, we control \(\dot V_t\). Recall that the kernel-native operator \(V_t\) is given by
\[
    V_t\Lambda
    =
    \mathbb E_{w\sim \rzm}
    \left[
        D\Phi(\xi_t(w))
        \big[
            P_{\xi_t(w)}\Lambda(w)
        \big]
    \right],
\]
for vector fields \(\Lambda \in L_2(\rzm, \R^d)\). Differentiating in time gives
\begin{align}
    \dot V_t\Lambda
    &=
    \mathbb E_{w\sim \rzm}
    \left[
        D^2\Phi(\xi_t(w))
        \big[
            \dot\xi_t(w),
            P_{\xi_t(w)}\Lambda(w)
        \big]
    \right] \\
    &\qquad
    +
    \mathbb E_{w\sim \rzm}
    \left[
        D\Phi(\xi_t(w))
        \big[
            \dot P_{\xi_t(w)}\Lambda(w)
        \big]
    \right].
\end{align}
The first term is bounded by Lemma~\ref{lemma:smoothness} as
\[
    \left\|
    \mathbb E_{w\sim \rzm}
    \left[
        D^2\Phi(\xi_t(w))
        \big[
            \dot\xi_t(w),
            P_{\xi_t(w)}\Lambda(w)
        \big]
    \right]
    \right\|_{\mc H}
    \leq
    \creg
    \sup_w\|\dot\xi_t(w)\|
    \|\Lambda\|.
\]
In the Euclidean case \(P_\xi=I\), the second term vanishes. In the spherical case, $\|\dot P_{\xi_t(w)}\| = O(\|\dot\xi_t(w)\|)$
and therefore the second term is controlled in the same way using the bound on
\(\|D\Phi(\xi)\|_{\op}\). Hence
\[\boxed{
    \|\dot V_t\|_{\mc H}
    \leq
    \creg
    \sup_w\|\dot\xi_t(w)\| \leq \creg^2\sqrt{\loss(\rtmf)}.}
\]
\end{proof}

\begin{proof}[Proof of Lemma~\ref{lemma:concentration_new}]
Recall that in the kernel mean discrepancy setting, we have $g_t = m_{\brt}-m_{\rtmf} \in \mc H$.
Recall that
\[
    \brt
    =
    \frac1m\sum_{i=1}^m \delta_{\xi_t(w_i)},
    \qquad
    \rtmf
    =
    (\xi_t)_\#\rho_0.
\]
We have
\begin{align}\label{eq:emp_avg}
    g_t-g_\infty
    &=
    \left(
        \mathbb E_{w\sim \brz}
        \big[
            \Phi(\xi_t(w))-\Phi(\xii(w))
        \big]
    \right)
    -
    \left(
        \mathbb E_{w\sim \rho_0}
        \big[
            \Phi(\xi_t(w))-\Phi(\xii(w))
        \big]
    \right).
\end{align}
Thus \(g_t-g_\infty\) is the empirical average of \(m\) i.i.d. mean-zero
\(\mc H\)-valued random variables.

Define
\[
    R_t
    :=
    \int_{s=t}^{\infty}
    \sqrt{\loss(\rho_s^{\mathrm{MF}})}\,ds.
\]
By Lemma~\ref{lemma:loss_to_D},
\[
    \|\xi_t(w)-\xii(w)\|
    \leq
    \creg R_t.
\]
Using Lemma~\ref{lemma:smoothness}~\ref{SD}, we have
\begin{align}
        \|\Phi(\xi_t(w))-\Phi(\xii(w))\|_{\mc H}
    &\leq \|\xi_t(w) - \xii(w)\|\sup_{w' \in \nspace, u \in \sd}\|u^{\top}\nabla \Phi(w')\|_{\mc H} \leq \creg^2 R_t.
\end{align}

Define $h_t = \mathbb E_{w\sim \brz}
            [\Phi(\xi_t(w))-\Phi(\xii(w))]$, and the truncated random variable
\begin{align}
    \tilde{h}_t := \mathbb E_{w\sim \brz}
        \big[
            \Phi(\xi_t(w))-\Phi(\xii(w))\mathbf{1}(\sup_{s < \infty}\|\xi_s(w)\| \leq \sqrt{\log(m)}\kappa_{\infty}).
        \big]
\end{align}
We have by the definition of $\kappa_{\infty}$ that
\begin{align}\label{eq:trunc_bound}
\mathbb{P}[\sup_{t < \infty} \|h_t - \tilde{h}_t\|_{\mc H} > 0] \leq 1 - \left(1 - \mathbb{P}[\sup_{s < \infty}\|\xi_s(w)\| \geq 2\sqrt{\log(m)}\kappa_{\infty}]\right)^m \leq \frac{1}{m}.
\end{align}

Now each random variable in $\tilde{h}_t$ is bounded by $\creg^2\min\left(4\sqrt{\log(m)}(1 + \kappa_{\infty}), R_t\right)$, since by Lemma~\ref{lemma:smoothness}~\ref{S0} we can also always bound $\|\Phi(\xi_t(w))\|_{\mc H} \leq \creg(1 + \|\xi_t(w)\|)$.

Applying a Hilbert-space concentration inequality~\cite{pinelis1994optimum, martinez2024empirical} to the centered empirical average $\tilde{h}_t - \mathbb{E}_{\brz}\tilde{h}_t$ yields
\[
    \mathbb P\left[
        \|h_t - \mathbb{E}_{\brz}h_t\|_{\mc H}
        \geq
        2\creg^2\sqrt{\log(e + t)}\eps_m\min(1 + \kappa_{\infty},R_t)
    \right]
    \leq
    \frac{2}{m^5(e+t)^2},
\]
where we used the definition of $\eps_m =  \frac{\creg^6 \sqrt{d}\log(dm)}{\sqrt{m}}$. Thus combining with \eqref{eq:trunc_bound}, we have that
\[
    \mathbb P\left[
        \|g_t - g_{\infty}\|_{\mc H}
        \geq
        2\creg^2\sqrt{\log(e + t)}\eps_m\min(1 + \kappa_{\infty},R_t)
    \right]
    \leq
    \frac{2}{m^5(e+t)^2} + \frac{2}{m},
\]
since $h_t - \mathbb{E}_{\brz}h_t = g_t - g_{\infty}$.

It remains to extend this estimate uniformly over \(t<\infty\). Indeed to check the conditions of \cite{martinez2024empirical}, observe that the RKHS is $(2, 1)$-smooth and separable, since $\nspace$ is separable.
For \(s\geq t\),
\begin{align}
    \|(g_t-g_\infty)-(g_s-g_\infty)\|_{\mc H}
    &\leq
    2\sup_w
    \|\Phi(\xi_t(w))-\Phi(\xi_s(w))\|_{\mc H} \\
    &\leq
    2\creg
    \sup_w
    \|\xi_t(w)-\xi_s(w)\| \\
    &\leq
    2\creg^2
    \int_{r=t}^{s}
    \sqrt{\loss(\rho_r^{\mathrm{MF}})}\,dr.
\end{align}
Thus \(g_t-g_\infty\) is uniformly continuous in \(t\), with continuity
modulus controlled by the loss. Taking an
\(\eps_m\)-net in time and union bounding the preceding tail estimate over
the net points gives, with probability \(1-3/m\),
\[
    \|g_t-g_\infty\|_{\mc H}
    \leq
    2\creg^2
    \sqrt{\log(e+t)}
    \eps_m
    \min\left(
        1 + \kappa_{\infty},\,
        \int_{s=t}^{\infty}
        \sqrt{\loss(\rho_s^{\mathrm{MF}})}\,ds
    \right)
\]
uniformly for all \(t<\infty\).

We now bound \(g_\infty\). Since
\[
    g_\infty
    =
    \mathbb E_{w\sim \brz}\Phi(\xii(w))
    -
    \mathbb E_{w\sim \rho_0}\Phi(\xii(w)),
\]
it is again the empirical average of \(m\) i.i.d. mean-zero
\(\mc H\)-valued random variables. By Lemma~\ref{lemma:smoothness},
\[
    \|\Phi(\xi)\|_{\mc H}
    \leq
    \creg (1 + \|\xi\|).
\]
Since \(\|\xii(w)\|\) has sub-Gaussian norm
\(\kappa_\infty\) under \(w\sim \rho_0\), the random variable
\[
    \Phi(\xii(w))
    -
    \mathbb E_{w\sim \rho_0}\Phi(\xii(w))
\]
has a 
\(O(\creg\kappa_\infty)\)-subgaussian tail in the $\mc H$-norm. Thus with probability $1 - \frac{1}{2m}$, all of the random variables $\Phi(\xii(w))
    -
    \mathbb E_{w_i \sim \rho_0}\Phi(\xii(w_i))$ are bounded by $O(\creg \kappa_{\infty}\sqrt{\log(m)})$. Again, the Hilbert-space concentration inequality therefore gives
\[
    \mathbb P\left[
        \|g_\infty\|_{\mc H}
        \geq
        \creg \kappa_\infty \eps_m
    \right]
    =
    \frac{1}{m}.
\]
This proves the lemma.
\end{proof}

\begin{proof}[Proof of Lemma~\ref{lemma:kmmd_error}]
We first decompose
\begin{align}
    \|m_{\rtm}-m_{\rtmf}\|_{\mc H}^2
    &\leq
    2\|m_{\brt}-m_{\rtmf}\|_{\mc H}^2
    +
    2\|m_{\brt}-m_{\rtm}\|_{\mc H}^2.
\end{align}

We begin by bounding the coupling term
\[
    \|m_{\brt}-m_{\rtm}\|_{\mc H}^2.
\]
Recall that
\[
    \brt=\frac1m\sum_{i=1}^m \delta_{\bwti},
    \qquad
    \rtm=\frac1m\sum_{i=1}^m \delta_{\bwti+\dit}.
\]
Therefore,
\begin{align}
    m_{\brt}-m_{\rtm}
    &=
    \mathbb E_{i}
    \left[
        \Phi(\bwti)-\Phi(\bwti+\dit)
    \right].
\end{align}
Using the second-order Taylor expansion of the \(\mc H\)-valued map \(\Phi\),
\begin{align}
    \Phi(\bwti)-\Phi(\bwti+\dit)
    &=
    D\Phi(\bwti)[-\dit]
    +
    \int_{s=0}^1
    \int_{s'=0}^{s}
    D^2\Phi(\bwti+s'\dit)[\dit,\dit]
    \,ds'\,ds.
\end{align}
Hence
\begin{align}
    m_{\brt}-m_{\rtm}
    &=
    -V_t\Delta_t + B_t,
\end{align}
where
\[
    V_t\Delta_t
    :=
    \mathbb E_i
    D\Phi(\bwti)[\dit]
\]
and
\[
    B_t
    :=
    \mathbb E_i
    \int_{s=0}^1
    \int_{s'=0}^{s}
    D^2\Phi(\bwti+s'\dit)[\dit,\dit]
    \,ds'\,ds.
\]
Therefore,
\begin{align}\label{eq:kmmd-coupling-upper}
    \|m_{\brt}-m_{\rtm}\|_{\mc H}^2
    &\leq
    2\|V_t\Delta_t\|_{\mc H}^2
    +
    2\|B_t\|_{\mc H}^2.
\end{align}

Recall from \eqref{eq:HVV}
\[
    \|V_t\Delta_t\|_{\mc H}^2
    =
    \Delta_t^\top \hpt \Delta_t.
\]
It remains to bound \(\|B_t\|_{\mc H}^2\). By Jensen's inequality,
\begin{align}
    \|B_t\|_{\mc H}^2
    &\leq
    \mathbb E_i
    \left\|
        \int_{s=0}^1
        \int_{s'=0}^{s}
        D^2\Phi(\bwti+s'\dit)[\dit,\dit]
        \,ds'\,ds
    \right\|_{\mc H}^2.
\end{align}
Using Lemma~\ref{lemma:smoothness}~\ref{SDD}, we have $\left\|D^2\Phi(\xi)[u,v]\right\|_{\mc H}^2 \leq \creg \|u\|\|v\|$ and thus
\begin{align}
    \|B_t\|^2_{\mc H}
    &\leq
    O(\creg^2
    \mathbb E_i \|\dit\|^4).
\end{align}
Combining this with \eqref{eq:kmmd-coupling-upper} yields
\begin{align}\label{eq:kmmd-coupling-final}
    \|m_{\brt}-m_{\rtm}\|_{\mc H}^2
    &\leq
    2\Delta_t^\top \hpt\Delta_t
    +
    O\left(
        \creg^2
        \|\Delta_t\|_4^4\bigr)
    \right).
\end{align}

We now bound the Monte-Carlo term
\[
    \|m_{\brt}-m_{\rtmf}\|_{\mc H}^2.
\]
Since \(\brt\) is the empirical measure of \(m\) i.i.d. samples from \(\rtmf\),
\[
    m_{\brt}-m_{\rtmf}
    =
    \frac1m\sum_{i=1}^m
    \left(
        \Phi(\bwti)-m_{\rtmf}
    \right)
\]
is an average of i.i.d. centered \(\mc H\)-valued random variables. 

Repeating the exact same argument as the one in the proof of Lemma~\ref{lemma:concentration_new} to bound $\|g_t - g_{\infty}\|_{\mc H}$, with the exception each term in the truncated random variable is only bounded by $4\creg^2\sqrt{\log(m)}\kappa_{\infty}$, we attain that with probability $1 - 1/m$, uniformly over $t < \infty$, we have 
\begin{align}
    \|m_{\brt}-m_{\rtmf}\|^2_{\mc H} \leq \frac{\creg^4\kappa_{\infty}^2\log(m(e + t))}{m}.
\end{align}



Combining this Monte-Carlo estimate with
\eqref{eq:kmmd-coupling-final}, we conclude that with probability
\(1-m^{-\Theta(1)}\),
\begin{align}
    \|m_{\rtm}-m_{\rtmf}\|_{\mc H}^2
    &\leq
    4\Delta_t^\top \hpt \Delta_t
    +
    \frac{\creg^4\kappa_{\infty}^2\log(m(e + t))}{m} + O\left(\creg^2\|\Delta_t\|_4^4\right).
\end{align}
This concludes the proof.
\end{proof}

%% file: neurips26/uniform_experiments.tex
\section{Experiments}\label{sec:experiments}

In this section, we provide several examples of learning problems which empirically demonstrate fast enough convergence rates to satisfy the conditions of our main theorem. We study two settings, and defer further experimental details to Appendix \ref{sec:exp_app}.

\paragraph{Misspecified Sobolev single-index model} 
We draw data $(x, f^*(x))$, where $x$ is drawn from  $\sqrt{d-1}\sd$, and then clamped to have all coordinates of $x \in [-1, 1]$. We let $f^*(x) := \phi(x_1) := \mathbb{E}_{w \sim \rho^*}\sigma(x^\top w)$, where $\sigma$ is the ReLU activation, and $w \in \R^d$. Inspired by \cite{chizat2026quantitative}, we study a class of problems $f^* \in \{f^{\gamma}\}_{\gamma = 1, 2, 4, 8}$ defined in Appendix \ref{sec:exp_app}, parameterized by $\gamma$, the largest value such that when $d = 2$, the $\gamma$-Sobolev norm of $\rho^*$ is finite~\footnote{The $\gamma$-Sobolev norm of a function $g$ is $\left(\sum_{k \in \mathbb{Z}^{d} \setminus 0} (2\pi|k|)^{2\gamma}|\hat{g}_k|^2\right)^{1/2}$, where $\hat{g}_k := \int_{\sd} g(x)\exp(-2\pi ikx)d\omega(x)$ is the Fourier coefficient of $g$.}. This in turn governs the smoothness of $\phi$, where larger $\gamma$ means more smooth. These target functions are pictured in Figure~\ref{fig:fstar}, and described mathematically in Appendix \ref{sec:exp_app}, along with further experimental details.
\cite{chizat2026quantitative} showed that for a variant of this problem that is equivalent to our setting in the case that $d = 2$, when both layers are trained, the local convergence rate is of order $t^{-(\gamma + 1)}$. 
In Figure \ref{fig:sim_loss2} we plot both the loss for a network of width $m = 32768$, which approximates $\loss(\rtmf)$, and the integral of the square root of the loss $R_t := \int_{s = 0}^t \sqrt{\loss(\rsmf)}ds$, which is the quantity Theorem~\ref{prop:uniform} assumes to be bounded to guarantee uniform-in-time PoC. We train up to $T = 512$ time. We observe that when $d=2$ and $\gamma = 1$, the least smooth setting, $R_t$ does not converge, while for larger values of $\gamma$, $R_t$ does converge (see Figure~\ref{fig:sim_loss}). For $d = 128$, plotted in Figure~\ref{fig:sim_loss2}, we observe convergence at all values of $\gamma$, though $R_t$ is larger for smaller $\gamma$. In all cases we observe global convergence. We remark that because these experiments were conducted with a ReLU activation, this setting does not satisfy the conditions of Assumptions~\ref{assm:reg}. 
\begin{figure}[t]
    \centering
    \includegraphics[width=0.85\textwidth]{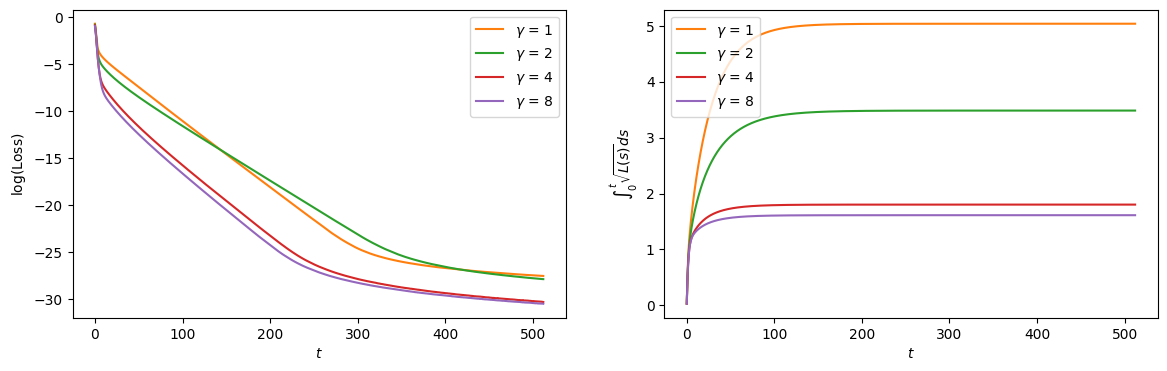}
    \caption{Approximate loss $\loss(\rtmf)$ (left) and $\int_{s = 0}^t \sqrt{\loss(\rsmf)}ds$ (right) for $d = 128$. We train both layers.}
    \label{fig:sim_loss2}
\end{figure}

\paragraph{Two-dimensional examples}
We illustrate low-dimensional examples where $\mathcal{S} = \mathbb{S}^2$, and we choose the arcosine kernel arising from the ReLU activation. We implement the Eulerian dynamics by gridding the domain and performing upwind integration to preserve the probability mass. 
Figure \ref{fig:low-dim} illustrates the densities of several targets distributions with varying smoothness. In qualitative agreement with the local analysis of \cite{chizat2026quantitative}, in the smooth settings we observe a sufficiently fast decay of the population loss, leading to an effective PoC rate. In contrast, the singular target measure does not satisfy our decay assumptions, even though we observe global mean-field convergence. 

\begin{figure}
    \centering
    \includegraphics[width=\linewidth]{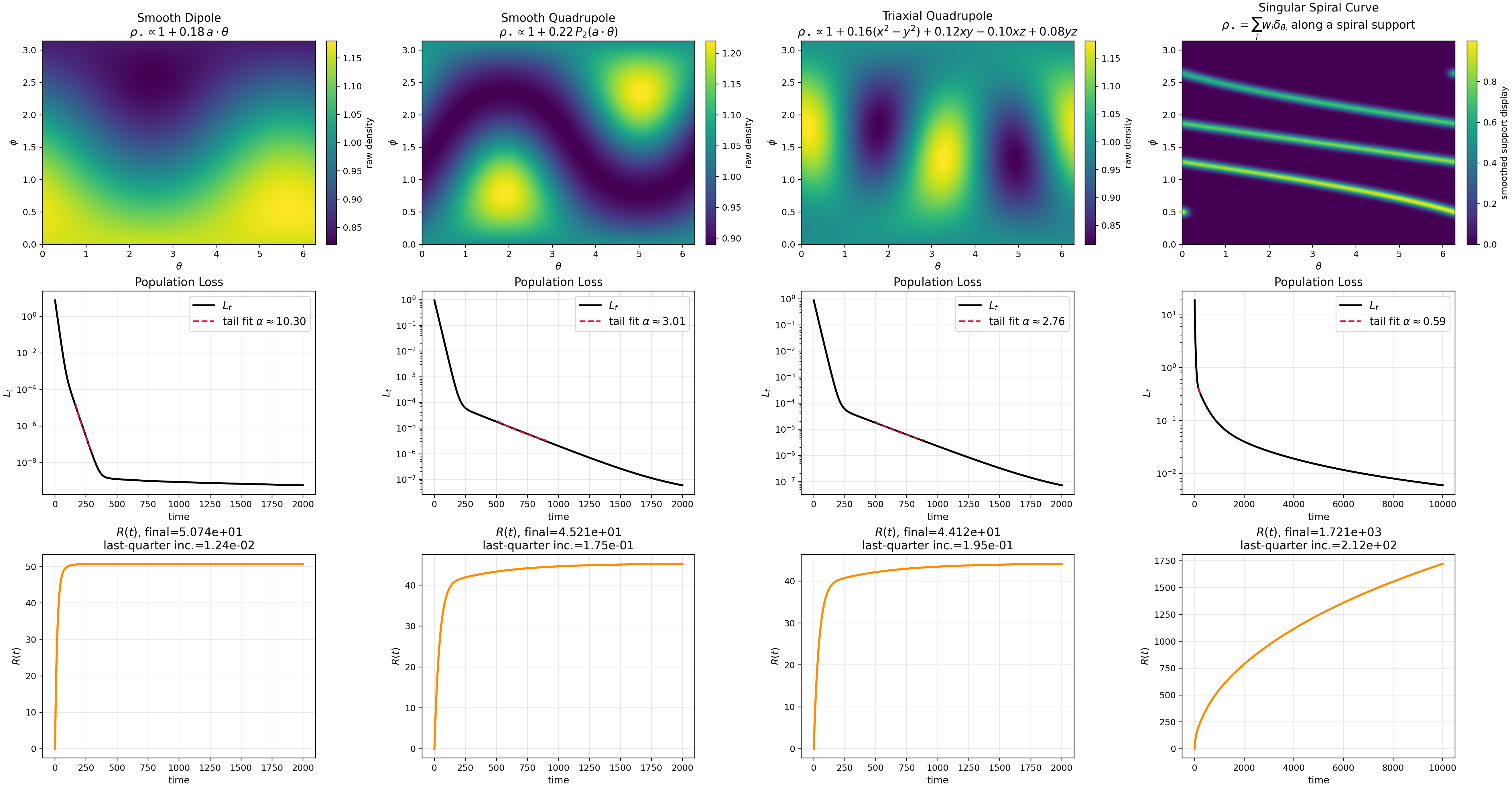}
    \caption{examples of target densities (top), alongside the behavior of $\mathcal{L}_t$ (middle) and the associated $R_t = \int_0^t \sqrt{\mathcal{L}_s} ds$ (bottom). }
    \label{fig:low-dim}
\end{figure}

%% file: neurips26/apx_conc_dyn_new.tex
\section{Proofs of Assumptions~\ref{uc:n}, \ref{uc:m} and Lemma~\ref{lemma:errdynamicstight}}\label{apx:dyn}
\subsection{Notations}
Throughout this section, we will use the following notation, which builds upon the notation in our setup from the main body.
\begin{align}
    F(w) &:= \mathbb{E}_{(x, y) \sim \md} y\sigma(w^\top x)\\ 
    F'(w) &:= P_w\nabla_w F(w)
\end{align}
and 
\begin{align}\label{eq:FKdef}
    K(w, w') &:= \mathbb{E}_{x \sim \md} \sigma(w'^\top x)\sigma(w^\top x)\\
    K'(w, w') &:= P_w \nabla_w K(w, w').
\end{align}
Let $\mc H$ be the RKHS generated by the kernel $K$, with inner product $\langle{,}\rangle_{\mc H}$. By the gradient flow dynamics in Equation~\ref{eq:GF_dynamics}, we have 
\begin{align}\label{eq:Vexpansion}\textstyle
        \frac{d}{dt} w = \nu_{\md}(w, \rho) := P_w\nabla_w F(w) - P_w\nabla_w \mathbb{E}_{w' \sim \rho} K(w, w').
\end{align}

We also define the \em empirical local Hessian \em $\bar{D}_t$ (closely related to $\dpt$), where the expectation is taken over $\brt$ instead of $\rtmf$:
\begin{align}
    \bar{D}_t(w) &:= \nabla_{\xi_t(w)} \nu(\xi_t(w), \brt) = \nabla_{\xi_t(w)} F'(\xi_t(w)) - \mathbb{E}_{w' \sim \brt}\nabla_{\xi_t(w)} K'(\xi_t(w), w').\\
    \dpt(w) &= \nabla_{\xi_t(w)} \nu(\xi_t(w), \rtmf) = \nabla_{\xi_t(w)} F'(\xi_t(w)) - \mathbb{E}_{w' \sim \rtmf}\nabla_{\xi_t(w)} K'(\xi_t(w), w').
\end{align}

\textbf{Standard Notation.} The notation
$(\sd)^{\otimes j}$ denotes rank-one tensors, with
\[
 \langle u,x^{\otimes j}\rangle
 =\prod_{\ell=1}^j u_\ell^\top x,
 \qquad u=u_1\otimes\cdots\otimes u_j.
\]
For $j=0$, we use the convention $(\mathbb S^{d-1})^{\otimes0}=\{1\}$ and
$\langle1,x^{\otimes0}\rangle=1$.

For a probability measure $P$, write $Pf=\int f\,dR$ and
$\|f\|_{L_2(P)}=(P|f|^2)^{1/2}$; $\|f\|_\infty$ denotes the
pointwise supremum of $|f|$. For a real random variable $U$, define
\[
 \|U\|_{\psi_\alpha}
 :=\inf\{a>0:\mathbb E\exp((|U|/a)^\alpha)\leq2\},
 \qquad \alpha=1,2.
\]
Thus the subgaussian and subexponential bounds below concern the
uncentered variables. Write $\operatorname{Var}_{\mathbb R}(f)$ for
the total variation of $f$ on $\mathbb R$.
Let $N(\varepsilon,\mathcal C,L_2(P))$ denote the smallest cardinality
of an $\varepsilon$-net for $\mathcal C$ in the $L_2(P)$ distance.

\subsection{Concentration Lemmas}
The main goal of this section is to show that under Assumption~\ref{assm:reg}, the following two uniform-convergence guarantees hold with the values $\eps_n$ and $\eps_m$ given in the introduction. We state these guarantees as assumptions, because as per Remark~\ref{rem:regularity}, they (along with Assumption~\ref{assm:smoothness}) suffice to yield our main results. The proof dependencies are summarized in Figure~\ref{fig:Assm_deps}.

\begin{figure}
    \centering
\begin{tikzpicture}[x=1cm,y=1cm]

\node[depassumption,text width=4.1cm] (reg) at (6.75,0)
  {\textbf{Assumption~\ref{assm:reg}}\\ \ref{assm:sigma}: activation; \ref{assm:data}: data;\\ \ref{assm:init}: bounded initialization};

\node[deplemma] (vc) at (1,-2.25)
  {\textbf{Claim~\ref{claim:vc}}\\VC dimension of\\intersections of two halfspaces};
\node[deplemma] (exp) at (1,-4.40)
  {\textbf{Claim~\ref{eq:five_class_mean}}\\Expected supremum\\of multiplier processes};
\node[deplemma] (l15) at (1,-6.55)
  {\textbf{Lemma~\ref{claim:threshold_vc}}\\Uniform concentration\\for weighted halfspaces};

\node[deplemma,text width=4.1cm] (hilbert) at (6.75,-4.40)
  {\textbf{Hilbert-space argument}\\Lipschitz map $z\mapsto U_z$,\\variance and concentration};
\node[deplemma,text width=4.1cm] (l18) at (6.75,-6.55)
  {\textbf{Lemma~\ref{lemma:residual_products}}\\Neural-feature and residual-product\\uniform concentration};
\node[deplemma,text width=4.1cm] (l19) at (6.75,-8.75)
  {\textbf{Lemma~\ref{lemma:regtouc}}\\Regularity implies\\Assumptions 2 and 3};

\node[depassumption] (a1b) at (13.5,-2.25)
  {\textbf{Assumption~\ref{assm:smoothness}}\\Kernel smoothness\\by H\"older estimates};
\node[deplemma] (l7b) at (13.5,-4.40)
  {\textbf{Lemma~\ref{lemma:smoothness}}\\Kernel and velocity\\derivative bounds};

\node[depassumption] (a2b) at (3.75,-11.0)
  {\textbf{Assumption~\ref{uc:n}}\\Data concentration\\and empirical kernels};
\node[depassumption] (a3b) at (9.75,-11.0)
  {\textbf{Assumption~\ref{uc:m}}\\Concentration over\\sampled neurons};

\begin{scope}[on background layer]
\draw[deparrow] (vc) -- (exp);
\draw[deparrow] (exp) -- (l15);
\draw[deparrow] (l15.east) -- (l18.west);
\draw[deparrow] (hilbert) -- (l18);
\draw[deparrow] (l18) -- (l19);
\draw[deparrow] (reg.south west) -- (4.3,-1.15)
  -- (4.3,-5.35) -- (l18.north west);
\draw[deparrow] (reg.south east) -- (a1b.north west);
\draw[deparrow] (a1b) -- (l7b);
\draw[depaux] (reg.south) -- (hilbert.north);
\draw[depaux] (reg.west) -- (-1.45,0) -- (-1.45,-6.55) -- (l15.west);
\draw[deparrow] (l19.south west) -- (a2b.north);
\draw[deparrow] (l19.south east) -- (a3b.north);
\draw[depaux] (l7b.south) -- (13.5,-8.75) -- (l19.east);
\end{scope}


\end{tikzpicture}
    \caption{Proof Dependencies in Verification of Assumptions~\ref{assm:smoothness}-\ref{uc:m} from Assumption~\ref{assm:reg}}
    \label{fig:Assm_deps}
\end{figure}

\ucn*

\ucm*

The following two lemmas use standard techniques from empirical process theory along with some analytic. Assume that all constants $\kappa$ and $\kappa_x$ throughout are at least $1$ in what follows. In the next two lemmas $C$ denotes a universal constant that may change from line
to line.

\begin{lemma}[Uniform convergence]\label{claim:threshold_vc}
Let $(X_i,Y_i)_{i=1}^n$ be i.i.d.\ with law $P$ on
$\mathbb R^d\times\mathbb R$, where $n\geq d\geq1$, and write
\[
 P_nf=\frac1n\sum_{i=1}^n f(X_i,Y_i),
 \qquad Pf=\mathbb E f(X,Y).
\]
Suppose that $X$ and $Y$ are $\kappa_x$-subgaussian.
Let $\mathcal A$ be the class of intersections of at most two
closed affine halfspaces in $\mathbb R^d$, including the full
and empty sets. For $j=0,1,2$, define
\[
 \mathcal G^{(j)}
 :=
 \left\{
 (x,y)\mapsto
 y\langle u,x^{\otimes j}\rangle\mathbf1_H(x):
 u\in(\mathbb S^{d-1})^{\otimes j},\ H\in\mathcal A
 \right\}.
\]

Write $\left.\mathcal G^{(j)}\right|_{y\equiv1}$ for the class
with the factor $y$ replaced by $1$. 

There is a universal constant $K$ such that, for every
$0<\delta<1$, with probability at least $1-\delta$,
\begin{equation}\label{eq:five_class_uc}
 \max\left\{
   \sup_{g\in\bar{\mathcal G}}|(P_n-P)g|,
   \left|\frac1n\sum_{i=1}^nY_i^2-\mathbb EY^2\right|
 \right\}
 \leq
 K(1+\kappa_x)^2
 \sqrt{\frac dn}\log\frac{2d}{\delta},
\end{equation}
simultaneously for
\[
 \bar{\mathcal G}\in
 \left\{
   \left.\mathcal G^{(0)}\right|_{y\equiv1},
   \left.\mathcal G^{(1)}\right|_{y\equiv1},
   \left.\mathcal G^{(2)}\right|_{y\equiv1},
   \mathcal G^{(0)},\mathcal G^{(1)}
 \right\}.
\]
\end{lemma}

\begin{proof}
Throughout the proof, $K$ denotes a universal constant that may
change from line to line.
\begin{claim}\label{claim:vc}
$\mathcal{A}$ has VC dimension at most $Kd$.
\end{claim}
\begin{proof}
Affine halfspaces have VC dimension $d+1$, so Sauer's lemma bounds the number of labelings of any $m\geq d+1$ points by intersections of two halfspaces by
\[
 \left(\frac{em}{d+1}\right)^{2(d+1)}.
\]
If these $m$ points were shattered, then
\[
 2^{m/(d+1)}
 \leq
 \left(\frac{em}{d+1}\right)^2,
\]
which forces $m/(d+1)$ to be bounded by a universal constant.
\end{proof}

Now fix $u,v\in\mathbb S^{d-1}$ and take one of the multipliers
\[
 a(x,y)\in
 \left\{
  1,\ y,\ u^\top x,\ y(u^\top x),\
  (u^\top x)(v^\top x)
 \right\}.
\]
Set
\[
 Z_a:=\sup_{H\in\mathcal A}
       |(P_n-P)(a\mathbf1_H)|.
\]
We will first bound $\mathbb EZ_a$, then bound it with high probability, and finally prove that this bound holds uniformly over all choices of $u, v$.
\begin{claim}\label{eq:five_class_mean}
\begin{align}
    \mathbb EZ_a \leq K\|a\|_{L_2(P)} \sqrt{\frac{d}{n}}.
\end{align}
\end{claim}
\begin{proof}
For any probability measure $R$ on $\mathbb R^d\times\mathbb R$ with $0<Ra^2<\infty$, define the probability measure $R_a$ on
$\mathbb R^d$ by up-weighting by $a^2$:
\[
 R_a(B):=\frac{R[a^2\mathbf1(x \in B)]}{Ra^2}.
\]
Then
\[
 \|a\mathbf1_H-a\mathbf1_{H'}\|_{L_2(R)}^2
 =
 \|a\|_{L_2(R)}^2
 \|\mathbf1_H-\mathbf1_{H'}\|_{L_2(R_a)}^2.
\]
Consequently, the VC covering bound for indicators
\cite[Theorem 2.6.4]{vaart2023empirical} gives from Claim~\ref{claim:vc},
\[
 \log N\left(
   \eta\|a\|_{L_2(R)},
   \{a\mathbf1_H:H\in\mathcal A\},
   L_2(R)
 \right)
 \leq Kd\log(K/\eta),
 \qquad 0<\eta<1.
\]
Here $N$ denotes the covering number in the indicated metric.

Apply this bound with $R=P_n$ and condition on the observations.
Symmetrization, Dudley's entropy integral
\cite[Theorem 5.22]{wainwright2019high}, and Jensen's inequality
give
\begin{align}
 \mathbb EZ_a
 &\leq
 K\sqrt{\frac dn}\,
 \mathbb E\|a\|_{L_2(P_n)}
 \int_0^1\sqrt{\log(K/\eta)}\,d\eta
 \leq
 K\|a\|_{L_2(P)}\sqrt{\frac dn}.
\end{align}
\end{proof}


We will now show a high probability bound on $Z_a$. The subgaussian assumptions imply, for each of the five choices of $a$,
\begin{equation}\label{eq:five_class_multiplier}
 \|a\|_{L_2(P)}+\|a(X,Y)\|_{\psi_1}
 \leq K(1+\kappa_x)^2.
\end{equation}
For the products, this follows from Cauchy--Schwarz and the fact that a product of two subgaussian variables is
subexponential.

Observe then that
\[
 \sup_{H\in\mathcal A}
 |a(X,Y)\mathbf1_H(X)-P(a\mathbf1_H)|
 \leq |a(X,Y)|+P|a|.
\]
Thus, by \eqref{eq:five_class_multiplier},
\[
 \left\|
 \max_{1\leq i\leq n}\sup_{H\in\mathcal A}
 \left|
 a(X_i,Y_i)\mathbf1_H(X_i)-P(a\mathbf1_H)
 \right|
 \right\|_{\psi_1}
 \leq
 K(1+\kappa_x)^2\log(2n).
\]
Also,
\[
 \sup_{H\in\mathcal A}
 P\left(a\mathbf1_H-P(a\mathbf1_H)\right)^2
 \leq Pa^2.
\]
Applying \cite[Theorem 4]{adamczak2008tail} with $\alpha=1$ to the centered class, and using
\eqref{eq:five_class_multiplier} and Claim~\ref{eq:five_class_mean},
shows that, for every $q>0$, with probability at least $1 - 4e^{-q}$.
\begin{equation}\label{eq:five_class_fixed}
 Z_a
 \leq
 K(1+\kappa_x)^2
 \left[
   \sqrt{\frac{d+q}{n}}
   +\frac{q\log(2n)}{n}
 \right]
\end{equation}

It remains to make this uniform over the directions $u, v \in \sd$.
Let $\mathcal N$ be a $1/4$-net of $\mathbb S^{d-1}$ with
$|\mathcal N|\leq9^d$.
For any vector $b$ and matrix $A$,
\[
\begin{aligned}
 \sup_{u\in\mathbb S^{d-1}}|u^\top b|
 &\leq
 \frac43\max_{u\in\mathcal N}|u^\top b|,\\
 \sup_{u,v\in\mathbb S^{d-1}}|u^\top Av|
 &\leq
 2\max_{u,v\in\mathcal N}|u^\top Av|.
\end{aligned}
\]
Indeed, approximating each direction by a net point produces an error at most $\frac{1}{4}$ of the linear supremum, or one half of the bilinear supremum, so for example, $\sup_{u\in\mathbb S^{d-1}}|u^\top b|
 \leq
\max_{u\in\mathcal N}|u^\top b| + \frac{1}{4}\sup_{u\in\mathbb S^{d-1}}|u^\top b|$.

For each fixed $H$, the directional empirical deviations have exactly these forms, with
\[
 b=(P_n-P)(X\mathbf1_H)
 \quad\text{or}\quad
 b=(P_n-P)(YX\mathbf1_H),
 \qquad
 A=(P_n-P)(XX^\top\mathbf1_H).
\]
The net inequalities therefore hold uniformly over $H$.
The multipliers $1$ and $y$ require no directional net.

Set
\[
 t:=\log\frac8\delta,
 \qquad
 q:=2d\log9+\log\frac{40}{\delta}
 \leq K(d+t).
\]
A union bound in \eqref{eq:five_class_fixed} over all 5 classes and all $9^{2d}$ net points for $u, v$, gives, with failure probability at most
$\delta/2$, simultaneously for all five classes,
\[
 \sup_{g\in\bar{\mathcal G}}|(P_n-P)g|
 \leq
 K(1+\kappa_x)^2
 \left[
   \sqrt{\frac{d+t}{n}}
   +\frac{(d+t)\log(2n)}{n}
 \right].
\]
For $n\geq d\geq1$, this can be simplified to the bound in the lemma.


Finally, $\|Y^2\|_{\psi_1}\leq K\kappa_x^2$, so Bernstein's inequality gives, with failure probability at most $\delta/2$,
\[
 \left|\frac1n\sum_{i=1}^nY_i^2-\mathbb EY^2\right|
 \leq
 K\kappa_x^2
 \left(\sqrt{\frac tn}+\frac tn\right)
 \leq
 K(1+\kappa_x)^2
 \sqrt{\frac dn}\log\frac{2d}{\delta}.
\]
A final union bound proves \eqref{eq:five_class_uc}.
\end{proof}

\begin{lemma}[Uniform convergence and products of residuals]\label{lemma:bv_vc}
\label{lemma:residual_products}
Fix $0<\delta<1$ and integers $n\geq d\geq1$. Suppose
$\sigma:\mathbb R\to\mathbb R$ is differentiable and $|\sigma(0)|\leq C, 
 \|\sigma'\|_\infty\leq C, 
 \operatorname{Var}_{\mathbb R}(\sigma')\leq C$.
Let $\phi:\mathbb R\to\mathbb R$ have $\|\phi\|_\infty\leq C$. Let $P$ be the joint law of
$(X,Y)\in\mathbb R^d\times\mathbb R$, with $X$-marginal $P_x$,
and let $P_z$ be the law of $Z\in\mathbb R^d$. Assume both $X$ and $Y$ are $\kappa_x$ subgaussian, and $\|Z\|$ is $\kappa_z$ subgaussian.

Let $(X_i,Y_i)_{i=1}^n$ be i.i.d. with law $P$ and let
$Z_1,\ldots,Z_n$ be i.i.d. with law $P^z$. Denote their empirical
averages
\[
 P_nf=\frac1n\sum_{i=1}^n f(X_i,Y_i),\qquad
 P^z_nh=\frac1n\sum_{i=1}^n h(Z_i).
\]
Define
\begin{align*}
 \mathcal F
 &=\{f_{u,w}(x,y):=y\sigma'(w^\top x)(u^\top x):
       u\in\mathbb S^{d-1},\ w\in\mathbb R^d\},\\
 \mathcal F'
 &=\left\{f_{u,w,\xi}(x):=
       \frac{\sigma'(w^\top x)(u^\top x)\sigma(\xi^\top x)}
            {1+\|\xi\|}:
       u\in\mathbb S^{d-1},\ w,\xi\in\mathbb R^d\right\},\\
 \mathcal H^{(j)}
 &=\bigl\{h_{u,w}(z):=\mathbb E_{X\sim P_x}
       \phi(w^\top X)\langle u,X^{\otimes j}\rangle\sigma(z^\top X):
u\in(\mathbb S^{d-1})^{\otimes j},\
       w\in\mathbb R^d\bigr\},\qquad j=1,2.
\end{align*}
With probability at least $1-\delta$, simultaneously, for some universal constant $K$,
\begin{align}
 \sup_{f\in\mathcal F}|(P_n-P)f|
 &\leq KC\kappa_x^2\sqrt{\frac dn}\log\frac2\delta,
       \label{eq:v3_F}\\
 \sup_{f\in\mathcal F'}|(P_n-P)f|
 &\leq KC^2\kappa_x(1+\kappa_x)\sqrt{\frac dn}\log\frac2\delta,
       \label{eq:v3_Fprime}\\
 \sup_{h\in\mathcal H^{(j)}}|(P^z_n-P^z)h|
 &\leq KC^2\kappa_x^{j+1}\kappa_z\sqrt{\frac{\log(2/\delta)}n},
       \qquad j=1,2.\label{eq:v3_H}
\end{align}

If additionally $\operatorname{Var}_{\mathbb R}(\sigma)\leq C$ define
\[
 \mathcal F''=\{(x,y)\mapsto
    (\sigma(\xi^\top x)-y)(\sigma(w^\top x)-y):
    w,\xi\in\mathbb R^d\}.
\]
Then, on the same event of probability at least $1-\delta$,
\begin{equation}
 \sup_{f\in\mathcal F''}|(P_n-P)f|
 \leq K(C+\kappa_x)^2\sqrt{\frac dn}\log\frac2\delta.
 \label{eq:v3_residual_concentration}
\end{equation}
\end{lemma}

\begin{proof}
Throughout the proof, $K$ denotes a universal constant that may
change from line to line. First observe that because $\sigma'$ has total variation bounded by
$C$, we can write
\begin{equation}
 \sigma'(z)=A+\int_{\mathbb R}\mathbf1\{z\geq s\}\,\mu(ds),
 \qquad |A|\leq C,\quad\|\mu\|_{\mathrm{TV}}\leq C.
 \label{eq:phi_bv}
\end{equation}
Here $\mu$ is a finite signed measure. Thus
\[
 \sup_{f\in\mathcal F}|(P_n-P)f|
 \leq (|A|+\|\mu\|_{\mathrm{TV}})
       \sup_{g\in\mathcal G^{(1)}}|(P_n-P)g|.
\]
Apply Lemma~\ref{claim:threshold_vc} with failure probability
$\delta/4$; replacing $\log(2/\delta)$ by $\log(8/\delta)$ changes
only the universal constants. Since $|A|+\|\mu\|_{\mathrm{TV}}\leq2C$,
this proves \eqref{eq:v3_F}.

For $\mathcal F'$, write $\xi=r\bar\xi$, where $r=\|\xi\|$ and
$\bar\xi\in\mathbb S^{d-1}$. If $\xi=0$, choose any unit vector
$\bar\xi$; the term multiplied by $r$ below is then zero. The
fundamental theorem of calculus gives
\[
 \frac{\sigma(\xi^\top x)}{1+r}
 =\frac{\sigma(0)}{1+r} +\frac r{1+r}(\bar\xi^\top x)
            \int_0^1\sigma'(\tau\xi^\top x)\,d\tau.
\]
Consequently, for $f_{u,w,\xi}\in\mathcal F'$,
\begin{equation}
\begin{aligned}
 f_{u,w,\xi}(x)
 ={}&\frac{\sigma(0)}{1+r}(u^\top x)\sigma'(w^\top x) +\frac r{1+r}\langle u\otimes\bar\xi,x^{\otimes2}\rangle
      \int_0^1\sigma'(w^\top x)\sigma'(\tau\xi^\top x)\,d\tau.
\end{aligned}
\label{eq:v3_Fprime_identity}
\end{equation}
For the first term, \eqref{eq:phi_bv} gives
\[
 (u^\top x)\sigma'(w^\top x)
 =A(u^\top x)
  +\int (u^\top x)\mathbf1\{w^\top x\geq s\}\,\mu(ds).
\]
The functions $(u^\top x)$ and
$(u^\top x)\mathbf1\{w^\top x\geq s\}$ belong to
$\left.\mathcal G^{(1)}\right|_{y\equiv1}$, and the expansion has
total absolute coefficient at most $|A|+\|\mu\|_{\mathrm{TV}}$.
For the second term, apply \eqref{eq:phi_bv} to each derivative
factor and multiply out:
\begin{equation}
\begin{aligned}
 \sigma'(w^\top x)\sigma'(\tau\xi^\top x)
 ={}&A^2
   +A\int\mathbf1\{w^\top x\geq s_1\}\,\mu(ds_1)\\
 &+A\int\mathbf1\{(\tau\xi)^\top x\geq s_2\}\,\mu(ds_2)\\
 &+\iint\mathbf1\{w^\top x\geq s_1\}
          \mathbf1\{(\tau\xi)^\top x\geq s_2\}\,
          \mu(ds_1)\mu(ds_2).
\end{aligned}
\label{eq:v3_two_halfspaces}
\end{equation}
The last integrand contains the indicator of the intersection
$\{w^\top x\geq s_1\}\cap\{(\tau\xi)^\top x\geq s_2\}$.
After multiplication by $\langle u\otimes\bar\xi,x^{\otimes2}\rangle$,
it is a member of $\left.\mathcal G^{(2)}\right|_{y\equiv1}$.
The underlying functions in the first three terms, after multiplication
by the same tensor factor, have zero or one halfspace restriction and
belong to this class by taking the missing restrictions to be the whole
space. Thus two halfspaces are sufficient, and the total
absolute coefficient in \eqref{eq:v3_two_halfspaces} is at most
\[
 |A|^2+2|A|\|\mu\|_{\mathrm{TV}}+\|\mu\|_{\mathrm{TV}}^2
 =(|A|+\|\mu\|_{\mathrm{TV}})^2.
\]
The integration over $\tau\in[0,1]$ has total mass one. Since
$1/(1+r)\leq1$ and $r/(1+r)\leq1$, the triangle inequality gives
\begin{equation}
\begin{aligned}
 \sup_{f\in\mathcal F'}|(P_n-P)f|
 \leq{}&|\sigma(0)|(|A|+\|\mu\|_{\mathrm{TV}})
       \sup_{g\in\left.\mathcal G^{(1)}\right|_{y\equiv1}}|(P_n-P)g| +(|A|+\|\mu\|_{\mathrm{TV}})^2
       \sup_{g\in\left.\mathcal G^{(2)}\right|_{y\equiv1}}|(P_n-P)g|.
\end{aligned}
\label{eq:v3_Fprime_reduction}
\end{equation}
Now the uniform convergence result for $G^{(1)}|_{y\equiv1}$ and $G^{(2)}|_{y\equiv1}$  from Lemma~\ref{claim:threshold_vc} prove \eqref{eq:v3_Fprime}.

\medskip\noindent
\emph{The integrated classes $\mathcal H^{(j)}$.}
Consider the separable Hilbert space $\mathsf H=L_2(P_x)$ and set
\[
 U_z(x)=\sigma(z^\top x)-\sigma(0),\qquad
 D_n=\frac1n\sum_{i=1}^n U_{Z_i}-\mathbb EU_Z\in\mathsf H.
\]
The map $z\mapsto U_z$ is Lipschitz into $\mathsf H$, since
\[
 \|U_z-U_{z'}\|_{\mathsf H}
 \leq C\bigl(\mathbb E[((z-z')^\top X)^2]\bigr)^{1/2}
 \leq KC\kappa_x\|z-z'\|.
\]
It is therefore strongly measurable, and
$\|U_z\|_{\mathsf H}\leq KC\kappa_x\|z\|$.
H\"older's inequality gives
\[
 \sup_{u\in(\mathbb S^{d-1})^{\otimes j}}
 \|\langle u,X^{\otimes j}\rangle\|_{L_2(P_x)}
 \leq K\kappa_x^j,\qquad j=1,2.
\]
For each $h=h_{u,w}\in\mathcal H^{(j)}$, Fubini's theorem yields
\[
 (P^z_n-P^z)h
 =\mathbb E_{X\sim P_x}\bigl[\phi(w^\top X)\langle u,X^{\otimes j}\rangle D_n(X)\bigr].
\]
Consequently, simultaneously for all $j$, the bound $\|\phi\|_{\infty} \leq C$ and Cauchy Schwartz give
\begin{equation}
 \sup_{h\in\mathcal H^{(j)}}|(P^z_n-P^z)h|
 \leq KC\kappa_x^j\|D_n\|_{\mathsf H}.
 \label{eq:v3_hilbert}
\end{equation}
Independence of the $Z_i$ gives the Hilbert-space variance identity
\begin{align}\label{eq:variance_H}
     \mathbb E\|D_n\|_{\mathsf H}^2
 =\frac1n\mathbb E\|U_Z-\mathbb EU_Z\|_{\mathsf H}^2
 \leq\frac{KC^2\kappa_x^2\kappa_z^2}{n}.
\end{align}

Now we seek a high-probability bound on $\|D_n\|_{\mathsf H}$.
Replacing $Z_i$ by an independent
copy $Z_i'$ changes $S := \|D_n\|_{\mathsf H}$ by at most
\[
 \frac{KC\kappa_x}{n}\bigl(\|Z_i\|+\|Z_i'\|\bigr).
\]
Averaging this inequality over $Z_i'$ gives
\[
 \left|S-\mathbb E\left[S\mid (Z_\ell)_{\ell\ne i}\right]\right|
 \leq
 \frac{KC\kappa_x}{n}
 \bigl(\|Z_i\|+\mathbb E\|Z\|\bigr).
\]
Since $\|\|Z\|\|_{\psi_2}\leq\kappa_z$, the centered function
of $Z_i$ on the left has conditional $\psi_2$ norm at most
$KC\kappa_x\kappa_z/n$, uniformly over the fixed values of the
other observations. The subgaussian bounded-difference inequality
\cite[Theorem 3]{maurer2021concentration} therefore yields
\[
 \|D_n\|_{\mathsf H}
 \leq
 \mathbb E\|D_n\|_{\mathsf H}
 +KC\kappa_x\kappa_z\sqrt{\frac tn}
\]
with probability at least $1-e^{-t}$, for every $t>0$.

Then \eqref{eq:variance_H} and Jensen's inequality give
\[
 \mathbb E\|D_n\|_{\mathsf H}
 \leq
 \bigl(\mathbb E\|D_n\|_{\mathsf H}^2\bigr)^{1/2}
 \leq \frac{KC\kappa_x\kappa_z}{\sqrt n}.
\]
Taking $t=\log(32/\delta)\geq1$ and absorbing the expectation
term into the bound, we obtain
\[
 \|D_n\|_{\mathsf H}
 \leq KC\kappa_x\kappa_z\sqrt{\frac tn}
\]
with probability at least $1-\delta/32$.
Equation \eqref{eq:v3_hilbert} consequently gives
\[
 \sup_{h\in\mathcal H^{(j)}}|(P^z_n-P^z)h|
 \leq
 KC^2\kappa_x^{j+1}\kappa_z
 \sqrt{\frac{\log(2/\delta)}{n}},
 \qquad j=1,2,
\]
on this same event, proving \eqref{eq:v3_H}.

\medskip\noindent
\emph{Products of residuals.}
Since $\sigma$ is continuous and has bounded variation, write
\[
 \sigma(z)=B+\int_{\mathbb R}\mathbf1\{z\geq s\}\,\nu(ds),
 \qquad |B|\leq2C,\quad\|\nu\|_{\mathrm{TV}}\leq C.
\]
Here $\nu$ is a finite signed measure. Expand the two residual
factors as in \eqref{eq:v3_two_halfspaces}.
The constant $B^2$ cancels in the empirical deviation. The unweighted
and $y$-weighted indicator terms have total absolute coefficients at
most $2|B|\|\nu\|_{\mathrm{TV}}+\|\nu\|_{\mathrm{TV}}^2\leq5C^2$
and $2(|B|+\|\nu\|_{\mathrm{TV}})\leq6C$, respectively. Hence
\begin{equation}
\begin{aligned}
 \sup_{f\in\mathcal F''}|(P_n-P)f|
 \leq{}&5C^2\sup_{g\in\left.\mathcal G^{(0)}\right|_{y\equiv1}}
                 |(P_n-P)g| +6C\sup_{g\in\mathcal G^{(0)}}|(P_n-P)g|
       +|(P_n-P)Y^2|.
\end{aligned}
\label{eq:v3_residual_reduction}
\end{equation}
Applying Lemma~\ref{claim:threshold_vc} for $\mathcal G^{(0)}|_{y\equiv1}$ and $\mathcal G^{(0)}$, and absorbing constants into $K$, proves
\eqref{eq:v3_residual_concentration}.

Finally, combine the event from Lemma~\ref{claim:threshold_vc},
applied with failure probability $\delta/4$, with the two
Hilbert-space concentration events. A finite union bound proves all
the stated simultaneous assertions.

\end{proof}

\begin{lemma}\label{lemma:regtouc}
Suppose Assumption~\ref{assm:reg} holds. Then for $n \geq \Omega(d\log(d))$, Assumption~\ref{uc:n} holds for $\eps_n = \frac{\creg^5 \sqrt{d}\log^2(dn)}{\sqrt{n}}$, with $\creg$ a polynomial factor in the $\creg$ from Assumption~\ref{assm:reg}.
Further, Assumption~\ref{uc:m} holds for with $\eps_m = \frac{\creg^6 \sqrt{d}\log(dm)}{\sqrt{m}}$. If we only want to prove Assumption~\ref{uc:m}, we can remove the assumption in \ref{assm:sigma} that $\sigma$ has bounded total variation.
\end{lemma}
\begin{proof}
For Assumption~\ref{uc:n}, we need to prove that with probability $1 - n^{-\Theta(1)}$, uniformly over $w, w' \in \nspace$, and $v \in \sd$, we have 
\begin{align}
    \left|\frac{1}{1 + \|w'\|}\left(\mathbb{E}_{x, y \sim \hat{\md}} \sigma(w'^\top x)\sigma'(w^\top x)x^\top v -  \mathbb{E}_{x, y \sim \md}\sigma(w'^\top x)\sigma'(w^\top x)x^\top v\right)\right| &\leq \eps_n/2 \\
    \left|\mathbb{E}_{x, y \sim \hat{\md}} y\sigma'(w^\top x)x^\top v-  \mathbb{E}_{x, y \sim \md}y\sigma'(w^\top x)x^\top v\right| &\leq \eps_n/2\\
    \left|\mathbb{E}_{x, y \sim \hat{\md}}(y - \sigma(w^{\top}x))(y - \sigma(w'^{\top}x)) -  \mathbb{E}_{x, y \sim \md}(y - \sigma(w^{\top}x))(y - \sigma(w'^{\top}x))\right| &\leq \eps_n
\end{align}
For the case that $\nspace = \R^d$, the desired bounds follow immediately from the results for $\mathcal{F}$, $\mathcal{F}'$, $\mathcal{F}''$ in Lemma~\ref{lemma:bv_vc} with $\delta = n^{-\Theta(1)}$, $C = \creg, \kappa_x = \creg$. 

The result for the case that $\nspace = \sd$ is given by standard epsilon-net arguments eg. similar to \cite[Lemma 23]{glasgow2025mean}, which can be appropriately tightened to attain this improved $\eps_n$ using Dudley's entropy integral. 

For Assumption~\ref{uc:n}, it remains to bound $\|\nabla^2_w K_{\hat{\md}}(w, w')\|$, $\|\nabla_w \nabla_{w'} K_{\hat{\md}}(w, w')\|$, and $\|\nabla^2_w F_{\hat{\md}}(w)\|$ uniformly with high probability. Observe that 
\begin{align}
    \sup_{w, w' \in \nspace} \|\nabla^2_w K_{\hat{\md}}(w, w')\| &= \sup_{u, v \in \sd, w, w' \in \nspace} \frac1n \sum_i \sigma''(w^{\top}x_i)\sigma(w'^{\top}x_i)(v^{\top}x_i)(x_i^{\top}u) \leq \creg^2 \frac{\lambda_{max} \left(\sum_i x_i x_i^{\top}\right)}{n}.
\end{align}
With probability $1 - n^{-\Theta(1)}$, this is at most $O(\creg^4\left(1 + \sqrt{\frac{d}{n}} + \frac{d}{n}\right) \leq O(\creg^4)$. The computation for $\|\nabla_w \nabla_{w'} K_{\hat{\md}}(w, w')\|$ and $\|\nabla^2_w F_{\hat{\md}}(w)\|$ is similar. This yields Assumption~\ref{uc:n} up to a polynomial factor in $\creg$.

Now for Assumption~\ref{uc:m} for $\mc{F}$, with $z_i \sim \rho$ iid, we seek to bound
\begin{align}
   \sup_{u \in \sd, w \in \nspace} \mathbb{E}_{z \sim \rho}u^{\top}\nabla_w K(w, z) - \frac1m\sum_{i = 1}^m u^{\top}\nabla_w K(w, z_i).
\end{align}
Now by definition for any $u \in \sd$, we have
\begin{align}
    u^{\top}\nabla_w K(w, z_i) = \mathbb{E}_x \sigma'(w^{\top} x)\sigma(z_i^{\top}x)(u^{\top}x),
\end{align}
so in the case that $\nspace = \R^d$, by Assumption~\ref{assm:sigma}, we can apply the second result in Lemma~\ref{lemma:bv_vc} with $\sigma = \sigma$, $\phi := \sigma'$ and $C = \creg$, $\kappa_z = \kappa_z$, and $\kappa_x := \creg$ to guarantee that
with probability $1 - \delta/(1 + t)^2$, 
\begin{align}
    \sup_{w \in \nspace, u \in \sd} \left\|\mathbb{E}_z  u^{\top}\nabla_w K(w, z_i) - \frac{1}{m}\sum_{i = 1}^m  u^{\top}\nabla_w K(w, z_i)\right\| &\leq \Theta\left(\creg^3\kappa_z \sqrt{\frac{\creg d\log^2(\creg d/\delta)}{m}}\right).
\end{align} 
Choosing $\delta := \frac{\eps_m^2}{16 \creg^2 m (1 + t)^2}$, we have that this bound is at most $\eps_m \kappa_t \log(e + t)/2$. 

The argument for proving uniform concentration over $\mc{F}'$ is similar. In the case that $\nspace = \R^d$, we have 
\begin{align}
    u^{\top}\nabla_w P_w \nabla_w K(w, z)v = \mathbb{E}_{x \sim \md} \sigma''(w^{\top}x)\sigma(z^{\top}x) (u^{\top}x)(x^{\top}v),
\end{align}
and so we can apply Lemma~\ref{lemma:bv_vc} with the activation $\sigma$, $\phi = \sigma''$, and $\kappa_z = \kappa_t$, and $\kappa_x = \creg$ to guarantee the desired result with probability $1 - \delta/(1 + t)^2$.

In the case that $\nspace = \sd$, the above uniform convergence bounds follow from standard empirical process theory arguments: from Lemma~\ref{lemma:smoothness} all the random variables are bounded since all neurons are on $\sd$, and we can take an $\eps$-net over $\nspace$ of size $\exp(d \log(O(1/\eps))$. We refer also the reader to \cite[Lemma 19]{glasgow2025mean}, where a similar result was carried with a slightly worse dependence on $d$ in $\eps_m$: this can be improved to our current $\eps_m$ using Dudley's entropy integral.

\end{proof}

\subsection{Proof of Lemma~\ref{lemma:errdynamicstight}.}

\begin{figure}[t]
    \centering
    \input{figures/fig_dynamics2}
    \caption{\small Decomposing $\frac{d}{dt}\dit = -\nu(\bwti, \rtmf) + \nu_{\hat{D}}(\hat{\xi}_{t_{\eta}}(w_i), \hat{\rho}_{t_{\eta}}^m)$. Upper bound on the approximate differences between the terms in the rectangles are given above the arrows. }
    \label{fig:dynamics}
\end{figure}

In this section, we prove Lemma~\ref{lemma:errdynamicstight}. Note that we assume Assumptions~\ref{assm:smoothness}, \ref{uc:n}, and \ref{uc:m} instead of Assumption~\ref{assm:reg}; these three assumptions are implied by Assumption~\ref{assm:reg} as per Lemma~\ref{lemma:regtouc} and the discussion following Assumption~\ref{assm:smoothness}.

Before proving Lemma~\ref{lemma:errdynamicstight}, we derive the following consequences from Assumptions~\ref{uc:n} and \ref{uc:m}.

\begin{lemma}[cf. \citep{glasgow2025mean}, Lemma 19]\label{lemma:concentration}
Suppose that Assumption~\ref{uc:m} holds and Lemma~{\ref{lemma:smoothness}} \ref{S3},\ref{S5} hold, and that for all $s \leq t < \infty$, we have that the random variable $\|w\|$ for $w \sim \rsmf$ is $\kappa_t$-sub-Gaussian for $\kappa_t \geq 1$. Then with probability $1 - m^{-\Theta(1)}$ over the initialization $\brz$, for all $t < \infty$ and $i \in [m]$, the following holds:
    \begin{align*}
        \|\nu(\bwti, \rtmf) - \nu(\bwti, \brt)\| &\leq \eps_m \kappa_t\log(e + t).\\
        \|D_t(i) - \bar{D}_t(i)\| &\leq \eps_m \kappa_t\log(e + t).
    \end{align*}
\end{lemma}

\begin{proof}[Proof of Lemma~\ref{lemma:concentration}]
Fix $t < \infty$ and $w \in \nspace$. By Equation~\eqref{eq:Vexpansion}, we have that
\begin{align}\label{eq:mean0object}
    \nu(w, \rtmf) - \nu(w, \brt) &:= P_w\left( \mathbb{E}_{w' \sim \rtmf}\nabla_w K(w, w') - \mathbb{E}_{w' \sim \brt} \nabla_w K(w, w')\right)
\end{align}
Thus 
\begin{align}
    \|\nu(w, \rtmf) - \nu(w, \brt)\| \leq \sup_{u \in \sd} \mathbb{E}_{w' \sim \rtmf}u^{\top}\nabla_w K(w, w') - \mathbb{E}_{w' \sim \brt} u^{\top}\nabla_w K(w, w').
\end{align}
Plugging in Assumption~\ref{uc:m} yields that 
with probability $1 - \delta/(1 + t)^2$, for $\delta = \eps_m^2 m^{-\Theta(1)}$,
\begin{align}
    \sup_{w \in \nspace} \left\|\mathbb{E}_{w' \sim \rho_0}\nabla_w K(w, \xi_t(w')) - \frac{1}{m}\sum_{i = 1}^m \nabla_w K(w, \bwti)\right\| &\leq \eps_m \kappa_t \log(e + t)/2.
\end{align}


Now we need to take a union bound over all $t < \infty$. Create a net over $[0, \infty)$ of spacing $\frac{\eps_m}{4\creg}$. By a union bound, with probability at least 
\begin{align}\label{eq:prob}
    1 - &\delta\sum_{k = 0}^{\infty}\left(\frac{1}{1 + k\frac{\eps_m}{6\creg}}\right)^2 \geq 1 - \frac{\delta 16\creg^2}{\eps_m^2} \geq 1 - m^{-\Theta(1)},
\end{align}
for any $t$ in the net, we have
\[
    \|\nu(w, \rtmf) - \nu(w, \brt)\| \leq \frac{\eps_m \kappa_t \log(e+t)}{2}.
\]

Now by Lemma~{\ref{lemma:smoothness}} \ref{S5}, for any $s, t < \infty$, and any $w_0$, we have
\begin{align}\label{eq:vel_bd}
    \|\xi_t(w_0) - \xi_s(w_0)\| &\leq \creg|t - s|\sup_{w \in \nspace}\sup_{r \in [s, t]}\nu(w, \rho_r^{\textsc{MF}})\\
    &\leq \creg|t - s|\left(1 + \sup_{r \in [s, t]}\mathbb{E}_{w' \sim \rho_r^{\textsc{MF}}}\|w'\|\right)\\
    &\leq 2\kappa_t \creg |t - s|.
\end{align}
Thus, for any $t < \infty$, there exists an $s$ in the net of distance at most $\frac{\eps_m}{4\creg}$. By a standard triangle inequality argument, we attain that with the probability $1 - m^{-\Theta(1)}$, for all $w \in \nspace$ and $t < \infty$, we have 
\begin{align}
    \|\nu(w, \rtmf) - \nu(w, \brt)\| \leq \eps_m \kappa_t \log(e + t).
\end{align}

The argument for proving concentration for $\bar{D}_t(w)$ uniformly over $w$ and $t$ is similar. We can write
\[
    D_t(w) - \bar{D}_t(w) = \mathbb{E}_{w' \sim \rho_0}\nabla_{\xi_t(w)} P_{\xi_t(w)} \nabla_{\xi_t(w)} K(\xi_t(w), 
    \xi_t(w')) - \frac{1}{m}\sum_{i = 1}^m \nabla_{\xi_t(w)} P_{\xi_t(w)} \nabla_{\xi_t(w)} K(\xi_t(w), \bwti),
\]
and thus we care about
\begin{align}
    X_t := \sup_{u, v \in \sd, w \in \nspace} \mathbb{E}_{z \sim \rtmf}u^{\top}\nabla_w P_w \nabla_w K(w, 
    z)v - \mathbb{E}_{z \sim \brt} u^{\top}\nabla_w P_w \nabla_w K(w, z)v
\end{align}
Assumption~\ref{uc:m} now gives the result that with probability $1 - 4\delta/(t + 1)^2$, $X_t \leq \frac{1}{2}\kappa_t \eps_m \log(e + t)$. Again we need to show Lipschitzness in $t$: we have by Lemma~\ref{lemma:smoothness}~\ref{S3} and \eqref{eq:vel_bd} above that
\begin{align}
    X_t - X_s &\leq 2\sup_{w' \in \nspace} \|\xi_t(w') - \xi_s(w')\| \sup_{z, \xi \in \nspace}\left\|\nabla_z \nabla_{\xi} P_{\xi}\nabla_{\xi} K(\xi, z)\right\|  \\
    &\leq 2(2\kappa_t\creg|t - s|)\creg.
\end{align}
The result now follows by the argument before by taking a net in $t$ over $[0, \infty)$ of spacing $\frac{\eps_m}{8\creg^2}$.

\end{proof}

\begin{lemma}[cf. Lemma 23 in \cite{glasgow2025mean}]\label{lemma:nconcentration}
Suppose Assumption~\ref{uc:n} holds. Then with probability $1 - n^{-\Theta(1)}$, uniformly over all $w \in \nspace$, and all $\rho \in \mathcal{P}(\nspace)$, we have
\[
    \|\nu_{\hat{\md}}(w, \rho) - \nu(w, \rho)\| \leq 2\eps_n  \left(1 + \mathbb{E}_{w \sim \rho}\|w\|\right).
\]
\end{lemma}
\begin{proof}
For the first bound, the velocity is linear in $\rho$, so it suffices to prove that uniformly over $w$ and $w'$, we have
\begin{align}
    \|\nu_{\hat{\md}}(w, \delta_{w'}) - \nu(w, \delta_{w'})\| \leq \eps_n (1 + \|w'\|).
\end{align}
We expand
\begin{align}
    \nu_{\hat{\md}}(w, \delta_{w'}) = P_w\mathbb{E}_{x, y \sim \hat{\md}} (y - \sigma(w'^\top x))\sigma'(w^\top x)x.
\end{align}
The result now follows immediately from Assumption~\ref{uc:n}. 

\end{proof}

We now restate Lemma~\ref{lemma:errdynamicstight}.
\begin{lemma}[Parameter-Space Error Dynamics]\label{lemma:errdynamicstight_assm}
Suppose Assumptions~\ref{assm:smoothness}, \ref{uc:n}, \ref{uc:m} hold. With probability $1 - \min(m, n)^{-\Theta(1)}$ over $\hat{D}$ and the initial sample of neurons $\rzm$, for all $t < \infty$ and $i \in [m]$,
$$
\frac{d}{dt}\dit = D_t(i) \dit - \mathbb{E}_{j \sim [m]}H_t(i, j) \djt +  \bm{\beta}_t(i) + {\bm{\epsilon}_{t}(i)},$$
where 
\begin{align}
    \bm{\beta}_t(i) := P_{\bwti}\left(\mathbb{E}_{j \sim [m]} \nabla K_t(i, j) - \mathbb{E}_{w \sim \rho_0}\nabla K_t(w_i, w) \right),
\end{align}
and
$\|{\bm{\eps}_{t}}(i)\| \leq \eps_n + \kappa_t^2\eps_{\eta} + \kappa_t\log(e + t)\|\dit\|\eps_m + 2\creg\left(\|\dit\|^2 + \mathbb{E}_j\|\djt\|^2\right)$, $\|\bm{\beta}_t(i)\| \leq \eps_m \log(e + t)$. 
\end{lemma}
\begin{proof}
We first decompose $\frac{d}{dt}\dit$ into five terms:
\begin{align}
    \frac{d}{dt}(\dit) &= -\nu(\bwti, \rtmf) + \nu_{\hat{\md}}(\hxiti, \rtm)\\
        &=  -\left(\nu(\bwti, \rtmf) -  \nu(\bwti, \brt)\right) - \left(\nu(\bwti, \brt) - \nu(\bwti, \rtm)\right)\\
        &\qquad - \left(\nu(\bwti, \rtm) - \nu(\hxiti, \rtm)\right) - \left(\nu(\hxiti, \rtm) - \nu_{\hat{\md}}(\hxiti, \rtm)\right)\\
        &\qquad -  \nu_{\hat{\md}}(\hxiti, \rtm) + \nu_{\hat{\md}}(\hat{\xi}_{t_{\eta}}(w_i), \hat{\rho}_{t_{\eta}}^m).
\end{align}
Note: We prove the lemma ignoring any higher order terms that arise from the projection $P_{\xi_t(w)}$. In the case that $\nspace = \sd$, following the proof of Lemma 5 in \cite{glasgow2025mean}, since both $\xi_t(w_i)$ and $\hat{\xi}_t(w_i)$ are on the sphere, we have that $\langle{\xi_t(w_i), \dit\rangle} = \frac{1}{2}\|\dit\|^2$. Thus any corrections to this analysis due to the projections will be on the order of $\creg\|\dit\|^2$.

\paragraph{First term:} $\nu(\bwti, \rtmf) -  \nu(\bwti, \brt)$.
By Equation~\eqref{eq:Vexpansion}, we have
\begin{align}\label{eq:term1}
\nu(\bwti, \rtmf) -  \nu(\bwti, \brt) &= P_{\bwti}\left( \mathbb{E}_{w' \sim \rtmf}\nabla_{\bwti} K(\bwti, w') - \mathbb{E}_{w' \sim \brt} \nabla_{\bwti} K(\bwti, w')\right)\\
&= P_{\bwti}\left(\mathbb{E}_{w \sim \rho_0}\nabla K_t(w_i, w)- \mathbb{E}_{j \sim [m]} \nabla K_t(i, j)\right).
\end{align}
Alternatively, by Lemma~\ref{lemma:concentration}, we have with high probability that $\|\nu(\bwti, \rtmf) -  \nu(\bwti, \brt)\| \leq \eps_m \kappa_t \log(e + t)$.

\paragraph{Second term: $\nu(\bwti, \brt) - \nu(\bwti, \rtm)$.} Here we have
\begin{align}
    \nu(\bwti, \brt) - \nu(\bwti, \rtm) &= \nabla F(\bwti) -  \mathbb{E}_{w' \sim \brt}\nabla K(\bwti, w')\\
    &\qquad - \nabla F(\bwti) + \mathbb{E}_{w' \sim \rtm}\nabla K(\bwti, w') \\
    &=  -\mathbb{E}_{j}\left(\nabla K(\bwti, \bwtj) - \nabla K(\bwti, \bwtj + \djt)\right)\\
    &= \mathbb{E}_{j \sim [m]}\left(H_t(i, j)\djt + \mathbf{v}_j\right),
\end{align}
where $\|\mathbf{v}_j\| \leq \creg\|\djt\|^2.$
Indeed we can plug Lemma~\ref{lemma:smoothness} \ref{S1} into the Lagrange error bound to attain
\begin{align}\label{eq:term_2}
    \|K'(w, w') - K'(w, w' + \Delta) - \nabla_{w'}K'(w, w')\Delta\|
    &\leq \|\Delta\|^2 \sup_{\tilde{w} \in\bar{\nspace}} \left\|\nabla^2_{\tilde{w}}K'(w, \tilde{w})\right\| \leq \|\Delta\|^2\creg.
\end{align}

\paragraph{Third term: $\nu(\bwti, \rtm) - \nu(\hxiti, \rtm)$.}
Here we have
\begin{align}
    \nu(\bwti, \rtm) - \nu(\hxiti, \rtm) &= -\nabla_w \nu(w, \rtm)|_{w = \bwti}\dit  + \mathbf{v},
\end{align}
where here again we use the Lagrange error bound and \ref{S5} to bound
\begin{align}
    \|\mathbf{v}\| \leq \|\dit\|^2 \sup_{w \in \bar{\nspace}} \left\|\nabla^2_{w} \nu(w, \rtm)\right\|_{op} \leq \creg\|\dit\|^2.
\end{align}
Recall that we have defined
\begin{align}
    \bar{D}_t(w) := \nabla_{\xi_t(w)} \nu(\xi_t(w), \brt) = \nabla_{\xi_t(w)} F'(\xi_t(w)) - \mathbb{E}_{w' \sim \brt}\nabla_{\xi_t(w)} K'(\xi_t(w), w').
\end{align}
Now
\begin{align}
    \nabla_{\bwti} \nu(\bwti, \rtm) &= \nabla_{\bwti} F'(\bwti) - \mathbb{E}_{j}\nabla_{\bwti} K'(\bwti, \hat{\xi}_t(w_j))\\
    &= \nabla_{\bwti} F'(\bwti) - \mathbb{E}_{j}\nabla_{\bwti} K'(\bwti, \bwtj) + \mathbf{M}_{j}\\
    &= \bar{D}_t(i) - \mathbb{E}_j\mathbf{M}_{j}.
\end{align}
where by the Lagrange error bound and \ref{S3},
\begin{align}
    \|\mathbf{M}_j\|_{op} \leq \|\djt\|\sup_{w, w'} \left\|\nabla_w\nabla_{w'}K'(w, w')\right\|_{op} \leq \creg\|\djt\|.
\end{align}
Now by Lemma~\ref{lemma:concentration}, we also have with high probability that $\|D_t(i) - \bar{D}_t(i)\| \leq \eps_m \kappa_t \log(e+t)$, and thus
\begin{align}\label{eq:term_3}
    \nu(\bwti, \rtm) - \nu(\hxiti, \rtm) &= -D_t(i)\dit  + \mathbf{v},
\end{align}
where  $\|\mathbf{v}\| \leq  \creg\|\dit\|\mathbb{E}_j\|\djt\| + \eps_m \kappa_t \log(e+t) \|\dit\|.$

\paragraph{Fourth term: $\nu(\hxiti, \rtm) - \nu(\hat{\xi}_{t_{\eta}}(w_i), \hat{\rho}_{t_{\eta}}^m)$}
First observe that since $t_{\eta} = \eta \lfloor{t/\eta}\rfloor$, we have $|t_{\eta} - t| \leq \eta$. We have 
\begin{align}
    \|\hat{\xi}_{t_{\eta}}(w_i) - \hxiti\| \leq |t - t_{\eta}|\sup_{s \in [t, t'] }\|\nu_{\hat{\md}}(\hat{\xi}_{s}(w_i), \rsm)\| \leq \eta\kappa_t\creg
\end{align}
by Lemma~\ref{lemma:smoothness}~\ref{S5} Thus by a second order analysis analogous to that in  {\bf{Third term}}, we have that 
\begin{align}
    \|\nu(\hxiti, \rtm) - \nu(\hat{\xi}_{t_{\eta}}(w_i), \hat{\rho}_{t}^m)\| \leq \creg(\eta\kappa_t\creg + (\eta\kappa_t\creg)^2).
\end{align}
Finally by a second order analysis analogous to that in  {\bf{Second term}},
\begin{align}
     \|\nu(\hat{\xi}_{t_{\eta}}(w_i), \hat{\rho}_{t}^m) - \nu(\hat{\xi}_{t_{\eta}}(w_i), \hat{\rho}_{t_{\eta}}^m)\| \leq \creg(\eta\kappa_t\creg + (\eta\kappa_t\creg)^2).
\end{align}
It follows that
\begin{align}\label{eq:time_disc_term}
    \|\nu(\hxiti, \rtm) - \nu(\hat{\xi}_{t_{\eta}}(w_i), \hat{\rho}_{t_{\eta}}^m)\| \leq 2\creg(\eta\kappa_t\creg + (\eta\kappa_t\creg)^2) \leq \kappa_t^2 \eps_{\eta}.
\end{align}
\paragraph{Fifth term: $\nu(\hat{\xi}_{t_{\eta}}(w_i), \hat{\rho}_{t_{\eta}}^m) - \nu_{\hat{\md}}(\hat{\xi}_{t_{\eta}}(w_i), \hat{\rho}_{t_{\eta}}^m)$}
By Lemma~\ref{lemma:nconcentration}, we have with high probability:
\begin{align}\label{eq:term4}
    \| \nu(\hat{\xi}_{t_{\eta}}(w_i), \hat{\rho}_{t_{\eta}}^m) - \nu_{\hat{\md}}(\hat{\xi}_{t_{\eta}}(w_i), \hat{\rho}_{t_{\eta}}^m) \| &\leq \kappa_{t_{\eta}}\eps_n.
\end{align}



\paragraph{Final result.}
Putting together Equations~\eqref{eq:term1}, \eqref{eq:term_2}, \eqref{eq:term_3}, \eqref{eq:time_disc_term},  and \eqref{eq:term4} or their spherically corrected counterparts, we have
\begin{align}
    \frac{d}{dt}\dit  &= D_t(i) \dit - \mathbb{E}_{j \sim [m], j \neq i}H_t(i, j) \djt + P_{\bwti}\left( \mathbb{E}_{j \sim [m]} \nabla K_t(i, j)- \mathbb{E}_{w \sim \rho_0} \nabla K_t(w_i, w)\right) + {\bm{\epsilon}},
\end{align}
    where 
    \begin{align}
        \|{\bm{\eps}}\| &\leq \eps_n + \kappa_t^2\eps_{\eta} + \kappa_t\eps_m \|\dit\|  + \creg\left(1.5\|\dit\|^2 + \|\dit\|\mathbb{E}_j\|\djt\| + 1.5\mathbb{E}_j\|\djt\|^2\right)\\
        &\leq \eps_n + \kappa_t^2\eps_{\eta} + \kappa_t\eps_m\log(e + t)\|\dit\| + 2\creg\left(\|\dit\|^2 + \mathbb{E}_j\|\djt\|^2\right).
    \end{align}
Using the alternative concentration approach to the first term in the decomposition, we also have the simpler result:
\begin{align}
    \frac{d}{dt}\dit  &= D_t(i) \dit - \mathbb{E}_{j \sim [m], j \neq i}H_t(i, j) \djt + {\bm{\epsilon}},
\end{align}
with 
\begin{align}
     \|{\bm{\eps}}\| &\leq \eps_n + \kappa_t^2\eps_{\eta} + 2\kappa_t\log(e + t)\eps_m + 2\creg\left(\|\dit\|^2 + \mathbb{E}_j\|\Delta_j\|^2\right).
\end{align}

\end{proof}

%% file: figures/fig_dynamics2.tex
\usetikzlibrary{positioning}
\begin{tikzpicture}[scale=0.9, every node/.append style={scale=0.9}]
    \node[draw, rectangle, minimum width=2cm, minimum height=1cm, thick] (rect1) {$\nu(\bwti, \rtmf)$};
    \node[draw, rectangle, minimum width=2cm, minimum height=1cm, right=0.5cm of rect1, thick] (rect2) {$\nu(\bwti, \brt)$};
    \node[draw, rectangle, minimum width=2cm, minimum height=1cm, right=0.5cm of rect2, thick] (rect3) {$\nu(\bwti, \rtm)$};
    \node[draw, rectangle, minimum width=2cm, minimum height=1cm, right=0.5cm of rect3, thick] (rect4) {$\nu(\hxiti, \rtm)$};
    \node[draw, rectangle, minimum width=2cm, minimum height=1cm, right=0.5cm of rect4, thick] (rect5) {$\nu(\hat{\xi}_{t_{\eta}}(w_i), \hat{\rho}_{t_{\eta}}^m)$};
     \node[draw, rectangle, minimum width=2cm, minimum height=1cm, right=0.5cm of rect5, thick] (rect6) {$\nu_{\hat{D}}(\hat{\xi}_{t_{\eta}}(w_i), \hat{\rho}_{t_{\eta}}^m)$};

    \draw[->, thick] (rect1.north east) to[out=30, in=150] node[pos=0.05, above] {$\bm{\beta}_{\infty} + \eps_m\|\xi_t(w_i) - \xii(w_i)\|$} (rect2.north west);
    \draw[->, thick] (rect2.north east) to[out=30, in=150] node[pos=1.0, above] {$\qquad-\mathbb{E}_{j}H_t(i, j)\djt$} (rect3.north west);
    \draw[->, thick] (rect3.north east) to[out=30, in=150] node[midway, above] {\quad\quad $D_t(i)\dit$} (rect4.north west);
    \draw[->, thick] (rect4.north east) to[out=30, in=150] node[midway, above] {$\leq \eps_n \approx \eta$} (rect5.north west);
    \draw[->, thick] (rect5.north east) to[out=30, in=150] node[midway, above] {$\leq \eps_n \approx \frac{1}{\sqrt{n}}$} (rect6.north west);
\end{tikzpicture}

%% file: neurips26/experiments_app.tex
\section{Supplemental Experimental Details for Misspecified Sobolev single-index model}
\label{sec:exp_app}

For $\gamma \in \{1, 2, 4, 8\}$, train a wide neural network with gradient descent on $n = 1024$ data points $(x_i, f^{\gamma}(x_i))$. For the best approximation of the population loss, when $d = 2$, we use $x_i$ evenly spaced around $\mathbb{S}^1$; otherwise we choose the $x_i$ randomly on $\sqrt{d-1}\sd$, and then truncate to be contained in $[-1, 1]^d$. We use a step size of $\eta = 0.1$ and a width of $m = 32768$ to approximate the mean field gradient flow dynamics on the population loss. We chose these values of $n, m, \eta$ because with twice as much granularity (choosing $n, m$ to be two times larger, or $\eta$ to be two times smaller), the results were very similar.

We plot the results for $d=2$ and $d = 128$.

For $x \in [-1, 1]^d$, let
\begin{align}
    f^*(x) = \sum_{k \in \mathbb{Z}}^{\infty} \hat{f}_k \cos(k\arccos(x_1))
\end{align}

Define $f^{\gamma}$ by 
\begin{align}
    \hat{f}^{\gamma}_k = \begin{cases}
        \frac{1}{2} & k = 0\\
        \frac{1}{4} & k \pm 1\\
        \frac{1}{2}(k/\sqrt{2})^{-(\gamma + 2.5)} & |k| > 1, \on{even}
    \end{cases}
\end{align}
We show in Observation~\ref{obs:bounded_sobolev} that in the case that $d = 2$, $f^*(x) = \mathbb{E}_{w \sim \rho^*}\on{ReLU}(x^{\top}w)$ for some $\rho^*$ with bounded $\gamma$-Sobolev norm. We chose this class of functions $\{f^{\gamma}\}$ because it was the simplest class we could think of which could be represented by a ReLU network (this requires zero odd Fourier coefficients for $k > 1$), and which had corresponding $\rho^*$ with bounded $\gamma$-Sobolev norm.

\begin{figure}[]
    \centering
    \includegraphics[width=\textwidth]{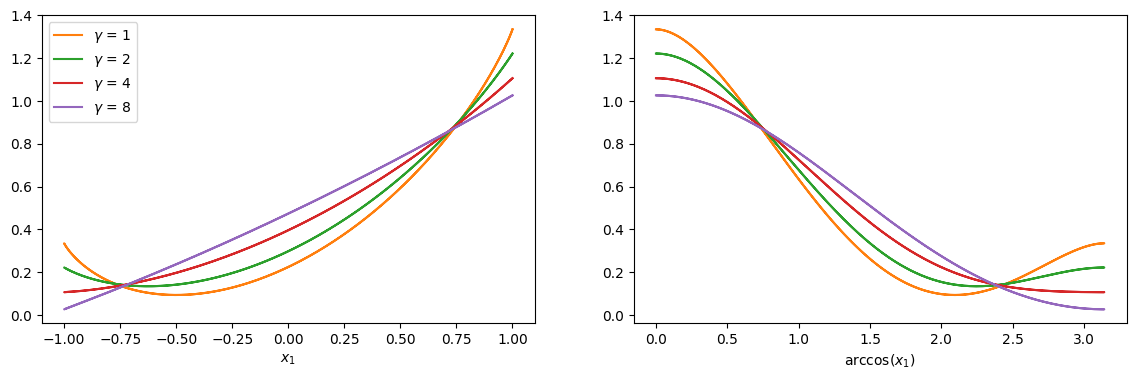}
    \caption{Plot of $\phi(x_1) = F(\arccos(x_1))$ for various values of $\gamma$. The function becomes smoother as $\gamma$ increases.}
    \label{fig:fstar}
\end{figure}

\begin{observation}[Representation on $\mathbb S^1$]\label{obs:bounded_sobolev}
Let $d=2$, and let $\theta(x)=\arccos(x_1)$. Let $F(\theta)=\sum_{k} \widehat f_k \cos(k\theta)$, and assume that $\widehat f_k=0$ for all odd $k>1$, and that for some
$\eta>0$,
\[
    |\widehat f_k| \leq C k^{-\gamma-2.5-\eta}
    \qquad \text{for all even } k\geq 2.
\]
Then there exists a finite signed measure $\rho^*$ on $\mathbb S^1$ whose
density has bounded $\gamma$-Sobolev norm such that
\[
    F(\arccos(x_1))
    =
    \int_{\mathbb S^1} \operatorname{ReLU}(\langle x,w\rangle)\,
    d\rho^*(w).
\]
\end{observation}

\begin{proof}
Let $\phi(t):=\operatorname{ReLU}(\cos t)$. We use the complex Fourier convention
\[
    h(t)=\sum_{k\in \mathbb Z} \widehat h_k e^{ikt},
    \qquad
    \widehat h_k = \frac{1}{2\pi}\int_0^{2\pi} h(t)e^{-ikt}\,dt.
\]
The Fourier coefficients of $\phi$ are
\[
    \widehat \phi_k =
    \begin{cases}
        \dfrac{1}{\pi}, & k=0,\\[6pt]
        \dfrac{1}{4}, & k=\pm 1,\\[6pt]
        \dfrac{(-1)^{m-1}}{\pi(4m^2-1)}, & k=\pm 2m,\ m\geq 1,\\[6pt]
        0, & |k|>1 \text{ odd}.
    \end{cases}
\]
Let $\mu$ be a finite signed measure on $[0,2\pi)$, and define
\[
    g_\mu(\theta)
    :=
    \int_0^{2\pi} \phi(\theta-\omega)\,d\mu(\omega).
\]
Writing $\widehat \mu_k := \int_0^{2\pi} e^{-ik\omega}\,d\mu(\omega)$, we have
\[
    \phi(\theta-\omega)
    =
    \sum_{k\in\mathbb Z}
    \widehat \phi_k e^{ik\theta}e^{-ik\omega}.
\]
Therefore
\[
    g_\mu(\theta)
    =
    \sum_{k\in\mathbb Z}
    \widehat \phi_k \widehat \mu_k e^{ik\theta}.
\]
Thus the Fourier coefficients of $g_\mu$ are $\widehat g_{\mu,k} = \widehat \phi_k \widehat \mu_k$.

Now we will construct $\rho^*$ to be the distribution of $(\cos(\omega), \sin(\omega))$ for $\omega \sim \mu^*$ for some measure $\mu^*$ on $[0, 2\pi)$.  Define the Fourier coefficients of $\mu^*$ by
\[
    \widehat \mu^*_k
    :=
    \begin{cases}
        \frac{\widehat f_k}{\widehat \phi_k} & \widehat \phi_k \neq 0 \\
        0 & \widehat \phi_k = 0.
    \end{cases}
\]
By the assumed decay,
\[
    |\widehat \mu^*_{\pm k}|
    \lesssim
    k^{-\gamma-0.5-\eta}.
\]

Therefore
\[
    \sum_{k\in\mathbb Z}
    (1+k^2)^\gamma |\widehat \mu^*_k|^2
    <\infty.
\]
Hence $\mu^*$ has bounded $\gamma$-Sobolev norm. Now
\[
    g_{\mu^*}(\theta)=F(\theta),
\]
Thus with $x = (\cos(\theta(x)), \sin(\theta(x)))$ we have 
\begin{align}
    \int_{\mathbb{S}^1} \on{ReLU}( \langle x,w\rangle) d\rho^*(w) = \mathbb{E}_{\omega \sim \mu^*} \on{ReLU}(\cos(\theta(x) - w)) = g_{\mu^*}(\theta(x)) = F(\theta(x)).
\end{align}
But since cosine is an even function $F(\theta(x)) = F(-\theta(x))$, so $\int_{\mathbb{S}^1} \on{ReLU}( \langle x,w\rangle) d\rho^*(w) = F(\arccos(x_1))$.
\end{proof}



\begin{figure}[]
    \centering
    \includegraphics[width=\textwidth]{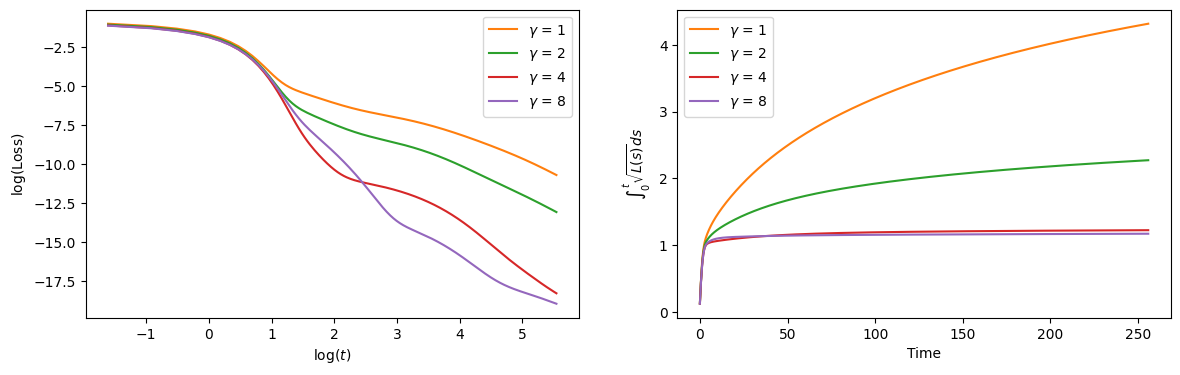}
    \caption{Approximate loss $L(\rtmf)$ (left) and $\int_{s = 0}^t \sqrt{L(\rsmf)}ds$ (right). Training both layers.}
    \label{fig:sim_loss}
\end{figure}

%% file: apx_second_layer.tex
\section{Extending the result to second layer weights}\label{apx:second_layer}
We include some discussion below of how we believe our results could be extended to the case of training second layer weights. We note that is is not a formal proof.

To extend the result, let $\nspace=\mathbb{R}^{d+1}$ denote the space of
concatenations $(a,w)$, where $a$ is the scalar second layer weight and
$w\in\mathbb{R}^d$, and for a distribution $\rho$ on $\nspace$, let
\[
f_\rho(x)
:= \mathbb{E}_{(a,w)\sim\rho}\bigl[a\sigma(w^\top x)\bigr].
\]
For $\theta=(a,w)$ and $\theta'=(a',w')$ in $\nspace$, define the kernel
\[
K(\theta,\theta')
= \mathbb{E}_{x\sim\mathcal{D}}
  \bigl[a\sigma(w^\top x)a'\sigma(w'^\top x)\bigr],
\]
which is equivalent to the kernel we define for only first layer weights,
multiplied by $aa'$. For our proof to go through, we need to check that
we have the desired smoothness properties of $K$ stated in Lemma~\ref{lemma:smoothness}
(essentially, Lipschitzness of the first few derivatives of $K$).

Now, because the kernel (and its derivatives) scale with $\|\theta\|$ and $\|\theta'\|$,
for these derivatives to be uniformly bounded, we need a uniform bound on all $\|\theta\|$ reached by gradient flow. Under our loss decay assumption, we can prove such a bound
using Grönwall's Inequality. For a distribution $\rtmf$ on $\nspace=\R^{d+1}$, let
\[
X_t
:= \sup_{\theta\in\operatorname{supp}(\rtmf)}\|\theta\|.
\]
Then we have
\[
\begin{aligned}
\frac{d}{dt}X_t
&\leq
\sup_{\theta=(a,w)\in\operatorname{supp}(\rtmf)}
\nu(\theta;\rtmf)
\\
&=
\sup_{\theta\in\operatorname{supp}(\rtmf)}
\mathbb{E}_{x\sim\mathcal{D}}
\bigl(f_{\rtmf}(x)-f^\star(x)\bigr)
\bigl(\sigma(w^\top x),a\sigma'(w^\top x)\bigr)
\\
&\leq
\sqrt{\loss_t}\sup_\theta \creg\|\theta\|
\leq \creg\sqrt{\loss_t}\,X_t.
\end{aligned}
\]
Thus, by Grönwall, for all $t$,
\[
X_t
\leq
\exp\left(
\creg\int_{s=0}^{\infty}\sqrt{\loss_s}\,ds
\right)X_0,
\]
and so all network parameters are bounded by this quantity throughout
training. It follows that (under our main assumption that
$S:=\int_{s=0}^{\infty}\sqrt{\loss_s}\,ds<\infty$), we have a uniform bound
of around $\exp(S)\creg$ on the desired derivatives of $K$,
and thus (with additional dependence on the constant $\exp(S)$), we believe our result should hold.

Note that for the finite data case, in order for the loss near-monotonicity argument to hold, we expect that we would may need to impose additional assumptions such as bounded second layer weights.